\documentclass[twoside]{article}

\usepackage[preprint]{aistats2027}

\usepackage[round]{natbib}

\usepackage{amssymb}
\usepackage{graphicx}
\usepackage{xcolor}
\usepackage{booktabs}
\usepackage{xurl}
\usepackage[hidelinks]{hyperref}

\usepackage{amsmath}
\usepackage{amsthm}

\theoremstyle{plain}
\newtheorem{theorem}{Theorem}[section]

\newtheorem{proposition}[theorem]{Proposition}
\newtheorem{lemma}[theorem]{Lemma}
\newtheorem{corollary}[theorem]{Corollary}
\theoremstyle{definition}

\newtheorem{definition}[theorem]{Definition}

\theoremstyle{plain}

\begin{document}

\runningtitle{Trade-off Functions for DP-SGD with Random Allocation}
\runningauthor{van Dijk, Ertan}

\twocolumn[

\aistatstitle{Trade-off Functions for DP-SGD with Subsampling based on Random Allocation: Tight Upper and Lower Bounds}

\aistatsauthor{ Marten van Dijk$^{*}$ \And Murat Bilgehan Ertan$^{*}$ }

\aistatsaddress{ CWI Amsterdam \And CWI Amsterdam } ]

\ifnum\statePaper=1 {\renewcommand{\thefootnote}{\fnsymbol{footnote}}\footnotetext[1]{Affiliated with Vrije Universiteit Amsterdam.}}\fi

\begin{abstract}
Within the $f$-DP framework, we derive a tight analysis of the trade-off function for Differentially Private Stochastic Gradient Descent (DP-SGD) with subsampling based on random allocation in which each sample is independently assigned to exactly one of $M$ minibatches per epoch, each minibatch corresponding to one of the $M$ SGD rounds within a single epoch. Our analysis holds under an explicit validity condition, whose hypotheses together force $\sigma \geq \sqrt{3/\ln M}$, where $\sigma$ is the DP noise multiplier. Unlike $f$-DP analyses for Poisson subsampling, which yield  non-closed implicit formulas that can be machine computed but are non-transparent, random allocation admits a tight analysis yielding transparent and interpretable closed-form bounds. For a single epoch, our concrete bounds, derived via the Berry-Esseen theorem, are tight up to constant factors. We demonstrate worked parameter settings  for a single epoch ($E=1$) with a corresponding trade-off function $\geq 1-a-\delta$, that is, only $\delta$ below the ideal random guessing diagonal $1-a$. For $\delta = 1/100$ and $\sigma = 1$, roughly $M \approx 1.14\times 10^6$ rounds and $N \approx 1.14\times 10^7$ training samples suffice to achieve meaningful differential privacy. This is in contrast to recent negative results for the regime $\sigma \leq 1/\sqrt{2 \ln M}$ for which no significant DP guarantee can exist.
\end{abstract}

\section{Introduction}
\label{sec:intro}

Differential privacy~\citep{DworkMNS06,Dwork} is the de facto standard for privacy in machine learning, and DP-SGD~\citep{Abadi2016} is the dominant mechanism for training deep models under DP constraints. Tight privacy accounting for DP-SGD is essential, because a loose analysis translates directly into either an over-conservative model or, worse, a guarantee that is weaker than advertised. The $f$-DP framework~\citep{DBLP:journals/corr/abs-1905-02383} characterizes privacy through the trade-off function between an adversary's false-positive and false-negative rates and provides, in a precise sense, the tightest available description of a mechanism's privacy guarantee, subsuming both $(\varepsilon,\delta)$-DP and R\'{e}nyi DP~\citep{Mironov2019}.

Nevertheless, $f$-DP analyses of DP-SGD have so far focused on \emph{Poisson subsampling}, in which each datapoint is independently included in {\em each} minibatch with a fixed probability
~\citep{Bu2019,Zheng2020}. Poisson subsampling is analytically convenient but practically uncommon. Deployed implementations of DP-SGD, and stochastic optimization more broadly, instead use \emph{random shuffling}, in which the dataset is randomly permuted once per epoch and partitioned into minibatches of equal size~\citep{MishchenkoKR20,Chua2024}. This mismatch between the algorithm that is analyzed and the algorithm that is run has left the privacy guarantee of the deployed mechanism imprecisely characterized. We analyze \emph{random allocation} (balls-and-bins), the scheme widely adopted as the analyzable stand-in for shuffling; the counterexample in Appendix~\ref{app:fixedsize} shows why random shuffling does not admit the worst-case reduction our bounds rest on. 

We give a tight closed-form $f$-DP analysis of DP-SGD with random allocation. For a single epoch with $M=N/m$ rounds (expected batch size $m=N/M$), under the hypotheses of Theorem~\ref{th:concrete}, which together force $\sigma \ge \sqrt{3/\ln M}$, the trade-off function $f$ satisfies
\begin{equation}
\begin{aligned}
  f(a) \;\geq\; 1-a \;-\; &
\biggl( 2
B\cdot
e^{1/\sigma^2}\cdot \frac{1+4\cdot e^{-3/\sigma^2}  }{(1-e^{-1/\sigma^2})^{2}} \\
&  +\frac{1}{\sqrt{2\pi}}\biggr)\cdot \mu -  O\!\big(1/M\big) \\
& \mbox{ with } \ \ \mu=\sqrt{\frac{e^{1/\sigma^2}-1}{M-1}},
\end{aligned}
\label{introbound}
\end{equation}
uniformly in $a\in[0,1]$, where $B$ is the Berry-Esseen constant and the $O(1/M)$ term assumes a fixed $\sigma$.
For suitable parameters this is close to the random guessing diagonal $1-a$ which provides perfect privacy (no leakage at all).
The $O(1/M)$ term in  lower bound (\ref{introbound}) is exactly expressed in concrete parameters $\sigma$ and $M$ in Theorem \ref{th:concrete}. This is a concrete non-asymptotic formula 
that can be evaluated directly to read off DP parameters; within the proof framework, which uses the Berry-Esseen theorem at its core, it is tight up to constants, since improving the dominant term requires improving the underlying Berry--Esseen estimate itself. Our regime is complementary to the recent impossibility result in~\citet{ErtanVanDijk2026}, which shows that for $\sigma < 1/\sqrt{2\ln M}$ the trade-off function lies far from the ideal random guessing diagonal $1-a$. This impossibility result is stated for one-epoch fixed-size shuffling; its Section~2 records that the same  result carries over  unchanged to random allocation.
Together, the two results  map out the privacy landscape of 
DP-SGD with random allocation. Noise below the lower threshold  $1/\sqrt{2\ln M}$ cannot give meaningful privacy, while noise satisfying the hypotheses of Theorem~\ref{th:concrete}, which force $\sigma\geq\sqrt{3/\ln M}$, is precisely characterized by the closed-form expressions we derive.

For $E$ epochs the lower bound (\ref{introbound}) composes to an $f^{\otimes E}$-DP guarantee for which $f^{\otimes E}(a)\geq 1-a-O(E/\sqrt{M})$, that is, the separation from the ideal random guessing diagonal has a loose upper bound which grows {\em linearly} in $E$. This implies that only for $E=o(\sqrt{M})$ the upper bound proves that  $f^{\otimes E}(a)\geq 1-a-o(1)$ is tightly characterized by the random guessing diagonal with perfect privacy. Analysis for Poisson subsampling shows that the separation behaves as $O(\sqrt{E/M})$ which allows $E=o(M)$ for perfect privacy \citep{Bu2019}.
For random allocation we expect the same asymptotic behavior. In order to prove this, we develop a sharper approximation than the Berry-Esseen theorem used for (\ref{introbound}) and Theorem~\ref{th:concrete}.
This needed because the Berry-Esseen term with constant $B$ adds up linearly over epochs. 
We introduce and prove a 
refinement of the central limit theorem based on  an Edgeworth expansion for two terms (Theorem \ref{theo:edge}) which we use for proving the desired asymptotic behavior for multiple $E>1$ epochs (Theorem~\ref{theo:asym}). Deriving an explicit convergence rate for it in order to obtain concrete well-behaving lower bounds of the form (\ref{introbound}) for general $E$ remains an open problem. 

The closed-form bounds enable concrete parameter recommendations; they are straightforward to compute and the bounds are human interpretable which provides valuable intuition. Worked settings (Section~\ref{sec:parameters}) show that, with a single epoch, $\sigma=1$, and $\delta=10^{-2}$, a dataset of size $N\approx 1.14\times 10^{7}$ partitioned into $M\approx 1.14\times 10^{6}$ minibatches achieves a meaningful DP guarantee. This is directly relevant to federated learning~\citep{McMahanMRHA17}, where datasets often comprise the records of millions of users contributed across a small number of training rounds. The relevance of having closed forms for parameter selection was argued for Poisson subsampling by~\citet{ProactiveDP}, whose guarantees are based on a generalization of the moment-accountant method for $(\epsilon,\delta)$-DP (which they show is tight for certain parameter settings); we give the analogous closed-form recipes for general parameter settings for random allocation within the general $f$-DP framework. 

We make three contributions.
\begin{enumerate}
    \item Closed-form lower bound (\ref{introbound}) and its exact expression (Theorem~\ref{th:concrete})  for single-epoch  DP-SGD with random-allocation.
    The assumptions used for proving this result together force $\sigma\geq\sqrt{3/\ln M}$. This complements the impossibility range $\sigma < 1/\sqrt{2\ln M}$ of ~\citet{ErtanVanDijk2026}.
   We show worked $f$-DP parameter settings at federated-learning scale (Section~\ref{sec:parameters}).
    \item An Edgeworth-type CLT refinement (Theorem~\ref{theo:edge}), $F_n(x)=\Phi(p_n(x))+o(1/n)$, which holds \emph{uniformly} on the range $|x|\leq o(n^{1/16})$, where $\Phi$ is the cumulative normal distribution and $p_n(x)$ is a third degree polynomial depending on the first four moments of the cumulative distribution $F_n$ that corresponds to a normalized sum of $n$ i.i.d. random variables. Within the range $|x|\leq o(n^{1/16})$, which our privacy analysis uses, this sharpens the Berry-Esseen bound which holds for all $x$ and depends on the first two moments and the  third absolute moment of $F_n$ and  implies a much more loose $O(1/\sqrt{n})$ error term.
    \item For fixed $\sigma$, an asymptotic Gaussian-DP characterization yielding $f^{\otimes E}\geq G_{c\sqrt{e^{1/\sigma^2}-1}}-o(1)$ uniformly (and $\to 1-a$ for $c=0$) over $E=r_M^2 M$ epochs with $r_M\rightarrow c$ as $M\rightarrow \infty$, improving the separation dependence from $O(\mu E)$ to $O(\mu\sqrt E)$ (Theorem~\ref{theo:asym}, Corollary~\ref{cor:ths}). The asymptotic result follows from a single-epoch lower bound  $f\geq G_{\mu+\gamma'_M}\otimes f_{0,\hat{\delta}_M}$, where $f_{0,\hat{\delta}_M}(a)=\max\{1-\hat{\delta}_M-a,0\}$, $\gamma'_M=O(\sqrt{\ln M}/M)$ and $\hat{\delta}_M=o(1/M)$. 
    Making convergence rates $\gamma'_M$ and $\hat{\delta}_M$ explicit is left open.
\end{enumerate}
All proofs are deferred to the appendix, organized as a common adversarial reduction followed by two parallel tracks, non-asymptotic Berry--Esseen (Theorem~\ref{th:concrete}) and asymptotic Edgeworth (Theorems~\ref{theo:edge} and~\ref{theo:asym}). A proposition-level map for navigation is given in Appendix~\ref{app:outline}. Related work is detailed in Appendix~\ref{sec:related} and summarized here.

\textbf{Related work.} For the shuffled mechanism specifically, \citet{Chua2024,Chua2024b} \emph{upper-bound} the trade-off function of shuffled DP-SGD (equivalently, lower-bound its privacy-loss curve $\delta(\varepsilon)$), via explicit adversarial events, in the single-epoch~\citep{Chua2024} and multi-epoch~\citep{Chua2024b} settings, with numerical evaluation showing substantial privacy loss at small $\sigma$. The impossibility result of~\citet{ErtanVanDijk2026} is in the same direction, additionally giving a closed-form threshold $\sigma < 1/\sqrt{2\ln M}$ below which the trade-off function is provably bounded away from the random guessing diagonal $1-a$; this result also holds for random allocation as studied in this paper. Further numerical tools appear in~\citet{Birrell2024}. 
None gives a matching \emph{two-sided} closed-form characterization of the trade-off function in $(\sigma, M, E)$ that one can inspect, differentiate, or invert to read off parameter requirements. 

\section{Background: 
DP-SGD \& \texorpdfstring{$f$-DP}{f-DP}}
\label{sec:background}

We consider differentially private stochastic gradient descent (DP-SGD)~\citep{Abadi2016} where the training dataset of $N$ samples is partitioned into $M$ minibatches, one processed per round, and a single pass through all $M$ rounds constitutes one \emph{epoch}. We analyze \emph{random allocation} (balls-and-bins) subsampling~\citep{FeldmanShenfeld2025,ChuaBB2024}, in which each of the $N$ samples is {\em independently} assigned to exactly one of the $M$ minibatches chosen uniformly at random. 
This makes minibatch sizes random with mean $m = N/M$. 
In each round $j \in \{1, \ldots, M\}$, the algorithm computes the gradients corresponding to the samples in the $j$-th minibatch, clips each individual gradient to norm at most $C$, sums the clipped gradients\footnote{Averaging by the public constant $m=N/M$ is post-processing and does not affect the privacy analysis; the realized minibatch sizes are random and are not used.}, and adds Gaussian noise. 

The $f$-DP framework~\citep{DBLP:journals/corr/abs-1905-02383} provides a more complete characterization of privacy by viewing the task of distinguishing whether a given training sample was used as one of the $N$ samples for computing DP-SGD updates as a hypothesis test. Given neighboring datasets $d$ and $d'$, let $H_0$ denote the null hypothesis that the underlying dataset is $d$ and $H_1$ the alternative that it is $d'$. A rejection rule $\phi$ outputs $1$ if $H_0$ is rejected.\footnote{The rejection rule can be probabilistic and output a probability $\phi$ after which a coin flip with bias $\phi$ decides whether the null hypothesis will be rejected.} The Type~I error (false positive rate) and Type~II error (false negative rate) are defined as
$$\begin{gathered}\alpha(\phi) = \mathbf{E}[\phi \mid d] \in [0,1] \qquad \text{and} \\ \beta(\phi) = 1 - \mathbf{E}[\phi \mid d'] \in [0,1].\end{gathered}$$
The \emph{trade-off function} captures the fundamental trade-off between these two error types:
$$
T(P^{(M)}_0,P^{(M)}_1)(a)  = \inf_{\phi} \{ \beta(\phi) : \alpha(\phi) \leq a \}
$$
in which $P^{(M)}_0$ plays the null hypothesis $H_0$ and $P^{(M)}_1$ plays $H_1$. The trade-off function is convex, non-increasing, and is $\leq 1-a$.
The trade-off function is not necessarily symmetric\footnote{A trade-off function $g$ is symmetric  if $g=g^{-1}$ with the inverse $g^{-1}$ defined as $g^{-1}(a) =
\inf \{ t\in [0,1] \ : \ g(t)\leq a\}$ for $a\in [0,1]$.}. 

\begin{definition}
The DP-SGD\footnote{The $f$-DP framework is general and we only specify it here toward DP-SGD, the topic of this paper.} mechanism ${\cal M}$ for one epoch with $M$ rounds is $f$-DP if for all neighboring $d$ and $d'$ the corresponding $P^{(M)}_0$ and $P^{(M)}_1$ have trade-off function $T(P^{(M)}_0,P^{(M)}_1)\geq f$.  
\end{definition}

Since this definition is symmetric in $d$ and $d'$, the (maximizing) $f$ is also symmetric.
A mechanism provides strong privacy if it is $f$-DP for a trade-off function $f(a)$ close to $1-a$, the ideal random guessing diagonal corresponding to zero leakage: in this case, no rejection rule can simultaneously achieve low false positive and false negative rates. Conversely, if $f(a)$ is far from $1-a$, then there exists neighboring data sets $d$ and $d'$ which the adversary (in the DP-SGD reduced hypothesis testing game) can distinguish  with high confidence, indicating significant privacy leakage. The \emph{Gaussian trade-off function} $G_\mu(a) = \Phi(\Phi^{-1}(1-a) - \mu)$, parameterized by $\mu \geq 0$, arises naturally from distinguishing two Gaussians $N(0,1)$ and $N(\mu,1)$ differing by a shift of $\mu$; smaller $\mu$ corresponds to stronger privacy ($G_0(a) = 1-a$ is the ideal). Trade-off function $f_{0,\delta}(a)=\max\{1-a-\delta,0\}$, parameterized by $\delta\geq 0$, arises naturally from distinguishing two 
uniform distributions $U(0,1)$ and $U(\delta,1+\delta)$ differing by a shift of $\delta$; smaller $\delta$ corresponds to stronger privacy ($f_{0,0}(a) = 1-a$ is the ideal).

The trade-off function provides a complete geometric representation of the privacy guarantee. The $f$-DP framework contains $(\varepsilon,\delta)$-DP as a special case, and $f$-DP guarantees convert losslessly into R\'{e}nyi DP guarantees~\citep{DBLP:journals/corr/abs-1905-02383}. For $E$-fold composition (corresponding to $E$ epochs), we write $f^{\otimes E}$ for the composed trade-off function; in the $f$-DP framework tensor $\otimes$ is exact and does not introduce a lower bound with slack on composition of trade-off functions. The $f$-DP framework also introduces a subsampling operator $C_p$ for analysing fixed-size subsampling, which also covers Poisson subsampling and which we do not need in this paper for random allocation; Appendix \ref{app:cproute} explains that the $C_p$ subsampling operator introduces a non-tight lower bound with some slack due to its definition (bounded by multiplying $\mu$ by at most a factor $\sqrt{2}$ in a Gaussian DP description).

\textbf{Adversarial model.}
We use the adversarial model of~\citet{ErtanVanDijk2026}, except that the adversary is not granted the minibatch sizes, which are random under random allocation. We briefly summarize it here for intuition; the full technical derivation is given in Appendix~\ref{app:adversarial}.

The standard worst-case adversary in the DP framework is granted, in addition to the noisy minibatch updates produced by DP-SGD, auxiliary information that allows it to isolate the noisy contribution of the potentially differing sample in each round. 
Informally, this means the adversary can ``subtract out'' the known contributions of all other samples in the minibatch, leaving only the noise and if present the signal from the differing sample. This is the strongest possible adversary: any privacy guarantee that holds against this adversary automatically holds against any weaker, more realistic adversary. 

\citet{FeldmanShenfeld2025}, their Claim~C.3 and Theorem~3.1, proves that this worst-case adversary is indeed worst-case for random allocation by showing that the resulting hypothesis testing pair described below in the next paragraph is a \emph{dominating pair} of distributions.
Their theorem covers adaptive query methods (the balls-and-bins $k=1$ case for Gaussian noise is Theorem~3.1 in~\citet{ChuaBB2024}). After rescaling, their dominating pair consists of exactly the two 
hypotheses that we analyze, in the reverse order to our convention $(H_0,H_1)$.

After projecting onto the gradient direction of the differing sample and normalizing by the noise scale, the worst-case adversary's observation in each round reduces to a scalar. Concretely, the adversary observes $M$ samples
$$
(x_j)_{j=1}^M \sim \big(N(e_1/\sigma, 1), \ldots, N(e_M/\sigma, 1)\big),
$$
where under $H_0$ (dataset $d$) all $e_j = 0$, and under $H_1$ (dataset $d'$) exactly one $e_j = 1$ for a uniformly random $j \in \{1,\ldots,M\}$ with the rest equal to zero. That is, under $H_0$ every round produces pure noise $N(0,1)$, while under $H_1$ the single round containing the differing sample produces a shifted observation $N(1/\sigma, 1)$.

By the Neyman--Pearson lemma, the optimal test thresholds the likelihood ratio, which reduces to thresholding the sample average of lognormal variables $\frac{1}{M}\sum_{j=1}^M e^{X_j/\sigma - 1/(2\sigma^2)}$. The trade-off function $f_M(a) = T(P^{(M)}_0,P^{(M)}_1)$
for distributions
$P^{(M)}_0$ and $P^{(M)}_1$ corresponding to $H_0$ and $H_1$
then characterizes the fundamental privacy guarantee of the mechanism (it is $f$-DP for the largest convex function below $\min\{f_M,f_M^{-1}\}$).

\section{Main Theorem}
\label{sec:main}

Our first main result provides concrete, non-asymptotic bounds on the trade-off function of DP-SGD with random allocation for a single epoch.

\begin{theorem} \label{th:concrete} 
Let $f$ be the largest trade-off function such that DP-SGD  for a single epoch with $M$ rounds with subsampling based on random allocation and noise multiplier $\sigma$ is $f$-DP.
    Let $\mu$ be as defined in (\ref{introbound}),
$B\in [0.4097,0.4748]$ be the universal constant in the Berry-Esseen theorem for independent and identically distributed random variables,    and
\begin{eqnarray}
 \delta &\geq & 2
B\cdot
e^{1/\sigma^2}\cdot \frac{1+4\cdot e^{-3/\sigma^2}  }{(1-e^{-1/\sigma^2})^{2}} \cdot \mu
+ 
\frac{1}{\sqrt{2\pi}} 
\mu \nonumber \\
&&
 +
\left\{ \frac{1}{4\sqrt{2\pi}} + \frac{1}{2\sqrt{2e\pi}}(1+ \frac{e^{1/\sigma^2}}{1-e^{-1/\sigma^2}}) \right\} \mu^2 \nonumber \\
&&
+
 \frac{1}{4\sqrt{2e\pi}}\mu^{3}  +\frac{1}{32\sqrt{2e\pi}}\mu^4 \nonumber \\
&& +\frac{4.52}{2.88\sqrt{\ln M}-2.41/\sqrt{\ln M}}\,M^{-25/24} \nonumber 
\end{eqnarray}
with
\begin{equation}
\delta + B\cdot
e^{1/\sigma^2}\cdot \frac{1+4\cdot e^{-3/\sigma^2}  }{(1-e^{-1/\sigma^2})^{2}} \cdot \mu
\leq  1/2- \Phi(-(e^{1/\sigma^2}-1)/2). \label{condth}
\end{equation}
Then, for $a\in [0,1]$,  
$$f(a) \geq 1-a-\delta.$$
If a fresh random allocation is drawn independently in every epoch, then adaptive composition implies that DP-SGD over $E$ epochs is $f^{\otimes E}$-DP and, for $a\in[0,1]$,
$$f^{\otimes E}(a) \geq \max\{(1-\delta)^E-a,0\}.$$
\end{theorem}

\subsection{Interpretation}
\label{sec:interpretation}

The hypotheses of the theorem together imply $\sigma \geq \sqrt{3/\ln M}$. By the lower bound on $\delta$, the left-hand side of~\eqref{condth} is at least $3B \cdot e^{3/(2\sigma^2)}/\sqrt{M-1} \geq e^{3/(2\sigma^2)}/\sqrt{M}$, and the right-hand side is at most $1/2$, so $\sigma$ cannot be too small. This naturally restricts our analysis to the regime complementary to the impossibility result of~\citet{ErtanVanDijk2026} applied to random allocation as discussed in the introduction.

\textbf{The Berry-Esseen term dominates.}
For large enough $M$, the first term involving the Berry-Esseen constant $B$ dominates the lower bound on $\delta$ (for small $\sigma$), and the theorem can be approximately restated by replacing the lower bound on $\delta$ by
\begin{equation}
\delta \geq \left( 2B\cdot
e^{1/\sigma^2}\cdot \frac{1+4\cdot e^{-3/\sigma^2}  }{(1-e^{-1/\sigma^2})^{2}}
+\frac{1}{\sqrt{2\pi}}\right) \cdot \mu =\delta_{\mathrm{bound}}.
\label{bounddelta}
\end{equation}
This observation is important: within the Berry-Esseen proof framework, the $B$-term is the bottleneck and any improvement to the constant in the concrete bounds requires improving the dominating Berry-Esseen approximation itself; the $M^{-1/2}$ order is unavoidable. 
This motivates the asymptotic approach of Section~\ref{sec:asymptotic}, which improves upon the Berry-Esseen approach by using a different proof technique (which we are unable to apply for concrete non-asymptotic values of $M$). 

We verify the approximation $\delta_{\mathrm{bound}}$ numerically against the exact trade-off $\delta_{\mathrm{exact}}$, computed by using the $\varepsilon\downarrow0$ extension of the privacy-profile statistic of \citet{FeldmanShenfeld2025} (Lemma~F.1, stated there for $\varepsilon>0$) by adaptive quadrature at $M=2$ and by numerical inversion of the characteristic function of the likelihood ratio at larger $M$. 
At the computed parameter settings of Table~\ref{tab:parameters} the leading term $\mu/\sqrt{2\pi}$ reproduces $\delta_{\mathrm{exact}}$  to within $0.02\%$, and the ratio tends to $1$ as $M\to\infty$. However, at small $M$ it is far looser, overestimating by $75.6\%$ at $M=2$, $\sigma=1$. The numerical verification shows $\delta_{\mathrm{exact}}\approx \mu/\sqrt{2\pi}$ for large enough $M$, hence, 
\[
  \frac{\delta_{\mathrm{bound}}}{\delta_{\mathrm{exact}}}\;\approx\;1+2B\,e^{1/\sigma^2}\,\frac{1+4e^{-3/\sigma^2}}{(1-e^{-1/\sigma^2})^{2}}\,\sqrt{2\pi}.
\]

\textbf{Exact third absolute moment.}
The factor $e^{1/\sigma^2}(1+4e^{-3/\sigma^2})/(1-e^{-1/\sigma^2})^2$ in the first term of $\delta_{\mathrm{bound}}$ and in~\eqref{condth} comes from a simplification~\eqref{Bsmallexact} of the Berry-Esseen bound~\eqref{upper}. Without the simplification, the factor is the smaller quantity $\rho^3/(e^{1/\sigma^2}-1)^2$, where $\rho^3=\mathbf{E}[|Y_j-1|^3]$ is the third absolute central moment of the lognormal variables $Y_j=e^{X_j/\sigma-1/(2\sigma^2)}$ with $X_j\sim N(0,1)$ (Proposition~\ref{prop:moments}). Theorem~\ref{th:concrete} also holds with this exact factor in both places\footnote{The reason is that its proof bounds the approximation error by~\eqref{upper} and applies~\eqref{Bsmallexact} only in its last step, where it also derives the conditions  $\sigma\geq\sqrt{3/\ln M}$ and $M\geq 3$ from the lower bound on $\delta$ together with~\eqref{condth}.}, provided that $\sigma\geq\sqrt{3/\ln M}$ and $M\geq 3$.  We state the theorem with the simplification because it involves only exponentials of $1/\sigma^2$ and is accurate for small $\sigma$. The two factors differ by $0.1\%$ at $\sigma=0.5$ and by $10\%$ at $\sigma=0.75$. For large $\sigma$ the simplified factor grows like $5\sigma^4$, whereas the exact factor grows only like $\sqrt{8/\pi}\cdot\sigma$. With the exact factor, $\delta_{\mathrm{bound}}$ improves and the ratio $\delta_{\mathrm{bound}}/\delta_{\mathrm{exact}}$ becomes $1+2B\rho^3\sqrt{2\pi}/(e^{1/\sigma^2}-1)^2$. For $\sigma\in[0.75,2]$ this ratio lies between $10.7$ and $19.9$, whereas with the simplification it lies between $18.8$ and $181.5$. In Figure~\ref{fig:sensitivity}, the number of rounds $M$ required with the exact factor (dashed curves) nearly coincides with the number required with the simplification (solid curves) for small $\sigma$, and keeps decreasing as $\sigma$ grows.

\begin{figure}[t]
    \centering
    \includegraphics[width=\columnwidth]{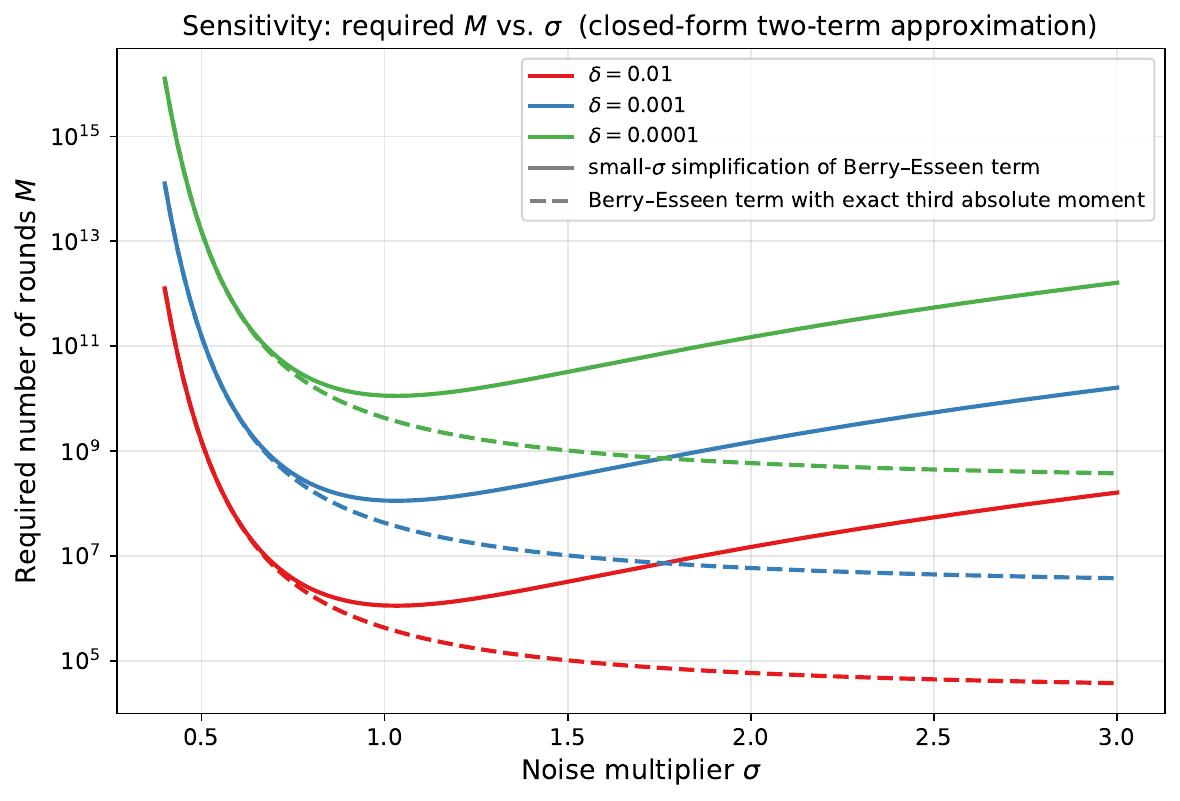}
    \caption{Required rounds $M$ vs.\ noise multiplier $\sigma$ from the closed-form two-term approximation, for $\delta\in\{10^{-2},10^{-3},10^{-4}\}$ at $E=1$. The minimum of the solid curves is an artifact of the simplification.
    }   \label{fig:sensitivity}
\end{figure}

\textbf{Lower bound on \texorpdfstring{$M$}{M}.}
The approximate bound~\eqref{bounddelta} yields a first lower bound on the number of rounds:
\begin{equation}
\begin{aligned}
M\geq 1+\Biggl(\frac{1}{\delta}\biggl[&2B\cdot
e^{3/(2\sigma^2)}\cdot \frac{1+4\cdot e^{-3/\sigma^2}}{(1-e^{-1/\sigma^2})^{3/2}} \\
&+\frac{\sqrt{e^{1/\sigma^2}-1}}{\sqrt{2\pi}}\biggr]\Biggr)^2.
\end{aligned}
\label{lowboundM}
\end{equation}
Combining the lower bound on $\delta$ with condition~\eqref{condth} gives a second lower bound on $M$:
\begin{equation}
\begin{aligned}
M\geq 1+\Biggl(&\frac{1}{1/2- \Phi(-(e^{1/\sigma^2}-1)/2)} \\
&\biggl[3B\cdot
e^{3/(2\sigma^2)}\cdot \frac{1+4\cdot e^{-3/\sigma^2}}{(1-e^{-1/\sigma^2})^{3/2}} \\
&\quad+\frac{\sqrt{e^{1/\sigma^2}-1}}{\sqrt{2\pi}}\biggr]\Biggr)^2.
\end{aligned}
\label{lowboundM1}
\end{equation}
For $\delta \leq \frac{2}{3}(1/2 - \Phi(-(e^{1/\sigma^2}-1)/2))$, the first bound~\eqref{lowboundM} implies the second. Since we require $\delta$ to be small for a meaningful DP guarantee, this condition is satisfied for practical parameter ranges, for example, $\delta=0.01$ allows $\sigma\leq 3.71$ and, vice versa, $\sigma=1$ allows $\delta\leq 0.20$.

\textbf{Gaussian noise and dataset size.}
Each round of DP-SGD adds Gaussian noise $N(0, (C\sigma/m)^2) = N(0, (C\sigma M/N)^2)$ to the averaged\footnote{By dividing by the expected minibatch size $m$.} clipped gradients. For DP-SGD to remain robust against this noise, $C\sigma M/N$ must be small, i.e., $N$ must be sufficiently large relative to $M$. This creates a tension: $M$ must be large enough for~\eqref{lowboundM} to be satisfied, but $M/N$ must remain small for utility. A caveat of using a large $N$ relative to $M$ is that the expected minibatch size $m=N/M$ is large, slowing convergence in a distributed DP-SGD setting with heterogeneously distributed training data and thus requiring more epochs, which in turn weakens the DP guarantee.

\textbf{From the concrete theorem to asymptotics.}
Theorem~\ref{th:concrete} can also be used to derive an asymptotic result: If $R>0$, $\delta = r_\delta \cdot M^{-R}$ and $E = r_E \cdot M^R$ is an integer, then
$$
f^{\otimes E}(a) \geq (1-\delta)^E - a \to e^{-r_\delta r_E} - a \ \  \text{ as } \ \ M \to \infty.
$$
However, this leads to a weak asymptotic bound since $\delta \approx 2B \cdot M^{-(1-3/s^2)/2}$ for $\sigma = s/\sqrt{\ln M}$ with $s > \sqrt{3}$, hence we must have $E = r_E \cdot M^{(1-3/s^2)/2}$ (which is at best  $O(\sqrt{M})$). At the same time, we need $M/N \to v$ for some small constant $v$ in order to be robust against the Gaussian noise $N(0, (C\sigma M/N)^2)$. The new asymptotic result of Section~\ref{sec:asymptotic} is stronger as it only requires $E = r_M^2 M$ with $r_M\rightarrow 0$ ($E$ is sub-linear in $M$) as $M \to \infty$. 
This means that for $E$ much larger than $O(\sqrt{M})$ (which leaks more privacy), we can still prove a comparable DP guarantee.

\section{Asymptotic Analysis}
\label{sec:asymptotic}

We introduce a new proof technique based on an Edgeworth-type CLT refinement that replaces the Berry-Esseen approximation with a finer asymptotic estimate:

\begin{theorem} \label{theo:edge}
    Let $F(x)$ be a cumulative distribution for which all moments $\mu_k$, $k\geq 1$, exist with $\mu_1=0$ and $\mu_2=\bar\sigma^2>0$.
Suppose that the corresponding characteristic function $\phi(\xi)$ tends to $0$ for $\xi \rightarrow \pm \infty$. Let \(X_1,X_2,\ldots\) be i.i.d. random variables with cumulative distribution $F(x)$ 
and define the standardized sum 
\[
\begin{gathered}
S_n=\frac{X_1+\cdots+X_n}{\bar\sigma\sqrt n} \ \ \mbox{ with cumulative distribution } \\
F_n(x)=\mathbb P(S_n\le x).
\end{gathered}
\]
Let
    \begin{eqnarray*}
p_n(x)&=& 
x+c_n(1-x^2)+d_n(3x-x^3) \\
&& {}+c_n^2(4x^3-7x)
\end{eqnarray*}
with $c_n$ and $d_n$ as defined by 
$$c_n=\frac{\mu_3}{6\mu_2^{3/2}\sqrt n} \ \ \mbox{ and } \ \ 
d_n=\frac{\mu_4-3\mu_2^2}{24\mu_2^2 n}.$$
Then, if $\mu_3\geq 0$, $p_n(-x)\leq -p_n(x)$ for $1\leq x$. 
Let $1\leq \kappa_n=o(n^{1/16})$.
Then, restricted to $|x|\leq \kappa_n$ and for sufficiently large $n$, polynomial $p_n(x)$ is strictly increasing and invertible. Furthermore, uniformly in $|x|\leq \kappa_n$,
\begin{eqnarray*}
\Phi(p_n(x))&=&  F_n(x) + o(1/n).
\end{eqnarray*} 
\end{theorem}

For completeness, Berry-Esseen implies $\Phi(x)=F_n(x)+O(1/\sqrt{n})$ for all $x$.

Theorem \ref{theo:edge} yields only an asymptotic result, and we do not know its explicit convergence rate, which we need in order to obtain {\em concrete} parameter settings that improve over Theorem~\ref{th:concrete}. The new asymptotic result does improve over the {\em asymptotic} result that can be derived from Theorem~\ref{th:concrete}:

\begin{theorem} \label{theo:asym}
Let $f$ be the largest trade-off function such that DP-SGD  for a single epoch with $M$ rounds with subsampling based on random allocation and noise multiplier $\sigma$ ($\sigma$ is assumed to be a constant and not dependent on $M$) is $f$-DP.
There exist sequences $\hat{\delta}_M=o(1/M)$ and $\gamma'_M=O(\sqrt{\ln M}/M)$ such that 
$$
f(a) \geq (G_{\mu+\gamma'_M} \otimes f_{0,\hat{\delta}_M})(a) \ \  \mbox{ for } \ \ a\in [0,1],
$$
where $\mu$ (which depends on $M$) is  as defined in (\ref{introbound}).

Finally, assume a fresh random allocation is drawn independently in every epoch, let $c\geq 0$ and suppose that $r_M\geq 0$ and $r_M\rightarrow c$ as $M\rightarrow \infty$ and consider a number of epochs  $E=r_M^2\cdot M$, assumed to be an integer. Then DP-SGD over $E$ epochs is $f^{\otimes E}$-DP and, as $M\rightarrow \infty$,
$$ f^{\otimes E}(a)\geq G_{c \cdot \sqrt{e^{1/\sigma^2}-1}}(a)-o(1)$$
uniformly in $a\in [0,1]$. 
If $r_M\rightarrow 0$ as $M\rightarrow \infty$, then 
$f^{\otimes E}(a) \rightarrow 1-a$ uniformly in $a\in[0,1]$. 
\end{theorem}

As a consequence, $E=r_M^2\cdot M$ epochs 
tends not to leak any privacy for $M\rightarrow \infty$ and $r_M\rightarrow 0$.

\textbf{Separation.}
We may define the separation \citep[Definition~5.2]{ErtanVanDijk2026} of a trade-off function $f(a)$ from the random guessing diagonal $1-a$ as the distance
$$ \mbox{sep}(f) = \max_{a\in [0,1]} \min_{\gamma\in[0,1]}\| (a,f(a))-(\gamma,1-\gamma)\|_2.$$
If $f\geq g$ and $g$ is a convex, decreasing and symmetric trade-off function, then 
$$ \mbox{sep}(f)\leq \mbox{sep(g)}= \frac{1}{\sqrt{2}}(1-2a^*), \ \ \mbox{ where } \ \  a^*=g(a^*)
$$
is the unique fixed point of $g$. For the Gaussian trade-off function $G_\mu(a) = \Phi(\Phi^{-1}(1-a) - \mu)$ we have $a^*=\Phi(-\mu/2)$. For any two trade-off functions $f$ and $h$ (not necessarily symmetric), 
$\mbox{sep}(f\otimes h)\leq \mbox{sep}(f)+\mbox{sep}(h)$.

\begin{corollary} \label{cor:ths}  For constant $\sigma$,  Theorem \ref{theo:asym} implies $\mbox{sep}(f^{\otimes E})=O(\mu\sqrt{E})$ as opposed to Theorem \ref{th:concrete} which can only be used to derive $\mbox{sep}(f^{\otimes E})=O(\mu E)$.
\end{corollary}

Both theorems yield the same $\mu=O(1/\sqrt{M})$ dependency. The asymptotic improvement of Theorem \ref{theo:asym} shows the much better $O(\sqrt{E})$ dependency on the number of epochs $E$.

\textbf{Comparing random allocation and Poisson subsampling.}
See Appendix \ref{app:cproute}. The lower bound of Theorem~\ref{theo:asym} on the composed trade-off function for random allocation and the exact composed trade-off function for Poisson subsampling~\citep{Bu2019}
have the same limit $G_{c\sqrt{e^{1/\sigma^2}-1}}$.
The bound of Corollary~5.4 in \citet{DBLP:journals/corr/abs-1905-02383} is loose and uses a coefficient $c$ in the Gaussian DP guarantee that is a factor $\sqrt{2}$ ($\sigma \rightarrow 0^+$) to $1$ ($\sigma \rightarrow \infty$) larger. 
In the regime $r_M\to0$,
the two exact asymptotics and the bound all converge to the ideal random guessing diagonal $1-a$.
None of these results teach concrete parameter settings, since no closed-form convergence rate that can be inverted in $M$ and $\sigma$ is available.  Both for Poisson subsampling as well as for random allocation, we need to have a concrete (closed formula) expression of the convergence rate of the exact asymptotic results in order to derive concrete parameter settings.

\textbf{Open problem.}
Theorem~\ref{th:concrete} penalizes the trade-off function by its $1/\sqrt{M}$ Berry-Esseen term across the whole range $a\in[0,1]$. We notice that even without the Berry-Esseen term in the lower bound of $\delta$ in Theorem~\ref{th:concrete} we still have the $\mu/\sqrt{2\pi}=O(1/\sqrt{M})$ term. The advantage of Theorem~\ref{theo:asym} is that our new proof technique allows us to replace the Berry-Esseen term by a suitable lower bound 
which maintains the $1/\sqrt{M}$ dependency for the separation but provides a much better asymptotic lower bound when considering the whole range $a\in[0,1]$. This allows an improved analysis when considering multiple epochs $E$. As an open problem for future work, the sequences $\hat{\delta}_M$ and $\gamma'_M$ (which characterize the convergence rate) can be made explicit such that we obtain a concrete result like Theorem \ref{th:concrete}  for multiple epochs that allows us to select tight practical parameters and also matches the asymptotic result of Theorem \ref{theo:asym}.  

\section{Concrete Parameter Settings}
\label{sec:parameters}

We now extract practical parameter settings from Theorem~\ref{th:concrete} and examine the interplay between $\sigma$, $\delta$, $M$, $E$, and $N$.

\subsection{Single-Epoch Regime \texorpdfstring{($E=1$)}{(E=1)}}

For the single-epoch case, the composition result simplifies to $f(a) \geq 1 - a - \delta$. For large enough $M$, the Berry-Esseen term dominates and $\delta$ is approximately given by~\eqref{bounddelta}. Consider $\sigma = 1$ and $\delta = 1/100$. Inverting the two-term closed-form lower bounds the number of rounds $M \approx 1.14 \times 10^6$ rounds\footnote{All numerical evaluations in the figures and tables of this paper substitute the proven Shevtsova upper bound $B = 0.4748$ \citep{Shevtsova2011} for the universal Berry-Esseen constant in~\eqref{lowboundM}.}, and the exact solution of Theorem~\ref{th:concrete} agrees to within $0.1\%$ (see Table~\ref{tab:parameters}). To keep the Gaussian noise per round $C\sigma M/N$ small (say, at most $0.1$), we need $N \approx 1.14\times 10^7$ training samples. In a distributed or federated learning setting, a client contributing a single epoch of training over a dataset of this size can achieve a meaningful DP guarantee with $\delta = 0.01$. 
Table~\ref{tab:parameters} summarizes the required parameters for several values of $\sigma$, with the simplification~\eqref{Bsmallexact} and with the exact third absolute moment. The latter requires $M\approx4.30\times10^5$ rounds at $\sigma=1$ and only $M\approx5.91\times10^4$ at $\sigma=2$.

\begin{table}[t]
\caption{Required number of rounds $M$ at $\delta=1/100$, $E=1$.
$M_{\mathrm{BE}+1/\sqrt{2\pi}}$ inverts the closed-form lower bound:
$\delta \approx \bigl(2B\cdot\mathrm{coef}(\sigma)+1/\sqrt{2\pi}\bigr)\cdot \mu$, where $2B\cdot\mathrm{coef}(\sigma)\cdot\mu$ is the Berry-Esseen term of Theorem~\ref{th:concrete}. The upper rows use Theorem~\ref{th:concrete} as stated, with the simplification~\eqref{Bsmallexact}, so that $\mathrm{coef}(\sigma)=e^{1/\sigma^2}(1+4e^{-3/\sigma^2})/(1-e^{-1/\sigma^2})^2$. The lower rows use the exact third absolute moment $\rho^3$ of Proposition~\ref{prop:moments}, so that $\mathrm{coef}(\sigma)=\rho^3/(e^{1/\sigma^2}-1)^2$ (Section~\ref{sec:interpretation}). $M_{\mathrm{exact}}$ is the smallest $M$ such that the full right-hand side of Theorem~\ref{th:concrete} (all six terms plus the validity condition~\eqref{condth}, with the same $\mathrm{coef}(\sigma)$) satisfies $\delta\le 1/100$, found by bisection. The closed form matches $M_{\mathrm{exact}}$ to within $0.1\%$ (upper rows) and $0.2\%$ (lower rows). Min.\ $N$ is the minimum dataset size needed to keep the per-round noise $C\sigma M/N\le 0.1$ at $M_{\mathrm{exact}}$.}
\label{tab:parameters}
\begin{center}
\begin{tabular}{cccc}
\toprule
$\sigma$ & $M_{\mathrm{BE}+1/\sqrt{2\pi}}$ & $M_{\mathrm{exact}}$ & Min.\ $N$ \\
\midrule
\multicolumn{4}{c} {Simplification~\eqref{Bsmallexact}} \\
\midrule
$0.50$ & $1.57\!\times\!10^{9}$ & $1.57\!\times\!10^{9}$ & $7.87\!\times\!10^{9}$ \\
$0.75$ & $3.71\!\times\!10^{6}$ & $3.72\!\times\!10^{6}$ & $2.79\!\times\!10^{7}$ \\
$1.00$ & $1.14\!\times\!10^{6}$ & $1.14\!\times\!10^{6}$ & $1.14\!\times\!10^{7}$ \\
$1.50$ & $3.23\!\times\!10^{6}$ & $3.23\!\times\!10^{6}$ & $4.85\!\times\!10^{7}$ \\
$2.00$ & $1.49\!\times\!10^{7}$ & $1.49\!\times\!10^{7}$ & $2.98\!\times\!10^{8}$ \\
\midrule
\multicolumn{4}{c}{Exact third absolute moment} \\
\midrule
$0.50$ & $1.57\!\times\!10^{9}$ & $1.57\!\times\!10^{9}$ & $7.86\!\times\!10^{9}$ \\
$0.75$ & $3.10\!\times\!10^{6}$ & $3.10\!\times\!10^{6}$ & $2.33\!\times\!10^{7}$ \\
$1.00$ & $4.29\!\times\!10^{5}$ & $4.30\!\times\!10^{5}$ & $4.30\!\times\!10^{6}$ \\
$1.50$ & $1.03\!\times\!10^{5}$ & $1.03\!\times\!10^{5}$ & $1.54\!\times\!10^{6}$ \\
$2.00$ & $5.90\!\times\!10^{4}$ & $5.91\!\times\!10^{4}$ & $1.18\!\times\!10^{6}$ \\
\bottomrule
\end{tabular}
\end{center}
\end{table}

\subsection{Multiple Epochs}

\textbf{Linear-in-$E$ vs.\ conjectured $\sqrt{E}$ dependence.\;\;}
The composition bound $f^{\otimes E}(a)\geq (1-\delta)^E - a$ obtained by iterating
Theorem~\ref{th:concrete} carries an effective $\delta\cdot E$ penalty, and this
linear-in-$E$ scaling is the dominant source of slack for $E\gg 1$. Two analyses suggest that the correct dependence may be \(\sqrt E\): our own asymptotic
result already gives $\mathrm{sep}(f^{\otimes E})=O(\mu\sqrt{E})$
(Corollary~\ref{cor:ths}), and the analogous Poisson-subsampling analysis, see Section \ref{sec:asymptotic},  shows that composing per-round
trade-off functions converges to a Gaussian trade-off function $G_\mu$, whose $E$-fold
composition is $G_{\mu\sqrt{E}}$. We therefore conjecture that the concrete
(non-asymptotic) bound for random allocation also obeys and can be improved to a $\sqrt{E}\cdot\delta$ penalty (our stated open problem for future work).

For $E$ epochs, the composition guarantee becomes $f^{\otimes E}(a) \geq (1-\delta)^E - a$. For $(1-\delta)^E$ to remain close to $1$, we need $\delta E$ to be small, i.e., $\delta \ll 1/E$. This places increasing demands on $M$ and $N$ as $E$ grows, and the parameters quickly become impractical.

\textbf{Multi-epoch explosion.\;\;}
Consider $E = 100$ epochs. For the composition guarantee to be meaningful, we need $(1-\delta)^{100}\approx 1-100\delta$ to be close to $1$, which requires $\delta \leq 1/10{,}000$. Plugging $\delta = 10^{-4}$ into the lower bound~\eqref{lowboundM} for $\sigma = 1$ pushes the required number of rounds to approximately $M \approx 1.14\times 10^{10}$. To keep the Gaussian noise $C\sigma M/N$ manageable, we then need \(N \gtrsim 1.14\times 10^{11}\), a dataset of one hundred billion samples. The explosion is driven by the quadratic dependence of $M$ on $1/\delta$ in~\eqref{lowboundM}: halving $\delta$ quadruples $M$, which in turn requires quadrupling $N$.
For $E = 4$ epochs, the situation is more moderate: $\delta \leq 1/400$ suffices, requiring \(M \approx 1.82\times 10^7\) and \(N \approx 1.82\times 10^8\) for $\sigma = 1$. While substantially larger than the single-epoch case, these numbers remain within the range of large-scale distributed learning systems.

\textbf{Even under the conjectured $\sqrt{E}\,\delta$ scaling, large $E$ is impractical.\;\;}
Suppose hypothetically that the loss from the random-guessing diagonal grows as
$\sqrt{E}\,\delta$ rather than $E\,\delta$. Hitting a target $\delta_{\mathrm{tgt}}$ over
$E$ epochs then requires $\delta< \delta_{\mathrm{tgt}}/\sqrt{E}$, and the rule of
thumb \(M\approx 114/\delta^2\) obtained by inverting~\eqref{lowboundM} at $\sigma=1$
(calibrated against the $M\!\approx\!10^{10}$ value for $\delta\!=\!10^{-4}$ in the
``Multi-epoch explosion'' paragraph above) gives \(M\approx 1.14\times 10^{6}\cdot E\) for
$\delta_{\mathrm{tgt}}=10^{-2}$, with the per-round noise condition $C\sigma M/N\le 0.1$
adding $N\gtrsim 10\,M$. So even under the more favorable $\sqrt{E}$ scaling,
$E\!=\!100$ already needs \(M\sim 1.14\times 10^8\), \(N\sim 1.14\times 10^9\), and $E\!=\!10^{4}$
pushes both into the $10^{10}$--$10^{11}$ range. 

\textbf{Practical sweet spot.}
The analysis points strongly toward single-epoch training as the practical sweet spot for DP-SGD with random allocation. In the federated learning setting, where each client participates in only a small number of epochs (e.g., $E = 1$ or $E = 4$), the single-epoch guarantee $f(a) \geq 1 - a - \delta$ with $\delta = 0.01$ and $N \approx 1.14\times 10^7$ is achievable and directly applicable. Single-epoch DP is likely the most relevant regime for clients contributing local training data. That meaningful privacy can be guaranteed in this setting is a positive practical finding.

\section{Conclusion}
We have given a tight closed-form $f$-DP analysis of DP-SGD with random allocation in the high-noise regime covered by Theorem~\ref{th:concrete}, complementary to the small-$\sigma$ impossibility result of~\citet{ErtanVanDijk2026}. The non-asymptotic Berry-Esseen bound (Theorem~\ref{th:concrete}) is tight up to constants within its proof framework and yields transparent parameter recipes. At $\sigma=1$ and $\delta=10^{-2}$, $M\approx 1.14\times 10^{6}$ rounds and $N\approx 1.14\times 10^{7}$ samples suffice in the single-epoch regime relevant to federated learning. Going beyond Berry-Esseen, an Edgeworth-type CLT refinement (Theorem~\ref{theo:edge}) underpins the asymptotic Gaussian-DP characterization of Theorem~\ref{theo:asym}, improving the multi-epoch separation from $O(\mu E)$ to $O(\mu\sqrt{E})$ (Corollary~\ref{cor:ths}); compared with the Poisson-subsampling asymptotic of~\citet{Bu2019},
both regimes collapse to the random-guessing diagonal $1-a$ as $r_M\to 0$, and for $r_M\to c>0$ the lower bound of Theorem~\ref{theo:asym} converges to the exact Poisson limit $G_{c\sqrt{e^{1/\sigma^2}-1}}$.
Two open problems remain. The first is making the sequences $\hat{\delta}_M$ and $\gamma'_M$ in Theorem~\ref{theo:asym} explicit to obtain a concrete multi-epoch "Gaussian DP plus a small error" bound (rather than an improvement of the $f_{0,\delta}$-form of Theorem~\ref{th:concrete}, where the $\mu/\sqrt{2\pi}$ separation contribution is unavoidable). The second, given that Corollary~\ref{cor:ths} and the analogous Poisson analysis both point toward a $\sqrt{E}\,\delta$ rather than $E\,\delta$ multi-epoch penalty, is a closed-form proof of this conjectured dependence for concrete parameter settings.

% The contribution of Marten van Dijk and Murat Bilgehan Ertan to this publication is part of the
% project CiCS of the research program Gravitation which is (partly) financed by the Dutch Research
% Council (NWO) under the grant 024.006.037.

\subsection*{AI use statement}

The results and proof techniques in this paper are our own. Generative AI tools (OpenAI ChatGPT and Codex, and Anthropic Claude through Claude Code) assisted by checking our proofs line by line and by helping to write and edit the text. Additionally, we used generative AI tools to edit the paper for readability, to draft and restructure some paragraphs. Generative AI tools also implemented the supplementary code for our numbers, Table~\ref{tab:parameters} and the figures, and helped check and format the references. We verified all AI-generated content ourselves and take responsibility for the final content of this work, including text, claims, or artifacts produced with the aid of generative AI.

\bibliographystyle{plainnat}
\bibliography{bib}

%%%%%%%%%%%%%%%%%%%%%%%%%%%%%%%%%%%%%%%%%%%%%%%%%%%%%%%%%%%%
%\input{checklist}

\clearpage
\appendix
\thispagestyle{empty}

\onecolumn
\aistatstitle{Trade-off Functions for DP-SGD with Subsampling based on Random Allocation: Tight Upper and Lower Bounds \\
Supplementary Materials}

\section{Appendix Outline}
\label{app:outline}

The appendix is organized into a common adversarial reduction (Appendix~\ref{app:adversarial}) and two parallel proof tracks: a non-asymptotic track based on Berry-Esseen (Appendices~\ref{app:berryesseen} and~\ref{app:tradeoff}, with a numerical sensitivity addendum in Appendix~\ref{app:sensitivity}) and an asymptotic track based on the Edgeworth expansion (Appendices~\ref{app:edgeworth} and~\ref{app:tradeoff_edgeworth}, with a comparison to asymptotic Poisson in Appendix \ref{app:cproute}). Appendix~\ref{sec:related} reviews related work, and Appendix~\ref{app:fixedsize} explains why the analysis is restricted to random allocation.

\textbf{Related work (Appendix~\ref{sec:related})} We review privacy accounting for DP-SGD, which mostly targets Poisson subsampling, together with the one-sided bounds for shuffled DP-SGD of~\citet{Chua2024,Chua2024b} and~\citet{ErtanVanDijk2026}. We then discuss random allocation and the dominating-pair reduction of~\citet{ChuaBB2024} and~\citet{FeldmanShenfeld2025}, other analyses of the same scheme, and Edgeworth expansions and Gaussian-DP limits, including the exact Poisson coefficient of~\citet{Bu2019}.

\textbf{Common setup (Appendix~\ref{app:adversarial})}
We state the worst-case adversarial model and reduce a single epoch of DP-SGD with random allocation to a scalar hypothesis test. Theorem~\ref{thm:fs-dom} recalls the result of~\citet{FeldmanShenfeld2025} that makes this reduction worst-case for random allocation, in both orders of the neighboring datasets. Under $H_0$ (dataset $d$) the adversary observes $M$ i.i.d.\ standard Gaussian rounds; under $H_1$ (dataset $d'$) exactly one round, chosen uniformly by the allocation, is replaced by an $N(1/\sigma,1)$ observation. The corresponding likelihoods are written out in~(\ref{pdfd})--(\ref{pdfdprime}); by the Neyman--Pearson lemma the most powerful test at each fixed Type-I error level thresholds the likelihood ratio, equivalently the sample mean of the lognormal random variables $Y_j=\exp(X_j/\sigma-1/(2\sigma^2))$ with $X_j\sim N(0,1)$. The trade-off function $f(a)=\beta(\alpha^{-1}(a))$ defined by~(\ref{alpha})--(\ref{beta}) is the object of study throughout the remaining appendices; this reduction is shared by both proof tracks. This function is $T(P^{(M)}_0,P^{(M)}_1)$, which is not symmetric. Both tracks bound it from below by a symmetric trade-off function, which then also bounds $T(P^{(M)}_1,P^{(M)}_0)$ from below, so that Theorem~\ref{thm:fs-dom} turns the bound into a guarantee for the mechanism.

\textbf{Non-asymptotic track: proof of Theorem~\ref{th:concrete}.}
\begin{itemize}
\item \emph{Appendix~\ref{app:berryesseen} (Application of the Berry-Esseen Theorem).} The Berry-Esseen theorem is stated and applied to the lognormal sample mean of~(\ref{Yj}). Proposition~\ref{prop:moments} computes the mean $u=\mathbf{E}[Y_j]=1$, the second, third, and fourth central moments $\mu_k=\mathbf{E}[(Y_j-u)^k]$, and the third absolute central moment $\rho^3=\mathbf{E}[|Y_j-u|^3]$. Proposition~\ref{prop:epsh} converts these into an explicit upper bound on the approximation error $\epsilon_n(h)$ between $\mathbf{P}[\bar{Y}_n<h]$ and the corresponding Gaussian probability, with separate small-$\sigma$ bound~(\ref{Bsmallexact}) and large-$\sigma$ bound~(\ref{Blarge}). The closing subsection identifies $\sigma=s/\sqrt{\ln n}$ with $s>\sqrt{3}$ as a sufficient regime for the small-$\sigma$ bound~(\ref{Bsmallexact}) to vanish as $n\to\infty$.

\item \emph{Appendix~\ref{app:tradeoff} (Trade-Off Function: Full Proof).} Lemma~\ref{lem:bound} shows that $1-a-f(a)$ is maximized at $a=\alpha(1)$, which the Berry-Esseen bound places in $[1/2-\delta,1/2+\delta]$, so the lower bound $f(a)\geq 1-a-\delta$ on this range extends to all $a\in[0,1]$, together with the corresponding $E$-fold composition statement. Its corollary transfers the bound to the mechanism, since $f_{0,\delta}^{\otimes E}$ is symmetric. The false-positive subsection (Proposition~\ref{prop:alpha}) bounds $\alpha^{-1}(a)$. The false-negative subsection (Propositions~\ref{prop:betaas} and~\ref{prop:beta}) first expresses $\beta(h)$ as an integral over an auxiliary standard-Gaussian variable $z$ with integrand $\Phi(\gamma(h,z))$, up to an error of size $\epsilon$. The trade-off function subsection then performs the substitution $h=\alpha^{-1}(a)$: Proposition~\ref{prop:fmu} represents $f(a)=\beta(\alpha^{-1}(a))$, up to the same error $\epsilon$, as a truncated integral with integrand $\Phi(x-\mu(z))$, and Proposition~\ref{prop:tail} bounds the tail truncation error. The Taylor-series subsection (Proposition~\ref{prop:fG}) expands $\Phi(x-\mu(z))$ to second order to obtain the lower bound~(\ref{Bound}); combined with Lemma~\ref{lem:bound} and the preceding propositions, this completes the proof of Theorem~\ref{th:concrete}.

\item \emph{Appendix~\ref{app:sensitivity} (Parameter Sensitivity).} Numerical sensitivity of the required number of rounds $M$ to the noise multiplier $\sigma$ at three target privacy levels (Figure~\ref{fig:sensitivity}), inverting the closed-form two-term approximation of Theorem~\ref{th:concrete}, once with the simplification~\eqref{Bsmallexact} and once with the exact third absolute moment.

\end{itemize}

\textbf{Asymptotic track: proofs of Theorems~\ref{theo:edge} and~\ref{theo:asym}, and Corollary~\ref{cor:ths}.}
\begin{itemize}
\item \emph{Appendix~\ref{app:edgeworth} (Edgeworth Expansion).} Starting from the two-term Edgeworth expansion~(\ref{eq:edge-correct}) of the standardized sum $S_n$, which holds for distributions whose moments are all finite and whose characteristic function tends to $0$ at infinity, as is the case for the lognormal, we introduce the compact coefficients $c_n,d_n$ defined in~(\ref{eq:cd-def}) and prove Theorem~\ref{theo:edge}: for $|x|\leq\kappa_n=o(n^{1/16})$ the polynomial $p_n(x)$ is strictly increasing and invertible, and $F_n(x)=\Phi(p_n(x))+o(1/n)$ uniformly for $|x|\leq\kappa_n$. The appendix then instantiates $c_n,d_n$ explicitly for the lognormal case using Proposition~\ref{prop:moments}.

\item \emph{Appendix~\ref{app:tradeoff_edgeworth} (Trade-Off Function via Edgeworth Expansion).} Its single-epoch analysis mirrors the structure of Appendix~\ref{app:tradeoff}, using the approximation $F_n=\Phi\circ p_n+o(1/n)$ derived from the Edgeworth expansion~(\ref{eq:edge-correct}) in place of the Berry-Esseen error bound, for a constant $\sigma$ that does not depend on $M$. The false-positive subsection (Proposition~\ref{prop:FPinv}) gives the Edgeworth-based form of $\alpha(h)$, controls the inverse threshold $\alpha^{-1}(a)$, and identifies the admissible range $[\tau_M,1-\tau_M]$ of $a$, where $\tau_M=o(1/M)$. The false-negative subsection (Proposition~\ref{prop:betaas1}) approximates $\beta(h)$ by an integral over an auxiliary standard-Gaussian variable $z$ with integrand $\Phi\circ p_{M-1}(\gamma(h,z))$, up to an error $o(1/M)$ that is uniform over the thresholds $h=\alpha^{-1}(a)$ with $a$ in the admissible range. The trade-off function subsection substitutes $h=\alpha^{-1}(a)$ and expands the integrand to second order. Proposition~\ref{prop:asf} bounds $f(a)$ from below by the Gaussian trade-off function $G_{\mu_M+\bar\gamma_M}(a+o(1/M))$ with $\bar\gamma_M=O(\sqrt{\ln M}/M)$, minus a sign-dependent $O(1/M)$ correction and an $o(1/M)$ remainder, uniformly over the admissible range, and gives $f(a)\geq 1-o(1/M)$ for $a\in[0,\tau_M]$. The closing subsection proves Theorem~\ref{theo:asym}. Proposition~\ref{prop:lowerf} turns the expression for $f(a)$ from Proposition~\ref{prop:asf} into the lower bound $f(a)\geq G_{\mu_M+\gamma'_M}(a+\delta'_M)-\delta'_M$ for $a\in[\tau_M,1-\delta'_M]$, with $\delta'_M=o(1/M)$ and $\gamma'_M=O(\sqrt{\ln M}/M)$. Where the correction is negative, the proof of Proposition~\ref{prop:lowerf} absorbs it into a larger parameter of the Gaussian trade-off function. Lemma~\ref{lem:boundF} shows, for a symmetric trade-off function $g$ with fixed point $a^*$ and for $0<\delta<a^*$, that $h(a)=\max\{g(\min\{a+\delta,1\})-\delta,0\}$ satisfies $h\geq g\otimes f_{0,\hat{\delta}}$ with $\hat{\delta}=\max\{\delta/a^*,1-g(\delta)+\delta\}$. The proof of Theorem~\ref{theo:asym} applies it with $g=G_{\mu_M+\gamma'_M}$ and $\delta=\sqrt{\delta'_M/M}=o(1/M)$, and the two properties of $G_\mu$ in Proposition~\ref{prop:Gaus} give $\hat{\delta}_M=o(1/M)$. The proof then establishes $f_M\geq g\otimes f_{0,\hat{\delta}_M}$ on $[0,1]$ in three parts. On $[\tau_M,1-\delta]$ it follows from Proposition~\ref{prop:lowerf}, which gives $f_M\geq h$. On $[1-\delta,1]$ it holds because $g\otimes f_{0,\hat{\delta}_M}$ vanishes on $[1-\hat{\delta}_M,1]$, which contains $[1-\delta,1]$. On $[0,\tau_M]$ it follows from $f_M(\tau_M)\geq 1-\hat{\delta}_M$, because both functions are non-increasing. Since this lower bound is symmetric, Theorem~\ref{thm:fs-dom} makes it a guarantee for the mechanism. Composing over $E=r_M^2\cdot M$ epochs and letting $M\rightarrow\infty$ gives the limiting conclusions of Theorem~\ref{theo:asym}. The appendix closes with the proof of Corollary~\ref{cor:ths}, which derives $\mbox{sep}(f^{\otimes E})=O(\mu\sqrt{E})$ from Theorem~\ref{theo:asym} and Lemma~\ref{lem:sepsub}, a bound on the separation of a tensor product.

\item \emph{Appendix~\ref{app:cproute} (The Exact Poisson Asymptotic, and the Looseness of the $C_p$ Route).} Compares the Gaussian-DP coefficient $\sqrt{e^{1/\sigma^2}-1}$ of Theorem~\ref{theo:asym} with the coefficient obtained from Corollary~5.4 of~\citet{DBLP:journals/corr/abs-1905-02383}, which is the limit of the $C_p$ subsampling bound. The ratio of the latter to the former decreases from $\sqrt{2}$ to $1$ as $\sigma$ grows (Figure~\ref{fig:asymptotic_mu}). The exact asymptotic for Poisson subsampling at the same rate~\citep[Theorem~6 and Lemma~B.1]{Bu2019} has the same coefficient as Theorem~\ref{theo:asym}.
\end{itemize}

\textbf{Scope (Appendix~\ref{app:fixedsize})} We show that the reduction of Appendix~\ref{app:adversarial} does not extend to fixed-size random shuffling. There, the partial sums of the other samples reveal the round of the differing sample, and an explicit neighboring pair has output distributions that the scalar pair $(P^{(M)}_1,P^{(M)}_0)$ does not dominate, exactly at $M=2$ and $\sigma=1$ and, for every fixed $\sigma$, once $M$ is large.

\section{Related Work}
\label{sec:related}

\textbf{Privacy accounting for DP-SGD.}
Tightly accounting for the privacy of DP-SGD has driven a substantial body of work. The original moments accountant of \citet{Abadi2016} has been refined through concentrated and R\'{e}nyi DP~\citep{BunSteinke16,Mironov2019,MironovTZ19}, tight $(\varepsilon,\delta)$-composition~\citep{KairouzOV15}, and divergence-based subsampling-amplification analyses~\citep{BalleBG18,WangBK19}; in parallel, the privacy-loss-distribution and privacy-loss-random-variable frameworks~\citep{KoskelaJH20,GopiLW21,DoroshenkoGKKM22} compute privacy guarantees numerically to within a small error. The $f$-DP framework~\citep{DBLP:journals/corr/abs-1905-02383,Bu2019} characterizes DP via trade-off functions and provides a central limit theorem for composition. Almost all of this work targets Poisson subsampling. For the shuffled mechanism specifically, \citet{Chua2024,Chua2024b} construct explicit adversarial events on the Adaptive Batch Linear Queries (ABLQ) abstraction underlying DP-SGD to upper-bound the trade-off function (equivalently, lower-bound the privacy loss $\delta(\varepsilon)$), in the single-epoch~\citep{Chua2024} and multi-epoch~\citep{Chua2024b} settings; their bound expressions are not closed form in $(\sigma, M, E)$ and are evaluated numerically. The impossibility result of~\citet{ErtanVanDijk2026} is in the same direction: it upper-bounds the trade-off function via an analytically tractable suboptimal test and yields a closed-form threshold $\sigma < 1/\sqrt{2\ln M}$ below which the trade-off function is provably bounded away from $1-a$. We provide the matching closed-form \emph{lower} bound on the trade-off function in the complementary regime, under hypotheses that together force $\sigma \geq \sqrt{3/\ln M}$ (Theorem~\ref{th:concrete}).

\textbf{Random allocation (balls-and-bins) and worst-case reductions.} Placing each sample into an independently and uniformly chosen batch (\emph{random allocation}, or balls-and-bins with $k=1$) has recently been analyzed as a tractable stand-in for reshuffling. \citet{ChuaBB2024} formalize the balls-and-bins scheme for DP-SGD, evaluate its privacy numerically, and prove a tightly dominating pair for it (their Theorem~3.1). After rescaling by $\sigma$ their ordered pair is $(P^{(M)}_1,P^{(M)}_0)$ as explained in Appendix \ref{app:adversarial}; it consists of exactly the two laws of our scalar test in our analysis, in the reverse order to our convention $(H_0,H_1)$. The underlying allocation idea appears earlier as the random check-in model of~\citet{BalleRandomCheckins2020}. \citet{FeldmanShenfeld2025} prove analytic amplification bounds (hockey-stick upper bounds and an RDP/partition accountant; an efficient, accuracy-controlled privacy-loss-distribution \emph{upper-bound} accountant is developed in their follow-up work~\citep{FeldmanShenfeld2026}); for our purposes they establish the worst-case (dominating-pair) reduction (their Claim~C.3 and Theorem~3.1) that a lower bound on the trade-off function requires, and their $k=1$ reduced mixture consists of exactly the two laws of the scalar test we analyze, in the reverse order to our convention (Appendix~\ref{app:adversarial}). \citet{DongChenOzgur2025} give R\'enyi-DP partition-amplification bounds for the same scheme, and \citet{FeldmanMT21} analyze a uniformly permuted sequence of adaptive local randomizers and apply the result to Gaussian-noise SGD with sampling without replacement. Their setting is fixed-size shuffling, and their guarantees take $(\varepsilon,\delta)$ and R\'enyi form. We give a \emph{closed-form, constant-explicit} $f$-DP characterization of the reduced test for random allocation, with transparent bounds one can differentiate and invert to read off parameters, plus an asymptotic Gaussian-DP lower bound under composition. Corollary~4.5 and Lemma~C.2 of~\citet{FeldmanShenfeld2025} do give a closed-form $(\varepsilon,\delta)$-DP bound, with unnamed constants in Lemma~C.2 and a side condition on $\sigma$ that grows with the number of rounds and with $1/\delta$ and is far more restrictive than ours. Our bounds have explicit constants and characterize the whole trade-off curve.

\textbf{Edgeworth expansions and Gaussian-DP limits.} Edgeworth refinements of the $f$-DP composition bound were introduced by \citet{Zheng2020} and developed into a finite-sample accountant by \citet{WangEdgeworth2022}. Both are general composition accountants whose DP-SGD examples use Poisson or fixed-size subsampling, and neither supplies the uniform moderate-deviation statement of Theorem~\ref{theo:edge}, which the proof of Theorem~\ref{theo:asym} needs. For the limiting Gaussian-DP parameter, \citet[Theorem~6 and Lemma~B.1]{Bu2019} give the Poisson coefficient $\sqrt{e^{1/\sigma^2}-1}$, and \citet[Remark~4.4]{FeldmanShenfeld2025} assert that allocation and Poisson agree asymptotically at the level of the privacy profile. 
A central limit theorem shows, under its scaling assumptions, that $f^{\otimes E}$ converges to $G_\mu$. It does not show that $f^{\otimes E}\geq G_\mu$ at finite $E$, so reporting the limiting $\mu$ as a guarantee can be optimistic. \citet{GopiLW21} find that the CLT-based GDP accountant can significantly underreport $\varepsilon$, a point also made by \citet{Gomez2025}. Theorem~\ref{theo:asym} instead bounds $f^{\otimes E}$ from below. 
The same scheme appears for the matrix mechanism in \citet[Lemma~3.2]{ChoquetteChoo2025}, whose dimension-reduction lemma is stated for balls-in-bins batching, and the dominating pair we analyze was studied numerically in the context of shuffling by \citet{KoskelaShuffle2023}. At $k=1$ the first term of the bound of \citet{DongChenOzgur2025} coincides with the expressions of \citet{LiewTakahashi2022} and \citet{FeldmanShenfeld2025} up to combinatorial simplification, although \citet{DongChenOzgur2025} note that \citet{LiewTakahashi2022} use it for the shuffle Gaussian mechanism, where fixed batch sizes make the allocations of different samples dependent and the expression does not apply.  
Concurrently, \citet{DongOzgur2026} study balanced iteration subsampling, in which each sample independently joins $k$ iterations chosen uniformly at random ($k=1$ is balls-and-bins). From expansions of the privacy loss as $\sigma\to0$ and $\sigma\to\infty$ they conclude that this scheme is optimal at both extremes among schemes with a fixed expected participation count, the variance in participation count being what costs Poisson subsampling; at finite $\sigma$ they compare with Poisson numerically via a Monte Carlo accountant. A single epoch of our scheme is their $k=1$ case, while $E$ fresh epochs require exactly one participation per epoch and so differ from their scheme with $k=E$ over all $EM$ iterations. We work at fixed $\sigma$ as $M\to\infty$ and give analytic lower bounds on the whole trade-off function.

\section{Adversarial Model: Full Details}
\label{app:adversarial}

We provide 
the soundness of the hypothesis-testing formulation underlying our privacy analysis. The adversarial model and the reduction to a scalar Gaussian test are due to~\citet{ErtanVanDijk2026}; we recall their derivation here in our own notation for completeness and to make the present paper self-contained.

\textbf{Setup.} We consider two neighboring datasets $d$ and $d'$ and a single epoch of DP-SGD with $M$ rounds. The adversary observes a sample from the multi-dimensional distribution $(x_j)_{j=1}^M$ where, under $d$, all observations are pure noise, and under $d'$, exactly one round (chosen uniformly at random by the random allocation) contains a signal contribution from the differing sample. The formal reduction from the DP-SGD mechanism to this scalar Gaussian observation model follows the framework of~\citet{ErtanVanDijk2026}.

\textbf{Worst-case adversary.} The standard worst-case DP adversary is granted, in addition to the noisy minibatch updates, the per-minibatch partial sum of clipped gradients from all samples other than the potentially differing sample as well as the clipped gradient of the potentially differing sample. (Under \emph{fixed-size} shuffling the minibatch sizes are the constant $m$ and may be granted for free; under random allocation they are random and must \emph{not} be granted,
since they carry information about which round contains the differing sample.) Because the updates released here are \emph{sums} of clipped gradients, this auxiliary information lets the adversary isolate the noisy contribution of the differing sample in each round by a single subtraction, with no minibatch size required. Proposition~4.3 of~\citet{ErtanVanDijk2026} gives the corresponding step for \emph{averaged} updates under fixed-size shuffling, where the minibatch size is a public constant. Because Gaussian noise is isotropic, projecting onto the gradient direction corresponding to what the differing sample's update would contribute, and normalizing (by assuming the worst-case for maximal privacy leakage where the differing sample's gradient norm is equal to the clipping constant $C$) reduces the observation in each round to a scalar: pure noise $N(0,1)$ under $d$, or signal-plus-noise $N(1/\sigma, 1)$ under $d'$ in the one round containing the differing sample.

\textbf{Worst-case of the reduction.} 
The reduction above yields a \emph{lower} bound on the trade-off function only if the pair of distributions it produces is worst-case over all neighboring datasets and all gradients, i.e., only if it is a \emph{dominating pair} in the sense defined below. An impossibility result such as that of~\citet{ErtanVanDijk2026}, which proves an \emph{upper} bound on the trade-off function, needs only a single witness pair of datasets and does not require this property; a privacy guarantee does.

\emph{Mechanism and neighboring datasets.} A mechanism is a randomized map from datasets to outputs. The mechanism $\mathcal{M}$ analyzed in this paper is one epoch of DP-SGD with random allocation: every sample of the dataset is placed in one of the $M$ rounds uniformly at random, independently of the other samples, and round $j\in\{1,\ldots,M\}$ releases
\[
U_j=\sum_{y\in\mathcal{B}_j}\mathrm{clip}_C\big(g_j(y)\big)+N\big(0,\sigma^2C^2I\big),
\]
where $\mathcal{B}_j$ is the set of samples placed in round $j$, $\mathrm{clip}_C$ scales a vector down to norm at most $C$, $\sigma$ is the noise multiplier, and the gradient $g_j(y)$ of sample $y$ may depend on $U_1,\ldots,U_{j-1}$. For a dataset $d$, $\mathcal{M}(d)$ denotes the distribution of $(U_1,\ldots,U_M)$; a single round is the Gaussian mechanism of DP-SGD. Two datasets are neighboring under \emph{zero-out} adjacency if one is obtained from the other by replacing a single sample $x$ by the null sample $\perp$, whose clipped gradient is $0$ in every round. This is the neighboring relation of~\citet{FeldmanShenfeld2025} and~\citet{ChuaBB2024}, who write $\perp$ for the null sample; \citet{FeldmanShenfeld2025} write $\ast$ for the present sample. Throughout, $d'$ denotes the dataset that contains $x$ and $d$ the one that contains $\perp$ in its place, so that $d$ corresponds to $H_0$ and $d'$ to $H_1$. (If $x$ is replaced by another sample, the mean of an update can move by up to $2C$. We do not analyze that neighboring relation.)

\emph{Hockey-stick divergence, privacy profile and domination.} For distributions $P,Q$ on a common space and $\varepsilon\in\mathbb{R}$, the \emph{hockey-stick divergence} is $H_{e^\varepsilon}(P\|Q)=\sup_S\{P(S)-e^\varepsilon Q(S)\}$, the supremum being over measurable sets $S$; if $P$ has density $L$ with respect to $Q$, then $H_{e^\varepsilon}(P\|Q)=\mathbf E_Q[(L-e^\varepsilon)_+]$. The function $\varepsilon\mapsto H_{e^\varepsilon}(P\|Q)$ on $\mathbb{R}$ is the \emph{privacy profile} of the ordered pair $(P,Q)$. A rejection rule $\phi$ for the null hypothesis $Q$ against the alternative $P$ has $\alpha(\phi)=\mathbf E_Q[\phi]$ and $\beta(\phi)=1-\mathbf E_P[\phi]$, hence
\[
H_{e^\varepsilon}(P\|Q)=\sup_{\phi}\big\{\mathbf E_P[\phi]-e^\varepsilon\,\mathbf E_Q[\phi]\big\}=\sup_{a\in[0,1]}\big\{1-T(Q,P)(a)-e^\varepsilon a\big\},
\]
and, because $T(Q,P)$ is convex and continuous, $T(Q,P)(a)=\sup_{\varepsilon\in\mathbb{R}}\{1-e^\varepsilon a-H_{e^\varepsilon}(P\|Q)\}$ for $a\in[0,1]$. We say that $(P,Q)$ \emph{dominates} $(P',Q')$ if $H_{e^\varepsilon}(P'\|Q')\le H_{e^\varepsilon}(P\|Q)$ for every $\varepsilon\in\mathbb{R}$. By the two identities and by Blackwell's theorem~\citep[Theorem~2.10]{DBLP:journals/corr/abs-1905-02383} to $P'=\varphi(P)$ and $Q'=\varphi(Q)$ for  randomized map $\varphi$, this is equivalent to $T(Q',P')\ge T(Q,P)$; in particular it also gives $T(P',Q')\ge T(P,Q)$. Both signs of $\varepsilon$ are needed, since the values $\varepsilon\geq 0$ determine $T(Q,P)$ only where its slope is at most $-1$. A pair $(P,Q)$ is a \emph{dominating pair} for $\mathcal{M}$ if it dominates $(\mathcal{M}(d'),\mathcal{M}(d))$ for all neighboring $d,d'$ as above, and a \emph{tightly} dominating pair if, in addition, $T(\mathcal{M}(d),\mathcal{M}(d'))=T(Q,P)$ for some $d,d'$.

\emph{Random-allocation pair.} Let $(Q_1,Q_0)$ be a pair of distributions for the output of one round, with the sample present and with $\perp$, respectively. Its \emph{random-allocation pair} over $M$ rounds is
\[
\Big(\frac1M\sum_{j=1}^M Q_0^{\otimes(j-1)}\otimes Q_1\otimes Q_0^{\otimes(M-j)},\ Q_0^{\otimes M}\Big).
\]
In the first distribution, one round chosen uniformly at random is drawn from $Q_1$ and every other round from $Q_0$, independently. In the second, every round is drawn from $Q_0$. For $(Q_1,Q_0)=(N(1,\sigma^2),N(0,\sigma^2))$, dividing every coordinate by $\sigma$, a bijection that changes no hockey-stick divergence and no trade-off function, turns this pair into $(P^{(M)}_1,P^{(M)}_0)$ with
\[
P^{(M)}_0=N(0,1)^{\otimes M},\qquad P^{(M)}_1=\frac1M\sum_{j=1}^M N\big(\mathbf e_j/\sigma,\,I_M\big),
\]
where $\mathbf e_j$ is the $j$-th standard basis vector of $\mathbb{R}^M$ and $I_M$ is the $M\times M$ identity matrix; their densities are~\eqref{pdfd} and~\eqref{pdfdprime}.

\begin{theorem}[\citealp{FeldmanShenfeld2025}, Claim~C.3 and Theorem~3.1]\label{thm:fs-dom}
Let $d,d'$ be neighboring datasets, where $d'$ contains the sample $x$ and $d$ contains $\perp$ in its place.
\emph{(a) (Claim~C.3.)} Fix a round $j$, the samples other than $x$ placed in round $j$, and the earlier updates $U_1,\ldots,U_{j-1}$. Then $(N(1,\sigma^2),N(0,\sigma^2))$ dominates the pair of distributions of $U_j/C$ with $x$ and with $\perp$ in round $j$, and the two pairs have the same trade-off function when $\|\mathrm{clip}_C(g_j(x))\|=C$.
\emph{(b) (Theorem~3.1.)} The random-allocation pair of $(N(1,\sigma^2),N(0,\sigma^2))$ over $M$ rounds is a dominating pair for $\mathcal{M}$. Consequently,
\[
T\big(\mathcal{M}(d),\mathcal{M}(d')\big)\;\ge\;T\big(P^{(M)}_0,P^{(M)}_1\big)
\quad\text{and}\quad
T\big(\mathcal{M}(d'),\mathcal{M}(d)\big)\;\ge\;T\big(P^{(M)}_1,P^{(M)}_0\big).
\]
\end{theorem}
Part~(b) is Theorem~3.1 of~\citet{FeldmanShenfeld2025} for one round per sample ($M$ is their number of steps $t$), derived there from~(a) for gradients that depend on the earlier updates; its statement bounds, at every $\varepsilon$, the largest privacy profile over the ordered pairs $(\mathcal{M}(d'),\mathcal{M}(d))$ and $(\mathcal{M}(d),\mathcal{M}(d'))$ of neighboring datasets, and its proof (their Appendix~A) shows that the random-allocation pair dominates $(\mathcal{M}(d'),\mathcal{M}(d))$. For Gaussian noise, (b) is also Theorem~3.1 of~\citet{ChuaBB2024}, which shows in addition that the pair is tightly dominating: both inequalities hold with equality when every sample other than $x$ has clipped gradient $0$ and $x$ has clipped gradient $C\mathbf v$ in every round, for a fixed unit vector $\mathbf v$. The components of the updates orthogonal to $\mathbf v$ are then pure noise under both datasets, and the components along $\mathbf v$, divided by $\sigma C$, are distributed as $P^{(M)}_0$ under $d$ and as $P^{(M)}_1$ under $d'$.

The random-allocation pair $(P^{(M)}_1,P^{(M)}_0)$ consists of \emph{exactly} the two distributions~\eqref{pdfd}--\eqref{pdfdprime} analyzed in this paper, listed with the sample present first as in~\citet{FeldmanShenfeld2025}, which is the reverse of our order $(H_0,H_1)$; by the identities above, its privacy profile determines $f=T(P^{(M)}_0,P^{(M)}_1)$. The exact expression of this profile in~\citet[Lemma~F.1]{FeldmanShenfeld2025} becomes ours under the substitution $Z_i=\sigma X_i$:
\[
H_{e^\varepsilon}\big(P^{(M)}_1\big\|P^{(M)}_0\big)=\mathbf E\Big[\Big(\frac1M\sum_{i=1}^M e^{(2Z_i-1)/(2\sigma^2)}-e^\varepsilon\Big)_+\Big]=\mathbf E\Big[\Big(\frac1M\sum_{i=1}^M e^{X_i/\sigma-1/(2\sigma^2)}-e^\varepsilon\Big)_+\Big],
\]
with $Z_i\sim N(0,\sigma^2)$ and $X_i\sim N(0,1)$ i.i.d.; the average inside the last expectation is the likelihood ratio of~\eqref{pdfdprime} to~\eqref{pdfd} that the tests~\eqref{alpha}--\eqref{beta} threshold. Hence every bound in this paper is a bound on this dominating pair. A lower bound $g$ on $f$ bounds $T(\mathcal{M}(d),\mathcal{M}(d'))$ by~(b). For $E$ epochs, each with its own independently drawn allocation, (b) applies to every epoch conditionally on the updates of the earlier epochs, because it holds for gradients that depend on them. Since domination is preserved under such adaptive composition~\citep[Section~2.3]{FeldmanShenfeld2025}, the $E$-fold product of $(P^{(M)}_1,P^{(M)}_0)$ dominates the pair of output distributions of the $E$ epochs, whose trade-off function is therefore at least $f^{\otimes E}$~\citep[Definition~3.1]{DBLP:journals/corr/abs-1905-02383}.

For random allocation, a data-processing argument shows that handing the adversary the partial sums of the other samples can only help it. Write the round-$j$ observation as $G_j=X_j+\Delta_j\,g_j+Z_j$, where $X_j$ is the partial sum of the clipped gradients of the other samples assigned to the $j$-th minibatch  (this is granted auxiliary information), $\Delta_j$
indicates whether the differing sample is 
allocated to the $j$-th minibatch, $g_j$ is the clipped gradient of the differing sample in round $j$ (this is granted auxiliary information), 
and $Z_j$ is the Gaussian noise. Write $Y_j=\Delta_j g_j+Z_j$, so that the vector $G=Y+X$ coordinatewise. Suppose that $X$ has the same distribution under $H_0$ and $H_1$ and is independent of $(\Delta,g,Y)$. Then $y\mapsto y+X$ is a randomized map that does not depend on the hypothesis, so the probability distribution of $G$ under either hypothesis is the image of the probability distribution of $Y$ under one and the same channel. 
Writing $\mathcal{L}(V\mid H_i)$ for the distribution of a random vector $V$ under hypothesis $H_i$, and using that the hockey-stick divergence is an $f$-divergence and therefore satisfies the data-processing inequality, we obtain, \emph{for every} $\varepsilon\in\mathbb{R}$,
\[
  H_{e^\varepsilon}\big(\mathcal{L}(G\mid H_1)\,\big\|\,\mathcal{L}(G\mid H_0)\big)\;\le\;H_{e^\varepsilon}\big(\mathcal{L}(Y\mid H_1)\,\big\|\,\mathcal{L}(Y\mid H_0)\big),
\]
so the adversary handed $X$ (and free to subtract it) is at least as powerful, and the scalar test on $Y$ is the worst case. It is essential that this holds at every $\varepsilon$ simultaneously. A single averaged quantity such as mutual information orders channels without ordering their trade-off functions, and so could not certify a dominating pair. This step needs the \emph{vector} $X$ to have the same distribution under both hypotheses and to be independent of $(\Delta,g,Y)$. Appendix~\ref{app:fixedsize} shows that it fails under fixed-size shuffling, where $X$ reveals the round of the differing sample. Coordinatewise independence would not suffice, because data processing is applied to the whole channel and therefore needs $X$ independent of the pair $(\Delta,g,Y)$ jointly. The channel argument is exact for fixed, non-adaptive gradients. Under random allocation the complete allocation of the other samples is independent of the differing sample's allocation, so independence holds for the whole background vector $X$. For adaptive queries, $X$ is history-dependent and this literal independence need not hold. We therefore use the channel argument only as intuition in the adaptive setting; the worst-case statement is Theorem~\ref{thm:fs-dom}.

\textbf{Hypothesis test.} Let $H_0$ correspond to dataset $d$ and $H_1$ to $d'$. Let $\phi$ denote a rejection rule: given the observation $(x_j)_{j=1}^M$, the rule outputs $\phi = 1$ if $H_0$ is rejected.\footnote{The rejection rule can be probabilistic and output a probability $\phi$ after which a coin flip with bias $\phi$ decides whether the null hypothesis will be rejected.} The Type~I error (false positive rate) is
$$\alpha(\phi)=\mathbf{E}[\phi \ | \ d]\in [0,1],$$
and the Type~II error (false negative rate) is
$$\beta(\phi)=1-\mathbf{E}[\phi \ | \ d']\in [0,1].$$
The trade-off function 
$$ f = T(P_0^{(M)},P_1^{(M)})$$
is
$$ f(a)=\inf_{\phi'} \{ \beta(\phi') \ : \ \alpha(\phi')\leq a\}.$$
When we wish to make the dependence on $M$ explicit, we write $f_M$. In the statements of Theorems~\ref{th:concrete} and~\ref{theo:asym}, $f$ denotes instead the largest trade-off function for which the mechanism is $f$-DP. Their proofs bound $f_M$ from below by a symmetric trade-off function, which then also bounds $f_M^{-1}$ and hence, by Theorem~\ref{thm:fs-dom}(b), the function $f$ of these theorems.

\textbf{Likelihood computation.} Under $H_0$, all observations are i.i.d.\ standard Gaussian: $P_0^{(M)}$ has probability density function
\begin{equation}
\mathbf{P}[(x_j)_{j=1}^M \ | \ d] =
\prod_{i=1}^M \frac{e^{-x_i^2/2}}{\sqrt{2\pi}}. \label{pdfd}
\end{equation}
Under $H_1$, exactly one round $j$ (uniform over $\{1,\ldots,M\}$) contains the shifted observation: $P_1^{(M)}$ has probability density function
\begin{eqnarray}
\mathbf{P}[(x_j)_{j=1}^M \ | \ d'] &=&
\frac{1}{M} \sum_{j=1}^M 
\frac{e^{-(x_j-1/\sigma)^2/2}}{\sqrt{2\pi}}
\prod_{i=1,\neq j}^M \frac{e^{-x_i^2/2}}{\sqrt{2\pi}} \nonumber \\
&=&
\frac{1}{M} \sum_{j=1}^M 
e^{x_j/\sigma-1/(2\sigma^2)} \cdot
\prod_{i=1}^M \frac{e^{-x_i^2/2}}{\sqrt{2\pi}}. \label{pdfdprime}
\end{eqnarray}

\textbf{Optimal test statistic.} By the Neyman--Pearson lemma, the uniformly most powerful test thresholds the likelihood ratio, which is equivalent to thresholding the sample average of lognormal variables. For $h\geq 0$, define the Type~I and Type~II error rates parametrized by threshold $h$:
\begin{equation} \alpha(h)= \mathbf{P}\left[\left. \frac{1}{M} \sum_{j=1}^M 
e^{x_j/\sigma-1/(2\sigma^2)} > h \  \right| \ d\right], 
\label{alpha}
\end{equation}
\begin{equation} \beta(h)= \mathbf{P}\left[\left. \frac{1}{M} \sum_{j=1}^M 
e^{x_j/\sigma-1/(2\sigma^2)} < h \  \right| \ d'\right]. 
\label{beta}
\end{equation}
The trade-off function is then $f(a) = \beta(\alpha^{-1}(a))$. Computing this trade-off function requires analyzing the distribution of the sample average of lognormal random variables $Y_j=e^{X_j/\sigma-1/(2\sigma^2)}$ with $X_j\sim N(0,1)$, which is the subject of Appendix~\ref{app:berryesseen}.

\textbf{Towards an $f$-DP guarantee.} Trade-off function $f$ is not symmetric. Therefore,  the two orientations of the neighboring adjacency must be distinguished if we wish to conclude a $g$-DP guarantee for a suitable symmetric trade-off function $g$. Under an adjacency relation that is symmetric in the two data sets both orientations occur, so a $g$-DP guarantee must hold for $f$ and for $f^{-1}$ alike, where $f^{-1}$ is the trade-off function obtained by exchanging $P^{(M)}_0$ and $P^{(M)}_1$. The two curves cross: at $M=2$, $\sigma=1$ their values are $0.6927$ and $0.7009$ at $a=0.1$, and $0.0212$ and $0.0190$ at $a=0.9$. Their pointwise minimum is not itself a trade-off function, since it is not convex: at the crossing, which occurs at $a=0.3511$ at likelihood threshold $h=0.98735$, the one-sided slopes are $-0.98735$ and $-1.01282$, so the derivative decreases there. The largest trade-off function $g$ below both curves is the convex lower closure of that minimum. This implies a $g$-DP guarantee.

\section{Application of the Berry-Esseen Theorem}
\label{app:berryesseen}

Let $\{Y_j\}_{j=1}^n$ be independent identically distributed random variables with $\mathbf{E}[Y_j-u]=0$,  
$\mathbf{E}[(Y_j-u)^2]=\bar{\sigma}^2$ and $\mathbf{E}[|Y_j-u|^3]=\rho^3$.
We define the normalized mean of $\{Y_j\}_{j=1}^{n}$ as the random variable
$$ \hat{Y}_n = \frac{\sum_{j=1}^n (Y_j-u)}{\sqrt{\sum_{j=1}^n \bar{\sigma}^2}}=\frac{\sum_{j=1}^n (Y_j-u)}{\sqrt{n} \cdot \bar{\sigma}}=(-u+\frac{1}{n} \sum_{j=1}^n Y_j) \cdot \sqrt{n}/\bar{\sigma}
$$
with cumulative probability distribution $F_n(x)$.\footnote{Distribution $F_n(x)$ is the $n$-fold convolution $F^{n\star}(x\bar{\sigma}\sqrt{n})$, see page 505 of \citet{Feller}, of the cumulative distribution $F$ corresponding to $Y_j-u$, a single random variable minus the expectation.} We define the $k$-th moment of $F$ as $\mu_k=\mathbf{E}[(Y_j-u)^k]$. We notice that $\mathbf{E}[\hat{Y}_n]=\mu_1\cdot \sqrt{n}/\bar{\sigma}=0$ and $\mathbf{E}[\hat{Y}_n^2]=\mu_2/\bar{\sigma}^2=1$.

For the density function of the normal distribution we use the notation
$$ \mathrm{n}(x)=\frac{e^{-x^2/2}}{\sqrt{2\pi}}$$
and by $\Phi(z)=\int_{-\infty}^z \mathrm{n}(x) \ dx$ we denote its cumulative probability distribution.

\begin{theorem}
The Berry-Esseen theorem \citep{Feller,Shevtsova2011} for identically distributed random variables states that there exists a universal constant $B>0$ such that for all $x$ and $n$
$$
|F_n(x)-\Phi(x)|\leq B \cdot \frac{\rho^3}{\sqrt{n}\cdot \bar{\sigma}^3}.
$$
Universal constant $B$ has been upper bounded to $0.4748$ \citep{Shevtsova2011} 
and lower bounded to $0.4097$ \citep{Esseen1956}. 
\end{theorem}

The Berry-Esseen theorem generalizes the classical central limit theorem in that it also 
gives a bound on the maximal error of approximation between the normal cumulative distribution and the true cumulative distribution of the normalized mean. In this sense it characterizes the convergence rate towards the normal cumulative distribution.

\subsection{Sample Average}

We apply the Berry-Esseen theorem to the independent identically distributed (lognormal) variables 
\begin{equation} Y_j = e^{X_j/\sigma -1/(2\sigma^2)} \mbox{ with } X_j\sim N(0,1).
\label{Yj}
\end{equation}

\begin{proposition} \label{prop:moments}
Let $\mu_k = \mathbf{E}[(Y_j-u)^k]$, where $u=\mathbf{E}[Y_j]$.
    We have 
    \begin{eqnarray*}
       u&=& \mathbf{E}[Y_j]=1,\\
      \mu_2 &=& \bar{\sigma}^2
    = e^{1/\sigma^2}-1, \\
    \mu_3 &=&  
(e^{1/\sigma^2}+2)
(e^{1/\sigma^2}-1)^2> 0, \\
\mu_4 &=& 
(e^{4/\sigma^2}+2e^{3/\sigma^2}+3e^{2/\sigma^2}-3)(e^{1/\sigma^2}-1)^2 >0.
    \end{eqnarray*}
In addition $\rho^3=\mathbf{E}[|Y_j-u|^3]$ is equal to
   $$\rho^3=
     \begin{array}{c}
 (1-2\Phi(-5/(2\sigma))) e^{3/\sigma^2} -3(1-2\Phi(-3/(2\sigma)))e^{1/\sigma^2} \\
 +4(1-2\Phi(-1/(2\sigma)))
 \end{array}.$$
\end{proposition}

\noindent
{\em Proof:} 
We compute
\begin{eqnarray}
&&  \int_{-\infty}^b  e^{kx/\sigma -k c/(2\sigma^2)} \mathrm{n}(x) 
 \ dx 
 \nonumber \\
&=& \int_{-\infty}^{b}  \frac{e^{-(x-k/\sigma)^2/2}}{\sqrt{2\pi}} e^{k(k-c)/(2\sigma^2)} \ dx \nonumber \\
&=& \int_{-\infty}^{b -k/\sigma} \frac{e^{-x^2/2}}{\sqrt{2\pi}} e^{k(k-c)/(2\sigma^2)} \ dx \nonumber \\
&=& \Phi(b-k/\sigma) \cdot e^{k(k-c)/(2\sigma^2)}. \label{moments}
\end{eqnarray}
By taking \(b=\infty\) and \(c=1\), we obtain the raw moments
\[
\mathbf{E}[Y_j^k]=e^{k(k-1)/(2\sigma^2)}.
\]
In particular,
\[
u=\mathbf{E}[Y_j]=1
\]
and
\[
\mu_2=\bar{\sigma}^2
=
\mathbf{E}[Y_j^2]-\mathbf{E}[Y_j]^2
=
e^{1/\sigma^2}-1.
\]
Using the raw moments above and expanding central moments around \(u=1\), we obtain
\[
\mu_3
=
\mathbf{E}[(Y_j-1)^3]
=
e^{3/\sigma^2}-3e^{1/\sigma^2}+2
\]
and
\[
\mu_4
=
\mathbf{E}[(Y_j-1)^4]
=
e^{6/\sigma^2}-4e^{3/\sigma^2}+6e^{1/\sigma^2}-3,
\]
which can be factored as given by the proposition.
By using (\ref{moments}) for $c=1$,
we first compute
\begin{eqnarray*}
&&    \int_{-\infty}^{b} (e^{x/\sigma -1/(2\sigma^2)}-1)^3 \cdot \mathrm{n}(x) 
\ dx \\
&=&
 \int_{-\infty}^{b} \left(
 \begin{array}{c} e^{3x/\sigma -3/(2\sigma^2)}-3e^{2x/\sigma -2/(2\sigma^2)} \\
 +3e^{x/\sigma -1/(2\sigma^2)}-1
 \end{array}
 \right) \cdot \mathrm{n}(x) 
 \ dx \\
 &=&
  \begin{array}{c}
 \Phi(b-3/\sigma) e^{3/\sigma^2} -3\Phi(b-2/\sigma)e^{1/\sigma^2} \\
 +3\Phi(b-1/\sigma)-\Phi(b)
 \end{array}.
\end{eqnarray*}
We use this equation to characterize
\begin{eqnarray*}
\rho^3&=&\mathbf{E}[|e^{X_j/\sigma-1/(2\sigma^2)}-1|^3] \\
&=& 
\int_{1/(2\sigma)}^{\infty} (e^{x/\sigma -1/(2\sigma^2)} -1)^3 \cdot \mathrm{n}(x) 
\ dx \\
&& +
\int_{-\infty}^{1/(2\sigma)} (1-e^{x/\sigma -1/(2\sigma^2)})^3 \cdot \mathrm{n}(x) 
\ dx \\
&=&
\int_{-\infty}^{\infty} (e^{x/\sigma -1/(2\sigma^2)} -1)^3 \cdot \mathrm{n}(x) 
\ dx \\
&& -2
\int_{-\infty}^{1/(2\sigma)} (e^{x/\sigma -1/(2\sigma^2)}-1)^3 \cdot \mathrm{n}(x) 
\ dx 
\\
&=&
  \begin{array}{c}
 (1-2\Phi(-5/(2\sigma))) e^{3/\sigma^2} -3(1-2\Phi(-3/(2\sigma)))e^{1/\sigma^2} \\
 +3(1-2\Phi(-1/(2\sigma)))-(1-2\Phi(1/(2\sigma))) 
 \end{array} \\
 &=&
  \begin{array}{c}
 (1-2\Phi(-5/(2\sigma))) e^{3/\sigma^2} -3(1-2\Phi(-3/(2\sigma)))e^{1/\sigma^2} \\
 +4(1-2\Phi(-1/(2\sigma)))
 \end{array}.
\end{eqnarray*}
\hfill $\Box$

\begin{proposition} \label{prop:epsh}
    For 
    \begin{equation} \hat{Y}_n =  (\bar{Y}_n-1)\cdot \frac{\sqrt{n}}{\bar{\sigma}}
\mbox{ with } \bar{Y}_n= \frac{1}{n} \sum_{j=1}^n Y_j, \label{Ybar}
\end{equation}
the Berry-Esseen theorem implies 
\begin{equation}
\mathbf{P}[\bar{Y}_n<h] = \Phi( (h-1) \cdot \sqrt{\frac{n}{e^{1/\sigma^2}-1}}) + \epsilon_n(h) \label{approx}
\end{equation}
with $|\epsilon_n(h)|$ at most
\begin{eqnarray}
  B \cdot \frac{  \begin{array}{c}
 (1-2\Phi(-5/(2\sigma))) e^{3/\sigma^2} -3(1-2\Phi(-3/(2\sigma)))e^{1/\sigma^2} \\
 +4(1-2\Phi(-1/(2\sigma)))
 \end{array}}{\sqrt{n}\cdot (e^{1/\sigma^2}-1)^{3/2}}. \label{upper}
\end{eqnarray}
This is in turn bounded by
\begin{equation}
\leq B \cdot \sqrt{\frac{8}{\pi}}\cdot n^{-1/2} \cdot 
\left(1+\frac{46.115+0.375 \cdot e^{1/\sigma^2}+5.625\cdot e^{3/\sigma^2}}{\sigma^2}\right) 
\label{Blarge}
\end{equation}
and (\ref{upper}) is also bounded by
\begin{equation}
\leq B\cdot
\frac{e^{3/(2\sigma^2)}}{\sqrt{n}}\cdot \frac{1+4\cdot e^{-3/\sigma^2}  }{(1-e^{-1/\sigma^2})^{3/2}}. \hspace{4cm} \mbox{ } \label{Bsmallexact}
\end{equation}
\end{proposition}

The proposition characterizes the convergence rate or approximation error $\epsilon_n$. For large $\sigma$ the bound (\ref{Blarge}) is a simpler, though not numerically tight, approximation of (\ref{upper}) and for small $\sigma$ the bound (\ref{Bsmallexact}) is a tight approximation of (\ref{upper}). 

\noindent
{\em Proof:} 
We have
$$ \hat{Y}_n =  (\bar{Y}_n-1)\cdot \sqrt{\frac{n}{e^{1/\sigma^2}-1}}
\mbox{ with } \bar{Y}_n= \frac{1}{n} \sum_{j=1}^n Y_j,$$
where $\bar{Y}_n$ is a sample average.
We derive
$$ \mathbf{P}[\bar{Y}_n<h]=
\mathbf{P}[\hat{Y}_n<(h-1)\cdot \sqrt{\frac{n}{e^{1/\sigma^2}-1}}]
=
F_n((h-1)\cdot \sqrt{\frac{n}{e^{1/\sigma^2}-1}}).$$
Taking $x=  (h-1)\cdot \sqrt{n}/\sqrt{e^{1/\sigma^2}-1}$ (which may depend on $n$) in the Berry-Esseen theorem yields (\ref{approx})
with $|\epsilon_n(h)|\leq B \cdot \frac{\rho^3}{\sqrt{n} \cdot \bar{\sigma}^{3}}$. From Proposition \ref{prop:moments} we obtain (\ref{upper}).

We want to understand how the convergence rate depends on $\sigma$. We first discuss large $\sigma$. A Taylor series expansion  around $0$ for $x\geq 0$ teaches  $\Phi(x)=1/2+x/\sqrt{2\pi}-x^3/(6\sqrt{2\pi})+ \Phi^{(5)}(\xi) x^5/120$ for some $\xi\in [0,x]$.
We notice that $|\Phi^{(5)}(\xi)|$ is maximized for $\xi=0$ over all real valued $\xi$. Similarly, for $x\geq 0$,
$e^{x}=1+x + e^\xi x^2/2$ for some $\xi\in [0,x]$ and within this range $0\leq e^\xi\leq e^x$.
We use this to bound the nominator in (\ref{upper}). For the denominator we notice that $e^{1/\sigma^2}-1\geq 1/\sigma^2$. Combining and reordering terms
yields convergence rate given by (\ref{Blarge}).

For small $\sigma$ we have the following derivation.
Together with $1-2\Phi(-5/(2\sigma))\leq 1$, $1-2\Phi(-3/(2\sigma))\geq 0$, $1-2\Phi(-1/(2\sigma))\leq 1$ and $(e^{1/\sigma^2}-1)^{3/2}=e^{3/(2\sigma^2)}(1-e^{-1/\sigma^2})^{3/2}$, we bound (\ref{upper}) to (\ref{Bsmallexact}).
\hfill $\Box$

\subsection{\texorpdfstring{Lower Bound on $\sigma$ for the Berry--Esseen Bound}{Lower Bound on sigma for the Berry-Esseen Bound}}

We can derive an f-DP guarantee if $|\epsilon_n(h)|$ in Proposition \ref{prop:epsh} converges to 0 for $n\rightarrow \infty$. 
In order to establish convergence to 0 through the simplification~\eqref{Bsmallexact}, it suffices that
\begin{equation} \sigma = s/\sqrt{\ln n} \ \mbox{ for some constant } \ s>\sqrt{3}
\label{eqsigma}
\end{equation}
in which case
the convergence rate of (\ref{Bsmallexact}) is 
\begin{equation} |\epsilon_n(h)|\leq B\cdot n^{-(1-3/s^2)/2} (1+ O(n^{-1/s^2}))
\label{Bsmall}
\end{equation}
with a negative exponent in $n$ (notice that in the application of the Berry-Esseen theorem $\sigma$ may depend on $n$). 

For completeness, by substituting $\sigma=s/\sqrt{\ln n}$ in (\ref{Blarge}), we have
\begin{equation}
|\epsilon_n(h)|\leq B \cdot \sqrt{\frac{8}{ \pi}} \cdot n^{-1/2} \cdot (1+O(n^{3/s^2}(\ln n)/s^2))
\label{BlargeF}
\end{equation}
describing how the convergence rate depends on large $s\gg \sqrt{\ln n}$.

Imposing the lower bound $s/\sqrt{\ln n}$ on $\sigma$ restricts our DP analysis. For $\sigma=s/\sqrt{\ln n}$ with a constant $s\leq\sqrt{3}$, the Berry-Esseen bound~(\ref{upper}) does not converge to $0$ and we are left with complex formulas which we cannot compute efficiently or interpret (the distribution $F_n$ of the sum of a small number $n$ of identically distributed lognormal random variables is known to be difficult to approximate using simple interpretable formulas). 

\section{Trade-Off Function: Full Proof}
\label{app:tradeoff}

We first prove an auxiliary lemma for $f(a) = \beta(\alpha^{-1}(a))$ which we use in the proof of Theorem~\ref{th:concrete}. The proof of Lemma~\ref{lem:bound} is more generally applicable beyond the random allocation setting to convex trade-off functions that correspond to an optimal test statistic given by the Neyman-Pearson lemma.

\begin{lemma} \label{lem:bound}
   Let $\epsilon$ be defined as the upper bound (\ref{upper}) for $n=M$. Let  $\epsilon\leq\delta\leq 1/2$. If $f(a)\geq 1-a-\delta$ for $a\in [1/2-\delta,1/2+\delta]$, then, for all $a\in[0,1]$,
    $$ f(a) \geq \max\{1-a-\delta,0\} = f_{0,\delta}(a)$$
    and
    $$
    f^{\otimes E}(a) \geq f_{0,\delta}^{\otimes E}(a) = f_{0,1-(1-\delta)^E}(a) =
    \max \{ (1-\delta)^E-a,0\}.
    $$
\end{lemma}

\noindent
{\em Proof of Lemma~\ref{lem:bound}:} 
By the Neyman--Pearson lemma the trade-off curve $f$ is given by $\{(\alpha(\lambda),\beta(\lambda))\}_{\lambda\geq0}$ with $\alpha(\lambda)=\mathbf P[\bar Y_M>\lambda\mid d]$ and $\beta(\lambda)=\mathbf P[\bar Y_M\leq\lambda\mid d']$, so that
$$1-\alpha(\lambda)-\beta(\lambda)=\mathbf P[\bar Y_M>\lambda\mid d']-\mathbf P[\bar Y_M>\lambda\mid d].$$
Let us write $p_0$ and $p_1$ for the densities of $P^{(M)}_0$ and $P^{(M)}_1$ given by (\ref{pdfd}) and (\ref{pdfdprime}). We notice that $\bar Y_M$ is the likelihood ratio $p_1/p_0$. Therefore,
$$ \mathbf P[\bar Y_M>\lambda\mid d']-\mathbf P[\bar Y_M>\lambda\mid d] = \int_{(x_j)_{j=1}^M \mbox{ s.t. } \bar Y_M>\lambda}(\bar Y_M((x_j)_{j=1}^M)-1)\,p_0((x_j)_{j=1}^M) \ \ \mbox{d}x_1 \ldots \mbox{d}x_M,
$$
whose integrand is positive exactly where $\bar Y_M>1$. This shows that its supremum is attained at $\lambda=1$ and we have that (since $\alpha(\lambda)$ is strictly decreasing in $\lambda$, $\alpha$ has an inverse $\alpha^{-1}$)
$$
1-a-f(a) = 1- \alpha(\alpha^{-1}(a)) -\beta(\alpha^{-1}(a))$$
is maximized for $a=\alpha(1)$. Proposition \ref{prop:epsh} together with the definition of $\epsilon$ in this lemma's statement gives $|\alpha(1)-1/2|\leq\epsilon$, which is in turn at most $\delta$. This means that $1-a-f(a)$ is maximized for $a=\alpha(1)\in [1/2-\delta,1/2+\delta]$. Therefore, if $f(a)\geq 1-a-\delta$ for $a\in [1/2-\delta,1/2+\delta]$, then, for all $a\in[0,1]$, $f(a)\geq 1-a-\delta$ and the first part of the lemma follows.
The remainder of the lemma follows from the definition of $f_{0,\delta}$ and the known $E$-fold composition of $f_{0,\delta}$ (see Section~3.3 in \citet{DBLP:journals/corr/abs-1905-02383}).  \hfill $\Box$

As a corollary, we notice that $g=f_{0,\delta}$ and $g=f_{0,\delta}^{\otimes E}$ are symmetric, that is, $g=g^{-1}$.
If $g=g^{-1}$ and $f\geq g$ on $[0,1]$, then applying the generalized inverse, which preserves the order, gives $f^{-1}\geq g^{-1}=g$. So $g$ lies below both curves $f$ and $f^{-1}$, and being convex it lies below their convex lower closure as well. Since $d$ and $d'$ form an arbitrary pair of neighboring datasets and the above argument covers both adjacency orientations, we have that the corresponding DP-SGD mechanism is $g$-DP from Theorem \ref{thm:fs-dom}(b).

\begin{corollary}
    Under the conditions of Lemma \ref{lem:bound}, we have that the DP-SGD mechanism with independent random allocation for $E$ epochs is $f_{0,\delta}^{\otimes E}$-DP.
\end{corollary}

\subsection{False Positive Rate}

We are ready to analyze $\alpha(h)$ of (\ref{alpha}) with the probability density (in the argument above $p_0$) of  $(x_j)_{j=1}^M$ conditioned on $d$ defined by (\ref{pdfd}). By using the definition of $\bar{Y}_M$ in (\ref{Ybar}) with the $Y_j$ defined in (\ref{Yj}) and by using approximation (\ref{approx}) we obtain
\begin{eqnarray}
\alpha(h) &=& \mathbf{P}[\bar{Y}_M>h]= 1- \mathbf{P}[\bar{Y}_M<h]
\nonumber \\
&=&
1- \Phi((h-1)\cdot \sqrt{\frac{M}{e^{1/\sigma^2}-1}}) -\epsilon_M(h), \label{eq:alpha}
\end{eqnarray}
for $h\geq 0$.
Since $\alpha(h)$ is strictly decreasing in $h$, $\alpha$ has an inverse $\alpha^{-1}$. We define
$$ \epsilon'_M(a)=\epsilon_M(\alpha^{-1}(a))$$
and notice that $|\epsilon'_M(a)|$ is at most $(\ref{upper})$ for $n=M$, which in turn is at most $\epsilon= (\ref{upper})$ for $n=M-1$.

For the inverse $\alpha^{-1}(a)=h$ we derive
\begin{eqnarray}
\alpha^{-1}(a)&=& h= 1+\Phi^{-1}(1-\epsilon_M(h)-a) \cdot \sqrt{\frac{e^{1/\sigma^2}-1}{M}} \nonumber \\
&=&1+\Phi^{-1}(1-\epsilon'_M(a)-a) \cdot \sqrt{\frac{e^{1/\sigma^2}-1}{M}},
\label{alphainv}
\end{eqnarray}
where $|\epsilon'_M(a)|\leq \epsilon$. We notice that (\ref{alphainv}) cannot be used to compute $\alpha^{-1}(a)$ exactly since $\epsilon'_M(a)$ depends on $\alpha^{-1}$ and $\epsilon_M$ with only a known bound for the latter.

In our characterization of trade-off function $f(a)$ we only need to analyze a restricted range of $a$.
For some $\kappa>0$ with
$$ \Phi(-\kappa)\leq 1/2-\epsilon - \delta,$$
where $\delta>0$ plays the role of the $\delta$ used in Lemma \ref{lem:bound}, we analyze $f(a)$ for
$$
a \in [1/2-\delta,1/2+\delta]\subseteq [\Phi(-\kappa)+\epsilon, 1-\Phi(-\kappa)-\epsilon].
$$
This implies (notice that $|\epsilon'_M(a)|\leq \epsilon$ and observe that $1-a$ is in the same restricted range as $a$)
$$
1-\epsilon'_M(a)-a \in [\Phi(-\kappa), 1-\Phi(-\kappa)=\Phi(\kappa)].
$$
Applying the inverse of $\Phi$ yields
\begin{equation}
|\Phi^{-1}(1-\epsilon'_M(a)-a)|\leq \kappa. \label{kappa}
\end{equation}

Combining (\ref{alphainv}) and (\ref{kappa}) proves:

\begin{proposition} \label{prop:alpha}  Let $\epsilon= (\ref{upper})$ for $n=M-1$. Let $a\in [1/2-\delta, 1/2+\delta]$ with $\Phi(-\kappa)\leq 1/2-\epsilon-\delta$ for some $\delta>0$.
Then, $|\epsilon'_M(a)|\leq \epsilon$, $|\Phi^{-1}(1-\epsilon'_M(a)-a)|\leq \kappa$, and
\begin{equation}
\alpha^{-1}(a)\in 
[1- \kappa \cdot
\sqrt{\frac{e^{1/\sigma^2}-1}{M}},
1+ \kappa \cdot
\sqrt{\frac{e^{1/\sigma^2}-1}{M}}
].
\label{rangeh}
\end{equation}
\end{proposition}

We will construct a suitable $\delta\geq \epsilon$ (as defined in this subsection) so that Lemma \ref{lem:bound} and its corollary are applicable.

\subsection{False Negative Rate} \label{secFN}

In order to analyze $\beta(h)=1-\mathbf{E}[\phi \ | \ d']$ defined by (\ref{beta}) with the probability density (denoted by $p_1$ earlier) of $(x_j)_{j=1}^M$ conditioned on $d'$ defined by (\ref{pdfdprime}), we 
notice that probability density (\ref{pdfdprime}) is equal to $\frac{1}{M} \sum_{j=1}^M e^{x_j/\sigma -1/(2\sigma^2)}$ times the probability density (\ref{pdfd}) for $d$. This implies
\begin{eqnarray*}
 \beta(h)&=& 1-\mathbf{E}[\phi \ | \ d'] = \mathbf{E}[1-\phi \ | \ d']  \\
 &=& 
 \mathbf{E}[(1-\phi)\cdot \frac{1}{M} \sum_{j=1}^M e^{x_j/\sigma -1/(2\sigma^2)} \ | \ d].
\end{eqnarray*}
Since the probability density (\ref{pdfd}) is symmetric with respect to permutations of indices $j\in \{1,\ldots, M\}$, we conclude that
$$
\beta(h) =  \mathbf{E}[(1-\phi)\cdot  e^{x_1/\sigma -1/(2\sigma^2)} \ | \ d]
$$
Since $1-\phi$ applied to $(x_j)_{j=1}^M$ is equal to $1$ if and only if the average 
$\frac{1}{M} \sum_{j=1}^M e^{x_j/\sigma -1/(2\sigma^2)}$ is $<h$, we have the integral
$$
\beta(h) =
\int_{(x_j)_{j=1}^M \ : \ \frac{1}{M} \sum_{j=1}^M e^{x_j/\sigma -1/(2\sigma^2)}<h}
e^{x_1/\sigma -1/(2\sigma^2)} \cdot \prod_{i=1}^M \frac{e^{-x_i^2/2}}{\sqrt{2\pi}} \ dx, 
$$
where $dx$ denotes the product $dx_1 \ldots dx_M$. Reordering terms gives
$$
 \int_{(x_j)_{j=1}^M \ : \ \frac{1}{M} \sum_{j=1}^M e^{x_j/\sigma -1/(2\sigma^2)}<h}
\frac{e^{-(x_1-1/\sigma)^2/2}}{\sqrt{2\pi}} \cdot \prod_{i=2}^M \frac{e^{-x_i^2/2}}{\sqrt{2\pi}} \ dx.
$$
Notice that the range 
$$ 0<\frac{1}{M} \cdot e^{x_1/\sigma -1/(2\sigma^2)}<h$$
translates into
$$
x_1< \sigma\ln(Mh) + 1/(2\sigma).
$$
We use the variable substitution
$$ 
z=x_1-1/\sigma \mbox{ with } z< \sigma\ln(Mh) - 1/(2\sigma)
$$
(notice that $dx_1=dz$).
This yields the following proposition:

\begin{proposition} \label{prop:betaas}
Let $M\geq 2$ and $h>0$. 
We may characterize $\beta(h)$ (the false negative rate as defined by (\ref{beta})) as the integral
\begin{eqnarray*} \beta(h) &=& \int_{z=-\infty}^{\sigma\ln(Mh) - 1/(2\sigma)}
   \int_{(x_j)_{j=2}^M \ : \ \begin{array}{l} \frac{1}{M-1} \sum_{j=2}^M e^{x_j/\sigma -1/(2\sigma^2)}< \\
 (h-\frac{1}{M}\cdot e^{z/\sigma+1/(2\sigma^2)})\cdot \frac{M}{M-1}
 \end{array}
 }  \\
&&  \prod_{i=2}^M \frac{e^{-x_i^2/2}}{\sqrt{2\pi}} \ dx_2\cdots dx_M \cdot
\frac{e^{-z^2/2}}{\sqrt{2\pi}}  \ dz.
\end{eqnarray*}
\end{proposition}

By using $\alpha$ as defined by (\ref{eq:alpha}) for $M-1$, the inner multi-dimensional integral in Proposition \ref{prop:betaas} is equal to
\begin{eqnarray}
&& 1-\alpha(h\cdot \frac{M}{M-1} - \frac{e^{z/\sigma+1/(2\sigma^2)}}{M-1}) \nonumber \\
&=&
\Phi(
(h\cdot \frac{M}{M-1} - \frac{e^{z/\sigma+1/(2\sigma^2)}}{M-1}
-1)\cdot \sqrt{\frac{M-1}{e^{1/\sigma^2}-1}}) 
\nonumber \\
&& +\epsilon_{M-1}(h\cdot \frac{M}{M-1} - \frac{e^{z/\sigma+1/(2\sigma^2)}}{M-1}),\label{inner}
\end{eqnarray}
where the absolute value of the $\epsilon_{M-1}(.)$ term is at most $\epsilon$ as defined in Proposition \ref{prop:epsh}. Since $|\epsilon_{M-1}(.)|\leq\epsilon$, the outer integral of the $\epsilon_{M-1}(.)$ term over $z<\sigma\ln(Mh)-1/(2\sigma)$ with respect to the density $\mathrm{n}(z)$ has absolute value at most $\epsilon\cdot\Phi(\sigma\ln(Mh)-1/(2\sigma))\leq\epsilon$.
This proves:

\begin{proposition} Let $\epsilon$ be the upper bound~(\ref{upper}) for $n=M-1$, and let $M\geq 2$ and $h>0$. Then,\label{prop:beta}
    $$
|\beta(h) - \int_{-\infty}^{\sigma\ln(Mh) - 1/(2\sigma)} \Phi(\gamma(h,z)) \cdot \frac{e^{-z^2/2}}{\sqrt{2\pi}}  \ dz| \leq \epsilon
$$
with 
$$ \gamma(h,z)=(h\cdot \frac{M}{M-1} - \frac{e^{z/\sigma+1/(2\sigma^2)}}{M-1} -1)\cdot \sqrt{\frac{M-1}{e^{1/\sigma^2}-1}}.$$
\end{proposition}

\subsection{Trade-Off Function}

In order to compute trade-off function $f(a)$ we substitute 
$h=\alpha^{-1}(a)$
with $\alpha^{-1}(a)$ approximated by (\ref{alphainv}) in $\gamma(h,z)$ of Proposition \ref{prop:beta}. This yields
(after reordering terms)
\begin{equation}
\Phi(\Phi^{-1}(1-\epsilon'_M(a)-a) \cdot 
\sqrt{\frac{M}{M-1}}
-
\frac{e^{z/\sigma+1/(2\sigma^2)}-1}{\sqrt{M-1}\sqrt{e^{1/\sigma^2}-1}}).
\label{eqinner}
\end{equation}

This looks like the well known Gaussian trade-off function except for the $\epsilon'_M(a)$ and $\sqrt{M/(M-1)}$ terms. This shows that the trade-off function $f(a)=\beta(\alpha^{-1}(a))$ is approximately an expectation (the  integral in Proposition \ref{prop:beta}) over Gaussian trade-off functions. Since we are interested in a DP guarantee that shows that the trade-off function $f(a)$ is close to the ideal $1-a$, we will use approximations that  lose the link to Gaussian trade-off functions. 

We rewrite (\ref{eqinner})   as
$$
\Phi(x - \mu(z)) \ \mbox{ with } \ x=\Phi^{-1}(1-\epsilon'_M(a)-a) $$
and
$$
\mu(z) =
\frac{e^{z/\sigma+1/(2\sigma^2)}-1}{\sqrt{M-1}\sqrt{e^{1/\sigma^2}-1}}
-
\Phi^{-1}(1-\epsilon'_M(a)-a)
\cdot 
(\sqrt{\frac{M}{M-1}}-1).
$$
Notice that
$$\sqrt{\frac{M}{M-1}}-1=\sqrt{1+\frac{1}{M-1}}-1\leq \frac{1}{2\cdot (M-1)}$$
and Proposition \ref{prop:alpha} tells us
$|\Phi^{-1}(1-\epsilon'_M(a)-a)|\leq \kappa$.
This shows that 
\begin{equation}
\mu(z) =
\frac{e^{z/\sigma+1/(2\sigma^2)}-1}{\sqrt{M-1}\sqrt{e^{1/\sigma^2}-1}}
-\epsilon''_M(a)
\ \ \mbox{ with } \ \ 
|\epsilon''_M(a)|\leq \frac{\kappa}{2\cdot (M-1)}. \label{epspp}
\end{equation}
Substituting these expressions into Proposition \ref{prop:beta} proves:

\begin{proposition} \label{prop:fmu} Let $\epsilon$, $a$ and $\kappa$ satisfy the conditions of Proposition \ref{prop:alpha}. Then,
   $$
|f(a) - \int_{-\infty}^{\sigma\ln(M\alpha^{-1}(a)) - 1/(2\sigma)} \Phi(x-\mu(z)) \cdot \frac{e^{-z^2/2}}{\sqrt{2\pi}}  \ dz| \leq \epsilon
$$
with $x=\Phi^{-1}(1-\epsilon'_M(a)-a)$ and 
$\mu(z)$ has the form (\ref{epspp}) where $\epsilon''_M(a)$ does not depend on $z$.
\end{proposition}

\begin{proposition} \label{prop:tail}
Let $\epsilon$, $a$ and $\kappa$ satisfy the conditions of Proposition \ref{prop:alpha}. Let $x=\Phi^{-1}(1-\epsilon'_M(a)-a)$ and $\mu(z)$ satisfy 
(\ref{epspp})
of Proposition \ref{prop:fmu}. Then,
\begin{eqnarray}
&& \int_{-\infty}^{+\infty}
\Phi(x-\mu(z))  \cdot 
\frac{e^{-z^2/2}}{\sqrt{2\pi}}  \ dz
\nonumber \\
 &\geq & \int_{-\infty}^{\sigma\ln(M\alpha^{-1}(a)) - 1/(2\sigma)}
\Phi(x-\mu(z)) \cdot 
\frac{e^{-z^2/2}}{\sqrt{2\pi}}  \ dz
\nonumber \\
&\geq &
\int_{-\infty}^{+\infty}
\Phi(x-\mu(z))  \cdot 
\frac{e^{-z^2/2}}{\sqrt{2\pi}}  \ dz \label{integral} \\
&& -\Phi(-(\sigma\ln(M\alpha^{-1}(a)) - 1/(2\sigma))). \nonumber
\end{eqnarray}
Let
\begin{equation}
\theta = \left. \frac{e^{-t^2/2}}{t \cdot \sqrt{2\pi}}\right|_{t=\sigma \ln( M (1-\kappa \cdot \sqrt{\frac{e^{1/\sigma^2}-1}{M}})) -\frac{1}{2\sigma}}. \label{fa0}
\end{equation}
For \(\sigma \geq \sqrt{3/\ln M}\) and
\(M\geq \max\{8\kappa^3,3\}\), we have
\[
\Phi(-(\sigma\ln(M\alpha^{-1}(a)) - 1/(2\sigma))) \leq
\theta
\leq
\frac{4.52}
{
2.88\cdot \sqrt{\ln M}
-
2.41/\sqrt{\ln M}
}
\cdot M^{-25/24}.
\]
\end{proposition}

The condition on $M$ suggests to use a maximal $\kappa$ of
\begin{equation} \kappa = (e^{1/\sigma^2}-1)/2
\label{eq:kappasigma}
\end{equation}
which results in the condition
$M\geq (e^{1/\sigma^2}-1)^3$. After substituting $\sigma\geq \sqrt{3/\ln M}$, we see that this condition is satisfied.

\noindent
{\em Proof:}
The integral can be upper and lower bounded as follows:
\begin{eqnarray}
&& \int_{-\infty}^{+\infty}
\Phi(x-\mu(z))  \cdot 
\frac{e^{-z^2/2}}{\sqrt{2\pi}}  \ dz
\nonumber \\
 &\geq & \int_{-\infty}^{\sigma\ln(M\alpha^{-1}(a)) - 1/(2\sigma)}
\Phi(x-\mu(z)) \cdot 
\frac{e^{-z^2/2}}{\sqrt{2\pi}}  \ dz
\nonumber \\
&=&
\int_{-\infty}^{+\infty}
\Phi(x-\mu(z))  \cdot 
\frac{e^{-z^2/2}}{\sqrt{2\pi}}  \ dz
\nonumber \\
&&
-
\int_{\sigma\ln(M\alpha^{-1}(a)) - 1/(2\sigma)}^{+\infty}
\Phi(x-\mu(z))  \cdot 
\frac{e^{-z^2/2}}{\sqrt{2\pi}}  \ dz  \nonumber \\
&\geq&
\int_{-\infty}^{+\infty}
\Phi(x-\mu(z))  \cdot 
\frac{e^{-z^2/2}}{\sqrt{2\pi}}  \ dz \nonumber \\
&&
-
\int_{\sigma\ln(M\alpha^{-1}(a)) - 1/(2\sigma)}^{+\infty}
\frac{e^{-z^2/2}}{\sqrt{2\pi}}  \ dz  \nonumber \\
&=&
\int_{-\infty}^{+\infty}
\Phi(x-\mu(z))  \cdot 
\frac{e^{-z^2/2}}{\sqrt{2\pi}}  \ dz \nonumber \\
&& -\Phi(-(\sigma\ln(M\alpha^{-1}(a)) - 1/(2\sigma))). \nonumber
\end{eqnarray}

By using (\ref{rangeh}), 
we derive
\begin{eqnarray}
&& \sigma\ln(M\alpha^{-1}(a)) - \frac{1}{2\sigma} \nonumber \\
&\geq &
\sigma \ln( M (1-\kappa \cdot \sqrt{\frac{e^{1/\sigma^2}-1}{M}})) -\frac{1}{2\sigma}  \nonumber
\end{eqnarray}
and we have
\begin{eqnarray}
    && \Phi(-(\sigma\ln(M\alpha^{-1}(a)) - 1/(2\sigma))) \nonumber \\
    &\leq &
    \Phi(-(\sigma \ln( M (1-\kappa \cdot \sqrt{\frac{e^{1/\sigma^2}-1}{M}})) -\frac{1}{2\sigma})).\nonumber
\end{eqnarray}
Formula~7.1.13 of~\citet{AbramowitzStegun1964} states for $x\geq 0$,
$$
\frac{1}{x+\sqrt{x^2+2}} < e^{x^2} \int_x^\infty e^{-z^2} \ dz \leq \frac{1}{x+\sqrt{x^2+4/\pi}}
$$
which, after substituting $t=\sqrt{2}\cdot x$ and reordering terms, yields
\begin{equation}
\frac{1}{t+\sqrt{t^2+4}} \sqrt{\frac{2}{\pi}} e^{-t^2/2} < \Phi(-t) \leq \frac{1}{t+\sqrt{t^2+8/\pi}} \sqrt{\frac{2}{\pi}} e^{-t^2/2}. \label{approxPhi}
\end{equation}
For $t>0$, the right hand side is at most $e^{-t^2/2}/(t \cdot \sqrt{2\pi})$. If 
\begin{equation}\sigma \ln( M (1-\kappa \cdot \sqrt{\frac{e^{1/\sigma^2}-1}{M}})) -\frac{1}{2\sigma}>0, \label{condtsigma} \end{equation}
then this yields the upper bound 
$$\Phi(-(\sigma \ln( M (1-\kappa \cdot \sqrt{\frac{e^{1/\sigma^2}-1}{M}})) -\frac{1}{2\sigma}))\leq \theta$$
for $\theta$ defined by
(\ref{fa0}).

We now consider $\sigma=s/\sqrt{\ln M}$ with $s\geq \sqrt{3}$. We derive
\begin{eqnarray}
&& \sigma \ln( M (1-\kappa \cdot \sqrt{\frac{e^{1/\sigma^2}-1}{M}})) -\frac{1}{2\sigma} 
\nonumber \\
&=&
\frac{s}{\sqrt{\ln M}}\cdot \ln( M -\kappa \cdot M^{1/2}
\sqrt{M^{1/s^2}-1}) -\frac{\sqrt{\ln M}}{2 s} \nonumber \\
&\geq & 
\frac{s}{\sqrt{\ln M}}\cdot  \ln( M -\kappa \cdot M^{(1+1/s^2)/2}) -\frac{\sqrt{\ln M}}{2 s}. \label{eqh}
\end{eqnarray}
Notice that if $M\geq 8 \cdot \kappa^3$, then
$$
\kappa \cdot M^{(1+1/s^2)/2}\leq \kappa \cdot M^{2/3}\leq M/2.
$$
From this we obtain the lower bound
$$
\frac{s}{\sqrt{\ln M}}\cdot  \ln( M/2) -\frac{\sqrt{\ln M}}{2 s}.
$$
By noticing that this lower bound is increasing in \(s\), and using
\(s\geq\sqrt{3}\), we have that (\ref{eqh}) is at least
\begin{eqnarray*}
\frac{\sqrt{3}}{\sqrt{\ln M}}\ln(M/2)
-
\frac{\sqrt{\ln M}}{2\sqrt{3}}
&=&
\left(\sqrt{3}-\frac{1}{2\sqrt{3}}\right)\sqrt{\ln M}
-
\frac{\sqrt{3}\ln2}{\sqrt{\ln M}} \\
&=&
\frac{5}{2\sqrt{3}}\sqrt{\ln M}
-
\frac{\sqrt{3}\ln2}{\sqrt{\ln M}}.
\end{eqnarray*}
Set
\[
t=
\frac{5}{2\sqrt{3}}\sqrt{\ln M}
-
\frac{\sqrt{3}\ln2}{\sqrt{\ln M}}.
\]
For \(M\ge3\), \(t>0\) and condition (\ref{condtsigma}) is satisfied. Since $t>0$, using the right hand side of
(\ref{approxPhi}),
\begin{eqnarray}
\Phi(-t)
&\leq&
\frac{1}{t+\sqrt{t^2+8/\pi}}
\sqrt{\frac{2}{\pi}}e^{-t^2/2}
\nonumber \\
&\leq&
\frac{1}{2t}
\sqrt{\frac{2}{\pi}}e^{-t^2/2}.
\label{tail-intermediate}
\end{eqnarray}
Moreover,
\[
2t
=
\frac{5}{\sqrt{3}}\sqrt{\ln M}
-
\frac{2\sqrt{3}\ln2}{\sqrt{\ln M}},
\]
and
\begin{eqnarray*}
t^2
&=&
\frac{25}{12}\ln M
-
5\ln2
+
\frac{3(\ln2)^2}{\ln M} \\
&\geq&
\frac{25}{12}\ln M
-
5\ln2.
\end{eqnarray*}
Consequently,
\[
e^{-t^2/2}
\leq
2^{5/2}M^{-25/24}.
\]
Combining these estimates gives
\begin{eqnarray}
\Phi(-t)
&\leq&
\frac{\sqrt{2/\pi}\,2^{5/2}}
{
\frac{5}{\sqrt{3}}\sqrt{\ln M}
-
\frac{2\sqrt{3}\ln2}{\sqrt{\ln M}}
}
\cdot M^{-25/24}.
\label{fa0small}
\end{eqnarray}
Notice that the upper bound holds for all
\(\sigma\geq \sqrt{3/\ln M}\) with
\(M\geq \max\{8\kappa^3,3\}\).  \hfill $\Box$

\subsection{Taylor Series Expansion and Proof of Theorem \ref{th:concrete}}

\begin{proposition} \label{prop:fG}
Let $\epsilon$, $a$ and $\kappa$ satisfy the conditions of Proposition \ref{prop:alpha}. Let $x=\Phi^{-1}(1-\epsilon'_M(a)-a)$ and $\mu(z)$ satisfy (\ref{epspp}). 
Then,
\begin{eqnarray}
&& \int_{-\infty}^{+\infty}
\Phi(x-\mu(z))  \cdot 
\frac{e^{-z^2/2}}{\sqrt{2\pi}}  \ dz
\nonumber \\
&& 
\geq 1-a- \epsilon
-\left\{ \frac{1}{\sqrt{2\pi}} 
\sqrt{\frac{e^{1/\sigma^2}-1}{M-1}}
 +
 \frac{1}{\sqrt{8\pi}} \frac{\kappa}{M-1}\right\} -
\nonumber \\
 && \hspace{-5mm} \left\{ \frac{e^{2/\sigma^2}+e^{1/\sigma^2}-1}{(M-1)}   +\frac{\kappa\cdot \sqrt{e^{1/\sigma^2}-1}}{(M-1)^{3/2}} + \frac{\kappa^2}{4(M-1)^2} \right\}\frac{1}{\sqrt{8e\pi}}.\nonumber
\end{eqnarray}
\end{proposition}

Substituting (\ref{eq:kappasigma}) and using
\begin{equation}
    \mu = \sqrt{\frac{e^{1/\sigma^2}-1}{M-1}} \label{defmu}
\end{equation}
yields (after reordering and combining terms) the lower bound
\begin{eqnarray}
&& 1-a-\epsilon  
 -\frac{1}{\sqrt{2\pi}} 
\mu \nonumber \\
&&
 -
\left\{ \frac{1}{4\sqrt{2\pi}} + \frac{1}{2\sqrt{2e\pi}}(1+ \frac{e^{1/\sigma^2}}{1-e^{-1/\sigma^2}}) \right\} \mu^2 \nonumber \\
&&
-
 \frac{1}{4\sqrt{2e\pi}}\mu^{3}  - \frac{1}{32\sqrt{2e\pi}}\mu^4.  \label{Bound} 
\end{eqnarray}

\noindent
{\em Proof:}
In order to study the integral (\ref{integral}) we use a Taylor series expansion. 
For all real $x$ and $\mu$, there exists a $\xi$ between $x-\mu$ and $x$ such that
$$
\Phi(x-\mu)= \Phi(x)-\mu \Phi'(x) + \frac{\mu^2}{2} \Phi''(\xi).
$$
We notice that $\Phi''(\xi)=-e^{-\xi^2/2} \cdot \xi/\sqrt{2\pi}$, which absolute value is maximized over the real numbers for $\xi=\pm 1$, i.e., $|\Phi''(\xi)|\leq 1/\sqrt{2e\pi}$. We have 
\begin{equation}
\Phi(x-\mu)= \Phi(x)-\mu \Phi'(x) + \frac{\mu^2}{2} \rho(x,\mu) \ \ \mbox{ with } \ \ |\rho(x,\mu)|\leq \frac{1}{\sqrt{2e\pi}}.
\label{eqTaylorPhi}
\end{equation}
This allows us to derive
\begin{eqnarray}
&&
\int_{-\infty}^{+\infty}
\Phi(x-\mu(z)) \cdot 
\frac{e^{-z^2/2}}{\sqrt{2\pi}}  \ dz \label{form:int} \\
&=& \int_{-\infty}^{+\infty}
\left(
\begin{array}{l}
\Phi(x)-\mu(z) \Phi'(x) \\
+ \frac{\mu(z)^2}{2} \rho(x,\mu(z)) 
\end{array}
\right)
 \cdot
\frac{e^{-z^2/2}}{\sqrt{2\pi}}  \ dz \nonumber 
\end{eqnarray}

We observe
\begin{eqnarray}
\int_{-\infty}^{+\infty}
\Phi(x) \cdot
\frac{e^{-z^2/2}}{\sqrt{2\pi}}  \ dz
&=&\Phi(x) \nonumber \\
&=& \Phi(\Phi^{-1}(1-\epsilon'_M(a)-a)) \nonumber \\
&=& 1-\epsilon'_M(a)-a.\nonumber 
\end{eqnarray}
and together with $|\epsilon'_M(a)|\leq \epsilon$ we obtain
\begin{equation}
  |(1-a) - \int_{-\infty}^{+\infty} \Phi(x) \cdot \frac{e^{-z^2/2}}{\sqrt{2\pi}}  \ dz|\leq \epsilon. \label{fa1}
\end{equation}

We notice that 
(\ref{epspp})
implies
\begin{eqnarray*}
&& |\int_{-\infty}^{+\infty}\mu(z)\cdot
\frac{e^{-z^2/2}}{\sqrt{2\pi}}  \ dz  
\\ && \hspace{1cm} 
-\int_{-\infty}^{+\infty} \frac{e^{z/\sigma+1/(2\sigma^2)}-1}{\sqrt{M-1}\sqrt{e^{1/\sigma^2}-1}}\cdot
\frac{e^{-z^2/2}}{\sqrt{2\pi}}  \ dz| \\
&\leq& \int_{-\infty}^{+\infty}\frac{\kappa}{2\cdot (M-1)}\cdot
\frac{e^{-z^2/2}}{\sqrt{2\pi}}  \ dz = \frac{\kappa}{2\cdot (M-1)}. 
\end{eqnarray*}
Applying (\ref{moments}) for $b=\infty$ and $c=-1$ with $k=1$ and $k=0$ yields after simplification 
$$\int_{-\infty}^{+\infty} \frac{e^{z/\sigma+1/(2\sigma^2)}-1}{\sqrt{M-1}\sqrt{e^{1/\sigma^2}-1}}\cdot
\frac{e^{-z^2/2}}{\sqrt{2\pi}}  \ dz = \sqrt{\frac{e^{1/\sigma^2}-1}{M-1}}.$$
Together with $\Phi'(x)=e^{-x^2/2}/\sqrt{2\pi}\leq 1/\sqrt{2\pi}$ we obtain
\begin{eqnarray}
&& |\int_{-\infty}^{+\infty}\mu(z)\Phi'(x)\cdot
\frac{e^{-z^2/2}}{\sqrt{2\pi}}  \ dz  
-\sqrt{\frac{e^{1/\sigma^2}-1}{M-1}}\cdot \Phi'(x)| \nonumber \\
&\leq&  \frac{\kappa}{2\cdot (M-1)}\cdot \Phi'(x)\leq \frac{1}{\sqrt{8\pi}} \frac{\kappa}{M-1},\nonumber
\end{eqnarray}
hence,
\begin{eqnarray}
&&   |\int_{-\infty}^{+\infty}\mu(z)\Phi'(x)\cdot
\frac{e^{-z^2/2}}{\sqrt{2\pi}}  \ dz |  \nonumber \\
&\leq &
 \frac{1}{\sqrt{2\pi}} 
\sqrt{\frac{e^{1/\sigma^2}-1}{M-1}}
 +
 \frac{1}{\sqrt{8\pi}} \frac{\kappa}{M-1}. \label{fa3}
\end{eqnarray}

The bound (\ref{epspp}) leads to the following derivation: First,
\begin{eqnarray*}
\mu(z)^2 &=& (\frac{e^{z/\sigma+1/(2\sigma^2)}-1}{\sqrt{M-1}\sqrt{e^{1/\sigma^2}-1}}
-\epsilon''_M(a))^2
\\ 
&=& \frac{e^{2z/\sigma+1/\sigma^2}-2e^{z/\sigma+1/(2\sigma^2)}+1}{(M-1)(e^{1/\sigma^2}-1)} \\
&& -2\epsilon''_M(a) \cdot \frac{e^{z/\sigma+1/(2\sigma^2)}-1}{\sqrt{M-1}\sqrt{e^{1/\sigma^2}-1}}  +\epsilon''_M(a)^2. 
\end{eqnarray*} 
Taking the expectation of $\mu(z)^2$ with respect to the normal distribution (we use (\ref{moments}) for $b=\infty$, $c=-1$ and $k=0$, $1$ and $2$) yields 
\begin{eqnarray}
&& \frac{e^{3/\sigma^2}-2e^{1/\sigma^2}+1}{(M-1)(e^{1/\sigma^2}-1)}   - 2\epsilon''_M(a) \cdot \frac{e^{1/\sigma^2}-1}{\sqrt{M-1}\sqrt{e^{1/\sigma^2}-1}} + \epsilon''_M(a)^2
\nonumber \\
&=&
\frac{e^{2/\sigma^2}+e^{1/\sigma^2}-1}{(M-1)} 
-2\epsilon''_M(a)\cdot \frac{ \sqrt{e^{1/\sigma^2}-1}}{\sqrt{M-1}} + \epsilon''_M(a)^2. \nonumber
\end{eqnarray}
Second, (\ref{epspp}) states $|\epsilon''_M(a)|\leq \kappa/(2(M-1))$ and from this we obtain the upper bound
\begin{equation}
 \frac{e^{2/\sigma^2}+e^{1/\sigma^2}-1}{(M-1)}   
 + \frac{ \kappa \cdot\sqrt{e^{1/\sigma^2}-1}}{(M-1)^{3/2}} +  \frac{\kappa^2}{4(M-1)^2}. \label{fa4} 
\end{equation}
Together with $|\rho(x,\mu(z))|\leq 1/\sqrt{2e\pi}$ this proves
\begin{equation}
|\int_{-\infty}^{+\infty}
 \frac{\mu(z)^2}{2} \rho(x,\mu(z)) 
 \cdot
\frac{e^{-z^2/2}}{\sqrt{2\pi}}  \ dz | \leq (\ref{fa4})/ \sqrt{8 e \pi}. \label{eq:mu2}
\end{equation}

Now we are ready to lower bound (\ref{form:int}) by using (\ref{fa1}), (\ref{fa3}), (\ref{fa4}) and (\ref{eq:mu2}) together with the triangle inequality. 
This proves Proposition \ref{prop:fG}. \hfill $\Box$

\noindent
{\em Proof of Theorem \ref{th:concrete}:} We are ready to prove our first main result. We combine Propositions \ref{prop:fG}, \ref{prop:tail} and \ref{prop:fmu}, where we use $\kappa$ satisfying (\ref{eq:kappasigma}) leading to (\ref{Bound}) with $\mu$ defined in (\ref{defmu}) as the lower bound in Proposition \ref{prop:beta}.
This leads to 
\begin{equation}
f(a)\geq (\ref{Bound}) - \frac{4.52}{2.88\sqrt{\ln M}-2.41/\sqrt{\ln M}}\,M^{-25/24} - \epsilon    \label{boundf}
\end{equation} 
for $\epsilon$, $a$ and $\kappa$ satisfying the conditions of Proposition \ref{prop:alpha} where $\delta$ is defined as $1-a$ minus the right hand side of (\ref{boundf}) which is $\geq \epsilon$ as Lemma \ref{lem:bound} and its corollary require, and  the requirement that follows gives $\delta\leq 1/2-\epsilon-\Phi(-\kappa)<1/2$.
In order to satisfy the conditions of Proposition \ref{prop:alpha} for  $\kappa$ given by (\ref{eq:kappasigma}), we need
$$ \Phi(-(e^{1/\sigma^2}-1)/2)=\Phi(-\kappa)\leq 1/2 -\epsilon -\delta.$$
This requirement translates into
\begin{eqnarray}
&& 3\epsilon + 
\frac{1}{\sqrt{2\pi}} 
\mu \nonumber \\
&&
 +
\left\{ \frac{1}{4\sqrt{2\pi}} + \frac{1}{2\sqrt{2e\pi}}(1+ \frac{e^{1/\sigma^2}}{1-e^{-1/\sigma^2}}) \right\} \mu^2 \nonumber \\
&&
+
 \frac{1}{4\sqrt{2e\pi}}\mu^{3}  +\frac{1}{32\sqrt{2e\pi}}\mu^4 \nonumber \\
&& +\frac{4.52}{2.88\cdot \sqrt{\ln M}-2.41/\sqrt{\ln M}} \cdot M^{-25/24} \nonumber \\
&\leq & 1/2- \Phi(-(e^{1/\sigma^2}-1)/2).
 \label{BoundM} 
\end{eqnarray}
For $\epsilon$ we use upper bound (\ref{Bsmallexact}) with $n=M-1$:
$$
\epsilon \leq B\cdot
\frac{e^{3/(2\sigma^2)}}{\sqrt{M-1}}\cdot \frac{1+4\cdot e^{-3/\sigma^2}  }{(1-e^{-1/\sigma^2})^{3/2}}
=
B\cdot
e^{1/\sigma^2}\cdot \frac{1+4\cdot e^{-3/\sigma^2}  }{(1-e^{-1/\sigma^2})^{2}} \cdot \mu
$$
Requirement (\ref{BoundM}) with $\epsilon$ replaced by the above upper bound implicitly lower bounds $M$. Notice that this implies $\sigma\geq \sqrt{3/\ln M}$ since for $\sigma\leq\sqrt{3/\ln M}$, three times the above upper bound is at least $3Be^{3/(2\sigma^2)}/\sqrt{M-1}\geq 3B\geq 1/2$ (for $M\geq 3$ the other terms on the left-hand side of~(\ref{BoundM}) are positive, and $M=2$ is excluded below). So, we do not need to repeat this condition in our first main result. Also the condition $M\geq 3$ is implicitly satisfied ($M=1$ leads to $\mu=\infty$; $M=2$ implies that three times the above upper bound   is equal to
$3B e^{3/(2\sigma^2)} (1+4e^{-3/\sigma^2})/(1-e^{-1/\sigma^2})^{3/2}\geq 3B\,(e^{2/\sigma^2}/(e^{1/\sigma^2}-1))^{3/2}\geq 24B$ because $e^{2/\sigma^2}-4(e^{1/\sigma^2}-1)=(e^{1/\sigma^2}-2)^2\geq 0$, and the {\em last} summand on the left-hand side of~(\ref{BoundM}) is at least $-4.42$ for $M=2$, so the left-hand side of~(\ref{BoundM}), with $\epsilon$ replaced by the above upper bound, is at least $24B-4.42>1/2$).
\hfill $\Box$

\section{Parameter Sensitivity (Berry--Esseen Track)}
\label{app:sensitivity}

Figure~\ref{fig:sensitivity} reports the required number of rounds $M$ as a function of the noise multiplier $\sigma$, for three target privacy levels $\delta\in\{10^{-2},10^{-3},10^{-4}\}$ at $E=1$. The solid curves invert the two-term closed-form approximate lower bound (\ref{bounddelta}) on $\delta$ from Theorem~\ref{th:concrete}, which agrees with the solution of the precise (multi-term) lower bound on $\delta$ as stated in Theorem~\ref{th:concrete} itself to within $\sim\!0.1\%$ across the $\sigma$ shown (cf.\ Table~\ref{tab:parameters}). They exhibit a U-shape. This is a property of the bound being inverted. 
Its shape is determined by the simplification~\eqref{Bsmallexact}, which is designed for small $\sigma$ and whose dependency on large $\sigma$ mainly comes from the factor $(1-e^{-1/\sigma^2})^{-2}$ in the simplification  and grows like $\sigma^4$.

The dashed curves use the exact third absolute moment, that is, the Berry-Esseen bound~\eqref{upper} without its simplification~\eqref{Bsmallexact}. Theorem~\ref{th:concrete} holds in this form, provided that $\sigma\geq\sqrt{3/\ln M}$ and $M\geq 3$ (Section~\ref{sec:interpretation}). The Berry--Esseen term of the lower bound on $\delta$ is then $2B\rho^3\mu/(e^{1/\sigma^2}-1)^2$, where $\rho^3=\mathbf{E}[|Y_j-1|^3]$ is the third absolute central moment of Proposition~\ref{prop:moments}. Like the solid curves, the dashed curves keep the two leading terms of the lower bound on $\delta$ and invert
\[
\delta=\left(\frac{2B\rho^3}{(e^{1/\sigma^2}-1)^{2}}+\frac{1}{\sqrt{2\pi}}\right)\mu .
\]
This closed form agrees with the exact (multi-term) solution of Theorem~\ref{th:concrete} for the exact third absolute moment to within $0.3\%$. Every plotted point satisfies $\sigma\geq\sqrt{3/\ln M}$, $M\geq 3$ and the validity condition~\eqref{condth} with the exact third absolute moment.

For small $\sigma$ the simplification is tight, and the solid and dashed curves coincide (at $\sigma=0.5$ they differ by $0.2\%$). As $\sigma$ grows, the dashed curves keep decreasing, so the U-shape of the solid curves is an artifact of the simplification. As $\sigma\to\infty$, the ratio $\rho^3/(e^{1/\sigma^2}-1)^{3/2}$ tends to $\sqrt{8/\pi}$, which is the third absolute moment of the standard normal distribution. For fixed $M$, the Berry--Esseen term $2B\rho^3\mu/(e^{1/\sigma^2}-1)^2=2B\rho^3/\big((e^{1/\sigma^2}-1)^{3/2}\sqrt{M-1}\big)$ therefore tends to $2B\sqrt{8/\pi}/\sqrt{M-1}$ (this also follows directly from (\ref{Blarge})), and the dashed curves tend to $M=1+(2B\sqrt{8/\pi}/\delta)^2$ (for example, this is equal to $2.3\times10^4$ for $\delta=10^{-2}$). We notice that Theorem~\ref{th:concrete} does not cover this limit, because along the dashed curves the validity condition~\eqref{condth} holds only for $\sigma$ up to about $3.75$, $11.6$ and $36.5$ for $\delta=10^{-2}$, $10^{-3}$ and $10^{-4}$, respectively.

\textbf{Source of the Berry-Esseen slack.}
The lognormal variables $Y_j$ are skewed, so the error of the normal approximation is of order $1/\sqrt{M}$, as in the Berry-Esseen theorem. At the likelihood threshold $h=1$, where $1-a-f(a)$ is maximal, the approximation error $\epsilon_M(1)$ in~\eqref{approx} is close to the skewness term $c_M/\sqrt{2\pi}$ of Theorem~\ref{theo:edge}. The bound of Theorem~\ref{th:concrete} is nevertheless loose, for two reasons. First, the Berry-Esseen bound~\eqref{upper} holds for every distribution with the same variance and third absolute central moment, and at the settings of Table~\ref{tab:parameters} it is $7.2$ to $9.5$ times larger than $|\epsilon_M(1)|$. Second, the proof of Theorem~\ref{th:concrete} bounds the approximation errors of $\alpha$ and $\beta$ separately, whereas the two errors carry the same skewness term and cancel in $1-a-f(a)$. At $h=1$ they agree to within $0.01\%$, which explains $\delta_{\mathrm{exact}}\approx\mu/\sqrt{2\pi}$. The analysis of Section~\ref{sec:asymptotic} and Appendix \ref{app:edgeworth} keeps the skewness term under both hypotheses through Theorem~\ref{theo:edge}, whose error is $o(1/n)$, so the lower bound of Theorem~\ref{theo:asym} has no Berry-Esseen term. Making this lower bound concrete (not just asymptotic) is the open problem stated in Section~\ref{sec:asymptotic}.

\section{Edgeworth Expansion}
\label{app:edgeworth}

The Berry-Esseen theorem provides an upper bound on the speed of convergence of the normalized mean of random variables to a normal distribution which only depends on the starting distribution $F$ through the first two moments and third absolute moment. The Edgeworth expansion provides a better asymptotic estimate, which we will use for our asymptotic analysis. 
The speed of convergence depends in a complex way on $F$ (not only on the moments of $F$) and cannot be used to obtain concrete (non-asymptotic) convergence results for smaller $n$.

We notice that all moments $\mu_k$, $k\geq 1$, exist for the lognormal distribution (in which we are interested for our hypothesis testing problem). Also the lognormal distribution has a characteristic function $\phi(\xi)$ with the property that its limit for $\xi \rightarrow \pm \infty$ tends to zero, see Chapter~XV in \citet{Feller}, and for this reason the Edgeworth expansion of the lognormal distribution exists, see Chapter XVI in \citet{Feller}, in particular Theorem 3 on page 515 (which requires $\mu_1=0$ and $\mu_2=\bar{\sigma}^2$). For the construction of polynomials $R_k(x)$,  pages 505-506 define Hermite polynomials $H_k(x)$ and polynomials $P_k(x)$ are defined/constructed on pages 508-509 which needs Lemma 2 of page 485 for $\xi=0$. We only need the Edgeworth expansion for two terms:

Let $F(x)$ be a cumulative distribution for which all moments $\mu_k$, $k\geq 1$, exist with $\mu_1=0$. Suppose that the corresponding characteristic function $\phi(\xi)$ tends to $0$ for $\xi \rightarrow \pm \infty$.
Let \(X_1,X_2,\ldots\) be i.i.d. random variables with cumulative distribution $F(x)$. By using $\mathbb E[X_1]=\mu_1=0$ and $\mathbb E[X_1^2]=\mu_2=\bar\sigma^2>0$,
define the standardized sum (as before)
\[
S_n=\frac{X_1+\cdots+X_n}{\bar\sigma\sqrt n},
\qquad
F_n(x)=\mathbb P(S_n\le x).
\]
Write
\[
\lambda_3=\frac{\mu_3}{\bar\sigma^3},
\qquad
\lambda_4=\frac{\mu_4-3\bar\sigma^4}{\bar\sigma^4},
\]
where \(\mu_k=\mathbb E[X_1^k]\) are central moments because \(\mathbb E[X_1]=0\).
 As $n\rightarrow \infty$,
\begin{equation}
\label{eq:edge-correct}
F_n(x)=\Phi(x)+\mathrm{n}(x)\left[
\frac{\lambda_3}{6\sqrt n}(1-x^2)
+
\frac{1}{n}\left\{
\frac{\lambda_4}{24}(3x-x^3)
+
\frac{\lambda_3^2}{72}(-x^5+10x^3-15x)
\right\}
\right]+o(n^{-1})
\end{equation}
with convergence uniformly in $x$.

For compact notation define
\begin{equation}\label{eq:cd-def}
c_n=\frac{\lambda_3}{6\sqrt n}=\frac{\mu_3}{6\mu_2^{3/2}\sqrt n} \ \ \mbox{ and } \ \ 
d_n=\frac{\lambda_4}{24n}=\frac{\mu_4-3\mu_2^2}{24\mu_2^2 n}.
\end{equation}
Then \((\ref{eq:edge-correct})\) is
\begin{equation}\label{eq:edge-cd}
F_n(x)=\Phi(x)+\mathrm{n}(x)\left[
c_n(1-x^2)+d_n(3x-x^3)+\frac{c_n^2}{2}(-x^5+10x^3-15x)
\right]+o(n^{-1}).
\end{equation}

For ease of readability, we repeat Theorem \ref{theo:edge}:
 {\em   Let
    \begin{eqnarray*}
p_n(x)&=& 
x+c_n(1-x^2)+d_n(3x-x^3)+c_n^2(4x^3-7x)
\end{eqnarray*}
with $c_n$ and $d_n$ as defined by (\ref{eq:cd-def}). Then, if $\mu_3\geq 0$, $p_n(-x)\leq -p_n(x)$ for $1\leq x$. 
Let $1\leq \kappa_n=o(n^{1/16})$.
Then,  restricted to $|x|\leq \kappa_n$ and for sufficiently large $n$, polynomial $p_n(x)$ is strictly increasing and invertible. Furthermore, for $n\rightarrow \infty$
\begin{eqnarray*}
\Phi(p_n(x))&=&  F_n(x) + o(1/n)
\end{eqnarray*}
uniformly in $|x|\leq \kappa_n$. }

The theorem states that for $1\leq \kappa_n=o(n^{1/16})$, we have the property that for all $\epsilon>0$, there exists an $N_\epsilon$ such that for all $n\geq N_\epsilon$ and for all $|x|\leq \kappa_n$, we have
$$
|F_n(x)-\Phi(p_n(x))|\leq \epsilon/n.
$$
The convergence rate in terms of $N_\epsilon$ depends in a complex way on  distribution $F$ (and not just on its moments).

We will fix $\epsilon>0$ and write $\epsilon_n=\epsilon/n$. We have for $n\geq N_\epsilon$ and $|x|\leq \kappa_n$
\begin{equation}
F_n(x)=\Phi \circ p_n(x) + \epsilon_n(x),
\label{Papprox}
\end{equation}
where $\epsilon_n(x)$ indicates an error term whose absolute value is $\leq \epsilon_n$. 

In order to make (\ref{Papprox}) exact, $\epsilon_n(x)$ is a function of $x$. However, we will only work with the $\epsilon_n$ as an upper bound of the absolute value of $\epsilon_n(x)$. For this reason, we will simply write $\epsilon_n$ and remember that it indicates some suitable value whose absolute value is at most $\epsilon_n$. Writing $+\epsilon_n$ or $-\epsilon_n$ will therefore have the same meaning. Furthermore, we will write $|\epsilon_n|$ to mean the upper bound on the absolute value of $\epsilon_n$. So, we replace $\epsilon_n(x)$ by $\epsilon_n$ and replace $\epsilon_n$ by $|\epsilon_n|$ in our notation.

We have
$$
1-a=F_n(x) \ \ \Rightarrow  \ \ x= p_n^{-1} \circ \Phi^{-1}(1-a-\epsilon_n).
$$

Notice that the Berry-Esseen theorem yields
the same result, but for
$$
N_\epsilon=1, \kappa_n=\infty, p_n(x)=x, \epsilon_n = B\cdot \frac{\mathbb E[|X|^3]}{\sqrt{n}\cdot \mu_2^{3/2}}=O(1/\sqrt{n}).
$$

\noindent
{\em Proof:}
By its definition  every term of $p_n(x)$ except $c_n(1-x^2)$ is odd, so $p_n(-\kappa)+p_n(\kappa)=2c_n(1-\kappa^2)$. If $\mu_3\geq 0$, then $c_n\geq 0$, and if $x\geq 1$, we have that $p_n(-x)\leq -p_n(x)$. We continue to prove that $p_n(x)$ is increasing and invertible in a suitable range:
First,
\begin{equation}
 p_n'(x)= 1-2c_nx+d_n(3-3x^2)+c_n^2(12x^2-7) . \label{derp}
\end{equation}
Because \(c_n=O(n^{-1/2})\), \(d_n=O(n^{-1})\), and \(\kappa_n=o(n^{1/16})\), uniformly for \(|x|\le\kappa_n\),
\[
p_n'(x)=1+O(n^{-1/2}\kappa_n+n^{-1}\kappa_n^2)=1+o(1).
\]
Thus \(p_n'(x)>0\) on \([-\kappa_n,\kappa_n]\) for all sufficiently large \(n\).

Set
\[
u_n(x)=p_n(x)-x=c_n(1-x^2)+d_n(3x-x^3)+c_n^2(4x^3-7x).
\]
Uniformly for \(|x|\le\kappa_n\),
\[
u_n(x)=O(n^{-1/2}\kappa_n^2+n^{-1}\kappa_n^3)= O(n^{-3/8}) =o(1).
\]
Taylor's theorem gives, for some \(\xi\) between \(x\) and \(x+u_n(x)\),
\begin{eqnarray}
\Phi(p_n(x))
&=& \Phi(x+u_n(x)) \notag\\
&=& \Phi(x)+u_n(x)\cdot \mathrm{n}(x)-\frac{u_n(x)^2}{2} \cdot x \mathrm{n}(x)
+\frac{u_n(x)^3}{6} \cdot (\xi^2-1) \mathrm{n}(\xi).\label{eq:taylor}     
\end{eqnarray}
The function \((\xi^2-1)\mathrm{n}(\xi)\) is bounded on \(\mathbb R\). Also, uniformly for \(|x|\le\kappa_n\),
\[
|u_n(x)|^3=
O((n^{-3/8})^3)=O(n^{-9/8})=o(n^{-1}).
\]
Hence the third-order remainder in \((\ref{eq:taylor})\) is \(o(n^{-1})\).

Next, uniformly on \(|x|\le\kappa_n\),
\begin{eqnarray*}
&& x\left(u_n(x)^2-c_n^2(1-x^2)^2\right) \nonumber \\ &=& x\left(
2 c_n(1-x^2)\cdot [d_n(3x-x^3)+c_n^2(4x^3-7x)] +[d_n(3x-x^3)+c_n^2(4x^3-7x)]^2 \right)\nonumber \\
&=& O(n^{-3/2}\kappa_n^6+
n^{-2}\kappa_n^7
)=
o(n^{-1}). \label{eq:u2-reduction}
\end{eqnarray*}
Therefore
\[
-\frac{u_n(x)^2}{2}x=-\frac{c_n^2}{2}x(1-x^2)^2+o(n^{-1}).
\]
Substituting this into \((\ref{eq:taylor})\),
$$ \Phi(p_n(x))
=\Phi(x)+\mathrm{n}(x)\Bigg[
c_n(1-x^2)+d_n(3x-x^3)+c_n^2(4x^3-7x) 
-\frac{c_n^2}{2}x(1-x^2)^2
\Bigg]+o(n^{-1}).
$$
The \(c_n^2\)-coefficient simplifies to the corrected Edgeworth polynomial:
\begin{align*}
4x^3-7x-\frac{x(1-x^2)^2}{2}
&=4x^3-7x-\frac{x-2x^3+x^5}{2}\\
&=-\frac{x^5}{2}+5x^3-\frac{15x}{2}\\
&=\frac{1}{2}(-x^5+10x^3-15x).
\end{align*}
Consequently,
\[
\Phi(p_n(x))=\Phi(x)+\mathrm{n}(x)\left[
c_n(1-x^2)+d_n(3x-x^3)+\frac{c_n^2}{2}(-x^5+10x^3-15x)
\right]+o(n^{-1}).
\]
Comparing this with \((\ref{eq:edge-cd})\) proves  the theorem.
\hfill $\Box$

Notice that for the lognormal distribution that we study, see Proposition \ref{prop:moments},
we have 
\begin{eqnarray}
c_n&=& \frac{\mu_3}{6\bar{\sigma}^3\cdot \sqrt{n}}
= \frac{1}{6\sqrt{n}}(e^{1/\sigma^2}+2)\sqrt{e^{1/\sigma^2}-1}, \label{cn} \\
d_n &=& \frac{\mu_4-3\bar{\sigma}^4}{24\bar{\sigma}^4\cdot n} \nonumber  \\
&=& \frac{1}{24n}
\left(e^{4/\sigma^2}+2e^{3/\sigma^2}+3e^{2/\sigma^2}-6\right) \nonumber \\
&=& \frac{1}{24n}
\left(e^{1/\sigma^2}-1\right)
\left(e^{3/\sigma^2}+3e^{2/\sigma^2}+6e^{1/\sigma^2}+6\right). \label{dn}
\end{eqnarray}
Also, notice that the ratio 
\begin{eqnarray}
\frac{d_n}{c_n^2}
&=&
\frac{
\frac{1}{24n}
\left(e^{1/\sigma^2}-1\right)
\left(e^{3/\sigma^2}+3e^{2/\sigma^2}+6e^{1/\sigma^2}+6\right)
}
{
\frac{1}{36n}
\left(e^{1/\sigma^2}+2\right)^2
\left(e^{1/\sigma^2}-1\right)
}
\nonumber \\
&=&
\frac{
3\left(e^{3/\sigma^2}+3e^{2/\sigma^2}+6e^{1/\sigma^2}+6\right)
}
{ 2
\left(e^{1/\sigma^2}+2\right)^2
}.
\label{ratioedge}
\end{eqnarray}
In our asymptotic analysis we consider $\sigma$ to be a constant independent of $n=M$. So, $\sigma$ is not a function of $M$ as this would change the asymptotic behavior of $c_M$ and $d_M$.

\section{Trade-Off Function via Edgeworth Expansion}
\label{app:tradeoff_edgeworth}

\subsection{False Positive Rate}

We are ready to analyze $\alpha(h)$ of (\ref{alpha}) with the probability density of $(x_j)_{j=1}^M$ conditioned on $d$ defined by (\ref{pdfd}). By using the definition of $\bar{Y}_M$ in (\ref{Ybar}) with the $Y_j$ defined in (\ref{Yj}) and by using approximation (\ref{Papprox}) derived from Theorem \ref{theo:edge} we obtain
\begin{eqnarray*}
\alpha(h) &=& \mathbf{P}[\bar{Y}_M>h]= 1- \mathbf{P}[\bar{Y}_M<h]
\\
&=&
1- \Phi\circ p_M ((h-1)\cdot \sqrt{\frac{M}{e^{1/\sigma^2}-1}}) -\epsilon_M(h),
\end{eqnarray*}
for 
\begin{equation}
h\in 
[1- \kappa_M \cdot
\sqrt{\frac{e^{1/\sigma^2}-1}{M}},
1+ \kappa_M \cdot
\sqrt{\frac{e^{1/\sigma^2}-1}{M}}
]
\label{Prangeh}
\end{equation}
such that the argument of $\Phi\circ p_M$ has an absolute value $\leq \kappa_M$ with $\kappa_M=o(M^{1/16})$ as defined in Theorem \ref{theo:edge}. Notice that $\epsilon_M(h)=o(1/M)$ uniformly in $h$ satisfying range (\ref{Prangeh}).

Both endpoints
$$
h_{\pm}=1\pm \kappa_M\cdot \sqrt{\frac{e^{1/\sigma^2}-1}{M}}
$$
lie in range (\ref{Prangeh}), so that the expression for $\alpha$ derived above applies at each of them and gives
$$
\alpha(h_+)\leq \Phi(-p_M(\kappa_M))+|\epsilon_M|
\ \ \mbox{ and } \ \
\alpha(h_-)\geq \Phi(-p_M(-\kappa_M))-|\epsilon_M|\geq \Phi(p_M(\kappa_M))-|\epsilon_M|,
$$
the last step by $p_M(-\kappa_M)\leq -p_M(\kappa_M)$ from Theorem \ref{theo:edge} ($\mu_3\geq 0$ for the lognormal distribution). For 
\begin{equation}a\in [\Phi(-p_M(\kappa_M))+|\epsilon_M|,\Phi(p_M(\kappa_M))-|\epsilon_M|] =
[\Phi(-p_M(\kappa_M))+|\epsilon_M|,1-\Phi(-p_M(\kappa_M))-|\epsilon_M|] \label{b1} \end{equation}  we therefore have $\alpha(h_+)\leq a\leq \alpha(h_-)$, and since $\alpha$ is strictly decreasing, $\alpha^{-1}(a)$ exists and lies between $h_-$ and $h_+$.
And for such $a$, $\alpha^{-1}(a)$  is in the range (\ref{Prangeh}) so that for the inverse $\alpha^{-1}(a)=h$ we may use the expression for $\alpha$ and derive
\begin{equation}
\alpha^{-1}(a)=h=1+p_M^{-1}\circ \Phi^{-1}(1-a-\epsilon_M(h)) \cdot \sqrt{\frac{e^{1/\sigma^2}-1}{M}}.
\label{Palphainv}
\end{equation}

\begin{proposition} \label{prop:FPinv}
Fix a constant $\gamma>0$ and let $\kappa'_M=\sqrt{2(1+\gamma)\ln M}$, assuming $M$ large enough such that $\kappa'_M\leq \kappa_M$.
There exists an $\epsilon_M(h)$ with $\epsilon_M(h)=o(1/M)$ uniformly in $h$ satisfying range (\ref{Prangeh}) such that the following holds:

(a) If $h$ satisfies range (\ref{Prangeh}), then
$$
\alpha(h)=1- \Phi\circ p_M ((h-1)\cdot \sqrt{\frac{M}{e^{1/\sigma^2}-1}}) -\epsilon_M(h).
$$

(b) Let $|\epsilon_M|=o(1/M)$ indicate the supremum of values $|\epsilon_M(h)|$ for $h$ in range (\ref{Prangeh}).
If $M$ is large enough and 
$$
a\in [ M^{-1}/\sqrt{4\pi \ln M} + |\epsilon_M|, 1-M^{-1}/\sqrt{4\pi \ln M} - |\epsilon_M| ],
$$
then (\ref{b1}) holds
with $\kappa_M$ replaced by $\kappa'_M=\sqrt{2(1+\gamma)\ln M}$ and as a consequence, the solution $h$ for $\alpha(h)=a$ (the false positive rate as defined by (\ref{alpha}))  satisfies range (\ref{Prangeh}) with $\kappa_M$ replaced by $\kappa'_M=\sqrt{2(1+\gamma)\ln M}$ (and the equation in (a) holds for $\alpha(h)$).

(c)
Furthermore, 
the inverse $h=\alpha^{-1}(a)$ satisfies (\ref{Palphainv}):
$$\alpha^{-1}(a)=1+p_M^{-1}\circ \Phi^{-1}(1-a-\epsilon_M(h)) \cdot \sqrt{\frac{e^{1/\sigma^2}-1}{M}}.$$
\end{proposition}

\noindent
{\em Proof:} Parts (a) and (c) follow from the discussion before the proposition. For part (b)
we use 
\begin{equation}\kappa_M'=\sqrt{2(1+\gamma)\ln M} \label{kmpr}
\end{equation}
for some small constant $\gamma>0$ and notice that 
$$
p_M(\kappa_M')=\kappa_M' - \kappa''_M,
$$
where by Theorem \ref{theo:edge} and equations (\ref{cn}) and (\ref{dn}),

\begin{eqnarray*}
\kappa''_M&=& c_M\cdot ({\kappa'_M}^2-1) + d_M\cdot ({\kappa'_M}^3-3\kappa'_M) - c_M^2 \cdot (4{\kappa'_M}^3-7\kappa'_M)  \\
&=& O(c_M {\kappa'_M}^2 + d_M{\kappa'_M}^3 +c_M^2{\kappa'_M}^3)=O(\frac{\ln M}{\sqrt{M}}).
\end{eqnarray*}
This shows that for $M$ large enough $p_M(\kappa'_M)=\kappa'_M-\kappa''_M\geq \sqrt{2\ln M}$.
We apply the 
upper bound $\Phi(-t)\leq e^{-t^2/2}/(\sqrt{2\pi}\cdot t)$ and obtain
\begin{eqnarray}
\Phi(-p_M(\kappa'_M)) &\leq & \Phi(-\sqrt{2\ln M}) \leq
\frac{e^{-\ln M}}{\sqrt{4\pi\ln M}} \nonumber \\
&=& M^{-1}/\sqrt{4\pi \ln M}=o(1/M). \label{PhitauM}
\end{eqnarray}
Hence, for $\alpha^{-1}(a)$ we restrict ourselves to the range (plugging the above expression into (\ref{b1}))
$$
a\in [ M^{-1}/\sqrt{4\pi \ln M} + |\epsilon_M|, 1-M^{-1}/\sqrt{4\pi \ln M} - |\epsilon_M| ].
$$
And for such $a$, $\alpha^{-1}(a)$ is in the range (\ref{Prangeh}) with $\kappa_M$ replaced by $\kappa'_M$, so that (\ref{Palphainv}) applies. \hfill $\Box$

\subsection{False Negative Rate}

Proposition \ref{prop:betaas} states that
 the false negative rate $\beta(h)$ as defined by (\ref{beta}) is equal to the integral
\begin{eqnarray*} \beta(h) &=& \int_{z=-\infty}^{\sigma\ln(Mh) - 1/(2\sigma)}
   \int_{(x_j)_{j=2}^M \ : \ \begin{array}{l} \frac{1}{M-1} \sum_{j=2}^M e^{x_j/\sigma -1/(2\sigma^2)}< \\
 (h-\frac{1}{M}\cdot e^{z/\sigma+1/(2\sigma^2)})\cdot \frac{M}{M-1}
 \end{array}
 }  \\
&&  \prod_{i=2}^M \frac{e^{-x_i^2/2}}{\sqrt{2\pi}} \ dx_2\cdots dx_M \cdot
\frac{e^{-z^2/2}}{\sqrt{2\pi}}  \ dz.
\end{eqnarray*}

Proposition \ref{prop:FPinv} suggests that the inner multi-dimensional integral is equal to
\begin{eqnarray}
&& 1-\alpha(h\cdot \frac{M}{M-1} - \frac{e^{z/\sigma+1/(2\sigma^2)}}{M-1}) \nonumber \\
&=&
\Phi\circ p_{M-1} (
(h\cdot \frac{M}{M-1} - \frac{e^{z/\sigma+1/(2\sigma^2)}}{M-1}
-1)\cdot \sqrt{\frac{M-1}{e^{1/\sigma^2}-1}}) 
\nonumber \\
&& +\epsilon_{M-1},\label{Pinner}
\end{eqnarray}
where $\epsilon_{M-1}$ is $\epsilon_{M-1}(x)$ for $x$ equal to the argument of $\alpha$.
This only holds true if the argument of 
$\alpha$ satisfies range (\ref{Prangeh}). 

We first note that we will want to compute the trade-off function $f(a)=\beta(\alpha^{-1}(a))$ and therefore we will restrict our evaluation of $\beta(h)$ to the range (\ref{Prangeh}) with $\kappa_M$ replaced by $\kappa'_M$ as defined in (\ref{kmpr}) and in Proposition \ref{prop:FPinv}.

For $z\rightarrow -\infty$, we have that the argument of $\alpha(.)$ tends to its supremum $hM/(M-1)$. We need to check that $hM/(M-1)$ fits range (\ref{Prangeh}) with $\kappa_M$ for $h$ fitting range (\ref{Prangeh}) with $\kappa_M$ replaced by $\kappa'_M$. We first want to verify whether the upper bound
\begin{equation}
\frac{M}{M-1} \cdot h \leq \frac{M}{M-1} \cdot (1+\kappa'_M\cdot \sqrt{\frac{e^{1/\sigma^2}-1}{M}})
\leq 
1+ \kappa_M \cdot
\sqrt{\frac{e^{1/\sigma^2}-1}{M}},   \label{check} 
\end{equation}
holds for $M$ large enough with $\kappa'_M=\sqrt{2(1+\gamma)\ln M}$.
Since Theorem \ref{theo:edge} holds for any $\kappa_M=o(M^{1/16})$ we have the liberty to use the next explicit definition for $\kappa_M$: We fix some constant $0<\gamma'<9/16$ and define 
\begin{equation}
\kappa_M = \frac{M}{M-1}(M^{1/16-\gamma'}+ \kappa'_M) = o(M^{1/16}). \label{kM}
\end{equation}
For fixed $\sigma$, this choice satisfies (\ref{check}) for $M$ large enough.

Let us restrict $z$ in the inner multi-dimensional integral to
\begin{equation}
z\leq \sigma \ln(M(h-q))-\frac{1}{2\sigma}
\label{Pz1}
\end{equation}
for 
\begin{equation}
q= \frac{M-1}{M}\cdot (1- \kappa_M \cdot \sqrt{\frac{e^{1/\sigma^2}-1}{M}}). \label{eqq}   
\end{equation}
With this restriction on $z$, the argument of $\alpha$ is equal to
\begin{eqnarray*}
h\cdot \frac{M}{M-1} - \frac{e^{z/\sigma+1/(2\sigma^2)}}{M-1} 
&\geq & h\cdot \frac{M}{M-1}-
\frac{M(h-q)}{M-1} \\
&=& q\cdot \frac{M}{M-1}
= 1- \kappa_M \cdot \sqrt{\frac{e^{1/\sigma^2}-1}{M}}.
\end{eqnarray*}
This means that for $z$ satisfying (\ref{Pz1}), the argument of $\alpha$ fits range  (\ref{Prangeh}) with $\kappa_M$.
We conclude that the inner multi-dimensional integral is equal to (\ref{Pinner}) for all $z$ satisfying (\ref{Pz1}) with $h$ satisfying range  (\ref{Prangeh}) with $\kappa_M$ replaced by $\kappa'_M$.

\begin{proposition}
There exists an $\epsilon_M(h)=o(1/M)$ uniformly in $h$ satisfying range (\ref{Prangeh}) with $\kappa_M$ replaced by $\kappa'_M$, see  (\ref{kmpr}) and (\ref{kM}) for their definitions, such that the false negative rate $\beta(h)$ as defined by (\ref{beta}) for $h$ in this range is equal to  
\begin{eqnarray*} && \beta(h)= \\
&& \int_{z=-\infty}^{\sigma\ln(M(h-q)) - 1/(2\sigma)}
   \Phi\circ p_{M-1} (
(h\cdot \frac{M}{M-1} - \frac{e^{z/\sigma+1/(2\sigma^2)}}{M-1}
-1)\cdot \sqrt{\frac{M-1}{e^{1/\sigma^2}-1}})  \cdot
\frac{e^{-z^2/2}}{\sqrt{2\pi}}  \ dz \\
&& \pm |\epsilon_{M-1}| + \int_{z=\sigma\ln(M(h-q)) - 1/(2\sigma)}^{\sigma\ln(Mh) - 1/(2\sigma)}
   \int_{(x_j)_{j=2}^M \ : \ \begin{array}{l} \frac{1}{M-1} \sum_{j=2}^M e^{x_j/\sigma -1/(2\sigma^2)}< \\
 (h-\frac{1}{M}\cdot e^{z/\sigma+1/(2\sigma^2)})\cdot \frac{M}{M-1}
 \end{array}
 }  \\
&& \hspace{5cm} \prod_{i=2}^M \frac{e^{-x_i^2/2}}{\sqrt{2\pi}} \ dx_2\cdots dx_M \cdot
\frac{e^{-z^2/2}}{\sqrt{2\pi}}  \ dz,
\end{eqnarray*}
where $q$ satisfies (\ref{eqq}).
\end{proposition}

The proposition implies that for all $h$ in range (\ref{Prangeh}) with $\kappa_M$ replaced by $\kappa'_M$,
$$
|\beta(h)-
\int_{-\infty}^{+\infty} \Phi\circ p_{M-1} (
(h\cdot \frac{M}{M-1} - \frac{e^{z/\sigma+1/(2\sigma^2)}}{M-1}
-1)\cdot \sqrt{\frac{M-1}{e^{1/\sigma^2}-1}}) \cdot
\frac{e^{-z^2/2}}{\sqrt{2\pi}}
dz| \leq \epsilon',$$
with
$$\epsilon' \leq |\epsilon_{M-1}| +  \Phi(-(\sigma\ln(M(h-q)) - 1/(2\sigma))).$$

Since $h$ satisfies (\ref{Prangeh}) for $\kappa_M$ replaced by $\kappa'_M$,
\begin{eqnarray*}
M(h-q) &=& Mh - (M-1)(1- \kappa_M \cdot \sqrt{\frac{e^{1/\sigma^2}-1}{M}}) \\
&\geq & M(1-\kappa'_M \cdot \sqrt{\frac{e^{1/\sigma^2}-1}{M}}) \\
&& - (M-1)(1- \kappa_M \cdot \sqrt{\frac{e^{1/\sigma^2}-1}{M}}) \\
&\geq &
((M-1)\kappa_M - M\kappa'_M)\cdot \sqrt{\frac{e^{1/\sigma^2}-1}{M}}.
\end{eqnarray*}
For $\kappa_M$ and $\kappa'_M$ defined by (\ref{kmpr}) and (\ref{kM}), 
$$
(M-1)\kappa_M - M\kappa'_M = M^{17/16-\gamma'}.
$$

Let $0<\gamma'<9/16$. Since $\sigma$ is constant,
$$
M(h-q) \geq M^{9/16-\gamma'} \sqrt{e^{1/\sigma^2}-1},
$$
and therefore
\begin{eqnarray*}
&& \sigma \ln(M(h-q))-1/(2\sigma) \\
&\geq &
\sigma (9/16-\gamma') \ln M
+\frac{\sigma}{2}\ln(e^{1/\sigma^2}-1)
-\frac{1}{2\sigma}.
\end{eqnarray*}
The right hand side is positive for $M$ large enough and grows like $\sigma(9/16-\gamma')\ln M$. Writing $t_M$ for it, the bound $\Phi(-t)\leq e^{-t^2/2}/(t\sqrt{2\pi})$ gives
$$
\Phi(-(\sigma \ln(M(h-q))-1/(2\sigma)))
\leq
\frac{e^{-t_M^2/2}}{t_M\sqrt{2\pi}}
= e^{-\Theta((\ln M)^2)}
=o(1/M),
$$
since $t_M^2$ grows like a constant multiple of $(\ln M)^2$.
We have 
$$ \epsilon'\leq |\epsilon_{M-1}|+\Phi(-(  \sigma \ln(M(h-q))-1/(2\sigma))) = o(1/M).$$
We may push this inside $\epsilon_{M-1}=o(1/M)$.

\begin{proposition} \label{prop:betaas1}
Suppose that $\sigma>0$ is a constant independent of $M$.
Then, for $M$ large enough, uniformly  for $h$ in range (\ref{Prangeh}) with $\kappa_M$ replaced by $\kappa'_M$, see  (\ref{kmpr}) and (\ref{kM}) for their definitions,
$$
|\beta(h)-
\int_{-\infty}^{+\infty} \Phi\circ p_{M-1} (
(h\cdot \frac{M}{M-1} - \frac{e^{z/\sigma+1/(2\sigma^2)}}{M-1}
-1)\cdot \sqrt{\frac{M-1}{e^{1/\sigma^2}-1}}) \cdot
\frac{e^{-z^2/2}}{\sqrt{2\pi}}
dz| = o(1/M).$$
\end{proposition}

\subsection{Trade-Off Function} \label{app:f}

In order to compute trade-off function $f(a)$ we substitute 
$h=\alpha^{-1}(a)$
in Proposition \ref{prop:betaas1} with $\alpha^{-1}(a)$ approximated by (\ref{Palphainv}), see Proposition \ref{prop:FPinv}(c).
After reordering terms, this yields
\begin{eqnarray}
f(a) &=& \int_{-\infty}^{+\infty} \Phi\circ p_{M-1}( \nonumber \\
&& \hspace{2cm}
p^{-1}_M\circ \Phi^{-1}(1-a-\epsilon_M(h)) \cdot 
\sqrt{\frac{M}{M-1}}
-
\frac{e^{z/\sigma+1/(2\sigma^2)}-1}{\sqrt{M-1}\sqrt{e^{1/\sigma^2}-1}} \nonumber
\\
&& \hspace{1cm} )\cdot \frac{e^{-z^2/2}}{\sqrt{2\pi}}dz+o(1/M)
\label{Peqinner2}
\end{eqnarray}
uniformly for $a$ in the range given by Proposition \ref{prop:FPinv}(b).
Let 
\begin{eqnarray}
x&=&p^{-1}_M\circ \Phi^{-1}(1-a-\epsilon_M), \nonumber \\
\epsilon &=& 
    x \cdot (\sqrt{\frac{M}{M-1}}-1), \nonumber \\
\mu(z)&=& \frac{e^{z/\sigma+1/(2\sigma^2)}-1}{\sqrt{M-1}\sqrt{e^{1/\sigma^2}-1}}. \label{muz}
\end{eqnarray}
This notation simplifies  (\ref{Peqinner2}) to
\begin{equation}
f(a) = \int_{-\infty}^{+\infty} \Phi\circ p_{M-1}( 
x + \epsilon  
-
\mu(z) ) \cdot \frac{e^{-z^2/2}}{\sqrt{2\pi}}dz+o(1/M). \label{simplef1}
\end{equation}
The remainder of Appendix
\ref{app:f} proves the next proposition.

\begin{proposition} \label{prop:asf} Let $\mu_M=\sqrt{\frac{e^{1/\sigma^2}-1}{M-1}}$.
We denote the range of Proposition~\ref{prop:FPinv}(b) by $[\tau_M,1-\tau_M]$.
There exist
$0\leq \bar\gamma_M=O(\sqrt{\ln M}/M)$ and $\tau_M\leq \bar\delta_M=o(1/M)$ (where $\bar\gamma_M$ and $\bar\delta_M$ depend on $M$ alone) such that:

(i) For every $a\in[\tau_M,1-\tau_M]$,
$$
f(a) \geq \Phi(y-(\mu_M+\bar\gamma_M)) -\frac{e^{2/\sigma^2}}{2(M-1)}\cdot y\mathrm{n}(y) -\bar\delta_M,
\qquad
y=\Phi^{-1}(1-a-\epsilon_M(\alpha^{-1}(a))).
$$

(ii) For every $a\in [0,\tau_M]$,
$$ f(a) \geq 1-\bar\delta_M.$$

\end{proposition}

\noindent
{\em Proof:}
Proposition \ref{prop:FPinv}(b) gives $|x|\leq \kappa'_M=\sqrt{2(1+\gamma)\ln M}$.
Furthermore,
\begin{equation}
|\epsilon| = |x\cdot (\sqrt{\frac{M}{M-1}}-1)|  \leq \frac{|x|}{2(M-1)} \\
\leq  \frac{\kappa'_M}{2(M-1)}=\frac{\sqrt{(1+\gamma)\ln M}}{\sqrt{2}(M-1)}= O(\frac{\sqrt{\ln M}}{M}).  \label{boundx}  
\end{equation}

Notice that there exists a $\xi\in [x-|\epsilon|-|\mu(z)|,x+|\epsilon| + |\mu(z)|]$ such that
$$
p_{M-1}(x+\epsilon -\mu(z))=p_{M-1}(x)-(\mu(z)-\epsilon)\cdot p'_{M-1}(\xi),
$$
where, see (\ref{derp}) and (\ref{ratioedge}) together with $c_{M-1}=O(1/\sqrt{M})$, $d_{M-1}=O(1/M)$,  and $|x|=O(\sqrt{\ln M})$ (and $|\epsilon|=O(|x|)$),
\begin{eqnarray*}
p'_{M-1}(\xi) 
&=&
1-2c_{M-1}\xi + d_{M-1}(3-3\xi^2)+
c_{M-1}^2(12\xi^2-7) \\
&=& 1+ O( c_{M-1}(|x|+|\mu(z)|)+(d_{M-1}+c_{M-1}^2)(1+(|x|+|\mu(z)|)^2)) 
 \\
 &=& 1+ O(\frac{\sqrt{\ln M}}{\sqrt{M}}+\frac{|\mu(z)|}{\sqrt{M}}+\frac{|\mu(z)|^2}{M}).
\end{eqnarray*}
This shows that
\begin{equation}
p_{M-1}(x+\epsilon -\mu(z))=p_{M-1}(x)-\mu(z) + \epsilon(z) \label{simplef2}
\end{equation}
with
\begin{eqnarray}
|\epsilon(z)| &=&
O(|\epsilon| + (|\epsilon|+|\mu(z)|) (\frac{\sqrt{\ln M}}{\sqrt{M}}+\frac{|\mu(z)|}{\sqrt{M}}+\frac{|\mu(z)|^2}{M}) )
\nonumber \\
&=&
O(\frac{\sqrt{\ln M}}{M}
+
\frac{|\mu(z)|\cdot \sqrt{\ln M}}{M^{3/2}}  +\frac{|\mu(z)|^2\cdot \sqrt{\ln M}}{M^2} \nonumber \\
&&
+
\frac{|\mu(z)| \cdot \sqrt{\ln M}}{\sqrt{M}}+\frac{|\mu(z)|^2}{\sqrt{M}}+\frac{|\mu(z)|^3}{M}).  \label{ez1}
\end{eqnarray}
Later in our analysis we will use that the integral of $|\epsilon(z)|$ over $z$ with respect to the normal distribution replaces (we use (\ref{muz}) with (\ref{moments})) each term $|\mu(z)|$ with $O(1/\sqrt{M})$. This gives
\begin{equation}
 \int_{-\infty}^{+\infty} |\epsilon(z)| \cdot \frac{e^{-z^2/2}}{\sqrt{2\pi}}dz
= O(\frac{\sqrt{\ln M}}{M}).   \label{ez2}
\end{equation}

After establishing (\ref{simplef1}) and obtaining equation (\ref{simplef2}), we now want to compute $p_{M-1}(x)$. The definition of $p_{M-1}(x)$ and $p_M(x)$ tells us that their difference
is
$$
|p_{M-1}(x) -p_M(x)| = O(\frac{1+|x|^2}{M^{3/2}}+\frac{|x|+|x|^3}{M^{2}}).
$$
Notice that 
$$
p_{M}(x)=y \ \ \mbox{ with } \ \  y= \Phi^{-1}(1-a-\epsilon_{M}(h)).
$$
Together with $|x|\leq \kappa'_M=O(\sqrt{\ln M})$ this shows
$$
|p_{M-1}(x)-y|=|p_{M-1}(x) -p_M(x)| = O(\frac{\ln M}{M^{3/2}}).
$$
We may merge this error term into $\epsilon(z)$ without affecting the properties (\ref{ez1}) and (\ref{ez2}).
This yields a refinement of equation (\ref{simplef2}),
\begin{equation}
p_{M-1}(x+\epsilon -\mu(z))=y-\mu(z) + \epsilon(z) \ \ \mbox{ with } \ \  y= \Phi^{-1}(1-a-\epsilon_{M}(h)). \nonumber 
\end{equation}
Substituting this into (\ref{simplef1}) yields
\begin{equation}
f(a) = \int_{-\infty}^{+\infty} \Phi( 
y + \epsilon(z)  
-
\mu(z) ) \cdot \frac{e^{-z^2/2}}{\sqrt{2\pi}}dz+o(1/M). \label{simplef3}
\end{equation}

Now we use a Taylor series expansion around $y$: There exists a $\xi$ between $y+\epsilon(z)-\mu(z)$ and $y$ such that
\begin{eqnarray}
&&\Phi(y-(\mu(z)-\epsilon(z))) \nonumber
\\
&=&
\Phi(y) -(\mu(z)-\epsilon(z))\cdot \mathrm{n}(y) - \frac{(\mu(z)-\epsilon(z))^2}{2}\cdot y\mathrm{n}(y) \nonumber \\
&& - \frac{(\mu(z)-\epsilon(z))^3}{6}\cdot (\xi^2-1)\mathrm{n}(\xi).
\label{K1} \end{eqnarray}
We compute the integral in (\ref{simplef3}) with respect to each of these terms.
We have, see (\ref{ez2}) and (\ref{moments}),
\begin{eqnarray}
&& \int_{-\infty}^{+\infty} (\mu(z)-\epsilon(z)) \cdot \frac{e^{-z^2/2}}{\sqrt{2\pi}}dz \nonumber \\
&=& 
\int_{-\infty}^{+\infty} (\frac{e^{z/\sigma+1/(2\sigma^2)}-1}{\sqrt{M-1}\sqrt{e^{1/\sigma^2}-1}}) \cdot \frac{e^{-z^2/2}}{\sqrt{2\pi}}dz -\int_{-\infty}^{+\infty} \epsilon(z)
\cdot \frac{e^{-z^2/2}}{\sqrt{2\pi}}dz \nonumber \\
&=& 
\mu-\rho_M  \ \ \mbox{ with } \ \ \mu=\sqrt{\frac{e^{1/\sigma^2}-1}{M-1}} \label{K2} \\
&& \hspace{1cm}
\mbox{ and }
\rho_M = \int_{-\infty}^{+\infty} \epsilon(z)
\cdot \frac{e^{-z^2/2}}{\sqrt{2\pi}}dz = O(\frac{\sqrt{\ln M}}{M})
. \nonumber
\end{eqnarray}

In addition, we derive
\begin{eqnarray*}
&& \int_{-\infty}^{+\infty} \frac{(\mu(z)-\epsilon(z))^2}{2} \cdot \frac{e^{-z^2/2}}{\sqrt{2\pi}}dz  \\
&=&
\int_{-\infty}^{+\infty} \frac{\mu(z)^2}{2} \cdot \frac{e^{-z^2/2}}{\sqrt{2\pi}}dz
+
\int_{-\infty}^{+\infty} \frac{\epsilon(z)^2-2\mu(z)\epsilon(z)}{2} \cdot \frac{e^{-z^2/2}}{\sqrt{2\pi}}dz, 
\end{eqnarray*}
where, after grouping  terms and applying (\ref{moments}) for $b=\infty$, $c=-1$ and $k=0$, $1$ and $2$,
\begin{eqnarray}
\int_{-\infty}^{+\infty} \frac{\mu(z)^2}{2} \cdot \frac{e^{-z^2/2}}{\sqrt{2\pi}}dz &=&
\int_{-\infty}^{+\infty}
\frac{e^{2z/\sigma+2/(2\sigma^2)}-2e^{z/\sigma+1/(2\sigma^2)}  +1}{2(M-1)(e^{1/\sigma^2}-1)}
\cdot \frac{e^{-z^2/2}}{\sqrt{2\pi}}dz \nonumber \\
&=&
\frac{e^{3/\sigma^2}-2e^{1/\sigma^2}+1}{2(M-1)(e^{1/\sigma^2}-1)}
=
\frac{e^{2/\sigma^2}+e^{1/\sigma^2}-1}{2(M-1)}  \nonumber \\
&=& \frac{e^{2/\sigma^2}}{2(M-1)}+\frac{\mu^2}{2} \label{K3}
\end{eqnarray}
and, see (\ref{ez2}) and applying the technique of its derivation,
\begin{equation}
\int_{-\infty}^{+\infty} \frac{|\epsilon(z)|^2+2|\mu(z)|\cdot |\epsilon(z)|}{2} \cdot \frac{e^{-z^2/2}}{\sqrt{2\pi}}dz = O(\frac{\sqrt{\ln M}}{M^{3/2}}).    \label{K4}
\end{equation}
 Similar to the derivation above,
\begin{equation}
\int_{-\infty}^{+\infty} \frac{(|\mu(z)|+|\epsilon(z)|)^3}{6} \cdot \frac{e^{-z^2/2}}{\sqrt{2\pi}}dz = O(1/M^{3/2}).    \label{K5}
\end{equation}
We notice that \( |(\xi^2-1)\mathrm{n}(\xi)| \) is bounded by a constant and also \(y\mathrm{n}(y)\) is bounded by a constant. Therefore, 
combining (\ref{simplef3}) with equations (\ref{K1}), (\ref{K2}), (\ref{K3}), (\ref{K4}) and (\ref{K5}) gives
$$
f(a) = \Phi(y) -(\mu-\rho_M)\cdot \mathrm{n}(y) -\left( \frac{e^{2/\sigma^2}}{2(M-1)}+\frac{\mu^2}{2}\right)\cdot y\mathrm{n}(y) +o(1/M).
$$
We observe that this looks similar to the Taylor series expansion
$$
\Phi(y-(\mu-\rho_M))=
\Phi(y)
-(\mu-\rho_M)\cdot \mathrm{n}(y)
-\frac{(\mu-\rho_M)^2}{2}\cdot y \mathrm{n}(y)
-\frac{(\mu-\rho_M)^3}{6}\cdot (\xi^2-1)\mathrm{n}(\xi)
$$
for some $\xi$ between $y$ and $y-(\mu-\rho_M)$.
Since $\rho_M = O(\sqrt{\ln M}/M)$ and $\mu=O(1/\sqrt{M})$, we have
$$
\Phi(y-(\mu-\rho_M))=
\Phi(y) -(\mu-\rho_M)\cdot \mathrm{n}(y) - \frac{\mu^2}{2}\cdot y \mathrm{n}(y) + O(\frac{\sqrt{\ln M}}{M^{3/2}}).
$$
This proves our final expression for $f(a)$: 
\begin{equation}
f(a) = \Phi(y-(\mu-\rho_M))-\frac{e^{2/\sigma^2}}{2(M-1)}\cdot y \mathrm{n}(y) + o(1/M).
\label{fsharp}
\end{equation}

Equation~(\ref{fsharp})  holds uniformly for $a\in[\tau_M,1-\tau_M]$, this being the range of Proposition~\ref{prop:FPinv}(b) in which its derivation takes place. Let $\bar\gamma_M$ be the supremum of $|\rho_M|$ and let $\bar\delta_M$ be the larger of $\tau_M$ and the supremum of the absolute value of the $o(1/M)$ remainder in~(\ref{fsharp}), both suprema taken over $a\in[\tau_M,1-\tau_M]$. Then $0\leq \bar\gamma_M=O(\sqrt{\ln M}/M)$ and $0\leq \bar\delta_M=o(1/M)$, and $\bar\delta_M\geq \tau_M$. Since the shift in~(\ref{fsharp}) is $\mu_M-\rho_M$ with $\rho_M\geq -\bar\gamma_M$, and $\Phi$ is increasing,
$$
f(a) \geq \Phi(y-(\mu+\bar\gamma_M)) -\frac{e^{2/\sigma^2}}{2(M-1)}\cdot y\mathrm{n}(y) -\bar\delta_M,
\qquad
y=\Phi^{-1}(1-a-\epsilon_M(\alpha^{-1}(a))),
$$
for every $a\in[\tau_M,1-\tau_M]$. Here $\bar\gamma_M$ and $\bar\delta_M$ depend on $M$ alone. This proves part (i) of the proposition.

Part (ii) of the proposition
treats $a\in[0,\tau_M]$. At $a=\tau_M$, equation 
(\ref{fsharp}) yields
\begin{equation}
f(\tau_M) = \Phi(y-(\mu-\rho_M))-\frac{e^{2/\sigma^2}}{2(M-1)}\cdot y \mathrm{n}(y) + o(1/M), \qquad
y=\Phi^{-1}(1-\tau_M-\epsilon_M(\alpha^{-1}(\tau_M))). \label{limitcase}
\end{equation}
By the definition of $\tau_M$ we have $|\epsilon_M(\alpha^{-1}(\tau_M))|\leq \epsilon_M\leq \tau_M$, hence,
$$
\frac{1}{M\sqrt{4\pi\ln M}}=\tau_M-\epsilon_M\leq \tau_M+\epsilon_M(\alpha^{-1}(\tau_M))\leq \tau_M+\epsilon_M\leq 2\tau_M.
$$
The lower bound implies $y\leq \Phi^{-1}(1-1/(M\sqrt{4\pi\ln M}))=-\Phi^{-1}(1/(M\sqrt{4\pi\ln M}))$. Application of (\ref{PhitauM}) yields $y\leq \sqrt{2\ln M}$.
Write $u_M=\tau_M+\epsilon_M(\alpha^{-1}(\tau_M))$ and $\nu=\mu-\rho_M$, so that $y=-\Phi^{-1}(u_M)$ and $1-\Phi(y-\nu)=\Phi(\Phi^{-1}(u_M)+\nu)=1-G_\nu(u_M)$. For $M$ large enough, $0<\nu=O(M^{-1/2})$, because $|\rho_M|=O(\sqrt{\ln M}/M)$. Since $1/(M\sqrt{4\pi\ln M})\leq u_M\leq 2\tau_M$, we have $\nu\sqrt{\ln(1/u_M)}=o(1)$, and Proposition~\ref{prop:Gaus}(i) gives
$$1-\Phi(y-(\mu-\rho_M))=(1+o(1))\,u_M\leq (2+o(1))\,\tau_M=O(\tau_M).$$

We also want to bound $y\, \mathrm n(y)$. This can be done by standard calculations as follows:
If $y=-\Phi^{-1}(x)$ for some $x<1/2$, then $y>0$ and 
$$x=\Phi(-y)=1-\Phi(y) =\int_{t=y}^\infty \mathrm n(t) \, \mbox{d}t.$$
Integration by parts using $\mathrm n'(t) = -t\, \mathrm n(t)$ yields
$$
\int_{t=y}^\infty \mathrm n(t) \, \mbox{d}t
=
\int_{t=y}^\infty (-1/t)\cdot (-t\, \mathrm n(t)) \, \mbox{d}t
=
-\mathrm n(t)/t |_y^\infty - \int_{t=y}^\infty (\mathrm n(t)/t^2) \, \mbox{d}t
=
\mathrm n(y)/y - \int_{t=y}^\infty (\mathrm n(t)/t^2) \, \mbox{d}t.
$$
Now notice that
$$
\int_{t=y}^\infty (\mathrm n(t)/t^2) \, \mbox{d}t \leq \int_{t=y}^\infty (\mathrm n(t)/y^2) \, \mbox{d}t =x/y^2.
$$
Combining these derivations yields
$$ x\geq \mathrm n(y)/y - x/y^2.$$
Rearranging terms gives
$$ y\, \mathrm n(y) \leq (1+y^2)x.$$
Applying this to our parameter setting implies 
$$
y\, \mathrm n(y) \leq (1+y^2) \cdot (2\tau_M) = O(\tau_M \ln M).
$$

Equation~(\ref{limitcase}) therefore yields
$$
1-f(\tau_M)=O(\tau_M)+O(\tau_M(\ln M)/M)+o(1/M)
=O(\tau_M)+o(1/M).
$$
Without loss of generality, we may increase $\bar \delta_M$ without violating $\tau_M\leq \bar \delta_M=o(1/M)$ such that
$$ 1-f(\tau_M)\leq \bar \delta_M.$$
Since $f$ is decreasing, this completes the proof. \hfill $\Box$

\subsection{Multiple Epochs and Proof of Theorem \ref{theo:asym}}

We first prove an auxiliary lemma which is used in the proof of Theorem~\ref{theo:asym}.
Lemma~\ref{lem:boundF} is a generic property of symmetric convex trade-off functions, applicable beyond the random-allocation setting. 

\begin{lemma} \label{lem:boundF}
Let $g$ be  a convex, continuous, decreasing  and symmetric (that is, $g=g^{-1}$ with the inverse $g^{-1}$ defined as $g^{-1}(a) =
\inf \{ t\in [0,1] \ : \ g(t)\leq a\}$ for $a\in [0,1]$) trade-off function. In addition, $g$ is strictly decreasing over the whole interval $a\in [0,1]$. Let $a^*=g(a^*)$ be the fixed point of $g$. 
Let $0<\delta< a^*$ and
define
$$ h(a)=\max\{g(\min\{a+\delta,1\})-\delta,0\}.$$
Let 
$$
\hat{\delta}=\max\{\delta/a^*, (1-g(\delta))+\delta\}.
$$
Define 
$$
k(a)=\left\{
\begin{array}{ll} (1-\hat{\delta})g(\frac{a}{1-\hat{\delta}}) & \mbox{if } a\leq 1-\hat{\delta}, \\
0 & \mbox{if } a\geq 1-\hat{\delta}.
\end{array}
\right.
$$
Then, both $h$ and $k$ are convex, continuous, decreasing  and symmetric trade-off functions and  
$$  h(a)\geq k(a)=(g \otimes f_{0,\hat{\delta}})(a) \ \ \mbox{ for } \ \ a\in[0,1],$$
where
\[
f_{0,\hat{\delta}}(a)=
\max\{1-a-\hat{\delta},0\}
\]
is the trade-off function of two uniform distributions \(T(U(0,1), U(\hat{\delta},1+\hat{\delta}))\).
\end{lemma}

\noindent
{\em Proof:}
The lemma assumes $\delta< a^*$, where $a^*$ is the unique fixed point of $g$, i.e.,  $g(a^*)=a^*$. 
Since $g$ is decreasing, this implies 
$g(\delta)-\delta\geq g(a^*)-\delta =a^*-\delta>0$. We have that both values in the max definition of $\hat{\delta}$ are $<1$ and function $k$ is well defined.
In addition, we have $h(0)=g(\delta)-\delta>0$. Since $g$ is a trade-off function less than the random guessing diagonal $a\rightarrow 1-a$, $h(0)=g(\delta)-\delta \leq 1-\delta-\delta\leq 1$ and we have that $h$ is well-defined. Since $g$ is convex and decreasing, both $h$ and $k$ are convex and decreasing.

Notice that $a\leq g(\delta)-\delta$ if and only if $a+\delta \leq g(\delta)$ if and only if $g(a+\delta)\geq g(g(\delta))$ since $g$ is strictly decreasing. Since $g$ is also symmetric (which implies that $g^2$ is the identity function), $g(g(\delta))=\delta$.
For this reason
\begin{equation} h(a)=
\left\{ \begin{array}{ll}
g(a+\delta)-\delta      &  \mbox{ if }  a\leq g(\delta)-\delta,\\
0     & \mbox{ if } a\geq g(\delta)-\delta.
\end{array} \right. \label{eqH}
\end{equation}

Let us compute $h(h(a))$. Since $h$ is decreasing, we have $h(a)\leq h(0)=g(\delta)-\delta$ and from (\ref{eqH}) we infer that
$h(h(a)) = g(h(a)+\delta)-\delta$.
If $a\leq g(\delta)-\delta=h(0)$, then 
$$ h(h(a))=g(h(a)+\delta)-\delta = g(g(a+\delta)-\delta+\delta)-\delta = g(g(a+\delta))-\delta = a+\delta-\delta =a.$$
This proves for $a\leq h(0)$,
$$ h^{-1}(a) =
\inf \{ t\in [0,1] \ : \ h(t)\leq a\}
=
\inf \{ t\in [0,1] \ : \ h(t)\leq h(h(a))\} = h(a)
$$
(since $h$ is strictly decreasing for $a\leq g(\delta)-\delta=h(0)$).
If $a\geq g(\delta)-\delta=h(0)$, then $t=0$ realizes the infinum
$$ h^{-1}(a) =
\inf \{ t\in [0,1] \ : \ h(t)\leq a\}
=0
$$
which is equal to $h(a)$ since $a\geq g(\delta)-\delta$.
This proves that $h=h^{-1}$ is also symmetric.

A similar analysis shows that, besides being convex, continuous and decreasing, $k$ is  also a symmetric trade-off function. In addition, $k$
is equal to $g\otimes f_{0,\hat{\delta}}$, where
$f_{0,\hat{\delta}}$ is the trade-off function of two uniform distributions $T(U(0,1), U(\hat{\delta}, 1+\hat{\delta}))$, see Section 3.3 in \citet{DBLP:journals/corr/abs-1905-02383}.

In order to prove $h\geq k$, we first notice that $1-\delta\geq g(\delta)-\delta$ and (\ref{eqH}) implies $h(1-\delta)=0$. Second, since $g$ is decreasing, $a^*$ is a fixed point of $g$, and $\delta<a^*$, $g(\delta)-\delta\geq g(a^*)-\delta=a^*-\delta$. Hence, (\ref{eqH}) implies that $h(a^*-\delta)=g(a^*-\delta+\delta)-\delta=a^*-\delta$. This proves that $a^*-\delta$ is the unique fixed point of $h$.
Since $a^*(1-\hat{\delta})\leq 1-\hat{\delta}$, we have $k(a^*(1-\hat{\delta}))=(1-\hat{\delta})g(a^*)=a^*(1-\hat{\delta})$. Hence, $a^*(1-\hat{\delta})$ is a fixed point of $k$. 
Since $\hat{\delta}$ is defined as a value $\geq \delta/a^*$, we have that
the fixed point of $k$ is smaller than the fixed point of $h$:
\begin{equation}
\hat{\delta} \geq \delta /a^* \ \ \mbox{ if and only if } a^*(1-\hat{\delta}) \leq  a^*-\delta. \label{fixhk} 
\end{equation}

The fixed points of the symmetric, convex, decreasing trade-off functions $h$ and $k$ define the intersection of the diagonal line $a\rightarrow a$ with the trade-off functions and denote the points on the trade-off functions that are farthest away from the random guessing diagonal $a \rightarrow 1-a$. Condition (\ref{fixhk}) implies that function $k$ is smaller than $h$ when considering the fixed points along the diagonal $a\rightarrow a$. In particular,  since $h$ and $k$ are decreasing, convex and symmetric, we have that 
\begin{equation}
h(a^*(1-\hat{\delta}))\geq k(a^*(1-\hat{\delta})).    \label{hkend}
\end{equation}
If we are able to prove that
$h(a)\geq k(a)$ for $a\in [0,a^*(1-\hat{\delta})]$ or, equivalently, for $a$ at most the value of $k$'s fixed point, then by symmetry and convexity of $h$ and $k$, we have that $h(a)\geq k(a)$ for all $a\in [0,1]$.
Let $a\in [0,a^*(1-\hat{\delta})]$. Then we have $a\leq a^*(1-\hat{\delta})\leq a^*-\delta\leq g(\delta)-\delta$ (since $a^*=g(a^*)\leq g(\delta)$ because $g$ is decreasing and $\delta\leq a^*$). From (\ref{eqH}) we infer $h(a)=g(a+\delta)-\delta$. Also, $a\leq a^*(1-\hat{\delta})\leq 1-\hat{\delta}$, hence, $k(a)=(1-\hat{\delta})g(a/(1-\hat{\delta}))$. We need to prove 
\begin{equation}
h(a)=g(a+\delta)-\delta \geq (1-\hat{\delta})g(a/(1-\hat{\delta}))=k(a) \ \ \mbox{ for }  \ \ 0\leq a\leq a^*(1-\hat{\delta}). \label{ineqhk}
\end{equation}
By the definition of $\hat{\delta}$, $\hat{\delta}\geq (1-g(\delta))+\delta$ from which we obtain
$h(0)=g(\delta)-\delta \geq 1-\hat{\delta}=k(0)
$ showing (\ref{ineqhk}) for $a=0$.
Together with (\ref{hkend}) we conclude that 
(\ref{ineqhk}) is satisfied for its end points $a=0$ and $a=a^*(1-\hat{\delta})$.

If $h'(a)\geq k'(a)$ for $a\in [0,x]$, then together with $h(0)\geq k(0)$ this implies $h(a)\geq k(a)$ for $a\in[0,x]$. Similarly, if $h'(a)\leq k'(a)$ for $a\in [x,a^*(1-\hat{\delta})]$, then together with $h(a^*(1-\hat{\delta}))\geq k(a^*(1-\hat{\delta}))$ this implies $h(a)\geq k(a)$ for $a\in [x,a^*(1-\hat{\delta})]$.
We have, almost everywhere,
\begin{eqnarray*}
    h'(a) &=& g'(a+\delta), \\
    k'(a) &=& g'(a/(1-\hat{\delta})).
\end{eqnarray*}
Since $g$ is convex, $g'$ is non-decreasing wherever it exists, which implies that
$g'(a+\delta)\geq g'(a/(1-\hat{\delta}))$ whenever
$$ a+\delta \geq a/(1-\hat{\delta}).$$
This means that $h'(a)\geq k'(a)$ for $a\in [0, \delta (1-\hat{\delta})/\hat{\delta}]$, and $h'(a)\leq k'(a)$ for larger $a$, since there $a+\delta \leq a/(1-\hat{\delta})$. This shows that there indeed exists an $x=\min \{ \delta (1-\hat{\delta})/\hat{\delta}, a^*(1-\hat{\delta})\}$ fitting the above argument.
We conclude that $h(a)\geq k(a)$ for $a\in [0,1]$.
\hfill $\Box$

In order to characterize $\hat{\delta}$ of Lemma \ref{lem:boundF} in our proof of Theorem \ref{theo:asym}, we need the next two properties:

\begin{proposition} \label{prop:Gaus}
Let $\mu=O(1/\sqrt{M})$.

(i) Let $a=O(1/M)$. If $\mu\sqrt{\ln(1/a)}=O(1)$, then $(1-G_\mu(a))/a = O(1)$. If $\mu\sqrt{\ln(1/a)}=o(1)$, then $(1-G_\mu(a))/a = 1 + o(1)$.

(ii) Let $G_\mu(a^*)=a^*$, then $a^*=\Phi(-\mu/2)=\frac{1}{2} - \frac{\mu}{\sqrt{8\pi}} +O(M^{-3/2})$.
\end{proposition}

\noindent
{\em Proof:} We rewrite
$$
1-G_\mu(a) = 1-\Phi(\Phi^{-1}(1-a)-\mu)
= \Phi(-(\Phi^{-1}(1-a)-\mu))
= \Phi(\Phi^{-1}(a)+\mu).
$$
Since $\Phi$ is increasing, also $\Phi^{-1}$ is increasing and (\ref{approxPhi}) implies
\begin{equation}
\Phi^{-1}\left(\frac{1}{t+\sqrt{t^2+4}} \sqrt{\frac{2}{\pi}} e^{-t^2/2} \right) < -t \leq \Phi^{-1}\left(\frac{1}{t+\sqrt{t^2+8/\pi}} \sqrt{\frac{2}{\pi}} e^{-t^2/2}\right). 
\nonumber
\end{equation}
For $t$ implicitly defined as the solution of
\begin{equation} a =\frac{1}{t+\sqrt{t^2+4}} \sqrt{\frac{2}{\pi}} e^{-t^2/2}, \label{at}
\end{equation}
we have $\Phi^{-1}(a)\leq -t$.
This implies (we use (\ref{approxPhi}) and (\ref{at}))
\begin{eqnarray*}
\Phi(\Phi^{-1}(a)+\mu) &\leq & \Phi(-t+\mu)= \Phi(-(t-\mu))\leq
\frac{1}{(t-\mu)+\sqrt{(t-\mu)^2+8/\pi}} \sqrt{\frac{2}{\pi}} e^{-(t-\mu)^2/2} \\
&=& a \cdot
\frac{t+\sqrt{t^2+4}}{(t-\mu)+\sqrt{(t-\mu)^2+8/\pi}} e^{\mu t} e^{-\mu^2/2} \\
&\leq & 
a \cdot
\frac{2t+2}{2(t-\mu)} e^{\mu t} .
\end{eqnarray*}
Notice that, by rearranging (\ref{at}),
\begin{equation}
t =
\sqrt{
2\ln \left(
\frac{1}{a}\cdot
\frac{1}{t+\sqrt{t^2+4}}\cdot
\sqrt{\frac{2}{\pi}}
\right)
}
\leq
\sqrt{
2\ln \left(
\frac{1}{a}\cdot \frac{1}{\sqrt{2\pi}\,t}
\right)
}.
\label{ineqt}
\end{equation}
Since $a=O(1/M)$, the corresponding $t$ of (\ref{at}) cannot remain bounded.
Thus, for $M$ large enough, $t\geq 1/\sqrt{2\pi}$, and (\ref{ineqt}) implies
\[
t\leq \sqrt{2\ln(1/a)}.
\]
Moreover, since \(\sqrt{t^2+4}\leq t+2\) for \(t\geq0\), we have
\[
t+\sqrt{t^2+4}\leq 2t+2
\leq 2\sqrt{2\ln(1/a)}+2.
\]
Therefore, using the exact identity above once more,
\[
t
\geq
\sqrt{
2\ln \left(
\frac{1}{a}\cdot
\frac{1}{2\sqrt{2\ln(1/a)}+2}\cdot
\sqrt{\frac{2}{\pi}}
\right)
}.
\]
Notice that the denominator $a\cdot (2\sqrt{2\ln (1/a)}+2)$ increases in $a$ for $0<a\leq e^{-1/2}$. This means that we may substitute $a=O(1/M)$ for a lower bound on $t$. This proves that $t$ is unbounded in $M$. Together with $\mu=O(1/\sqrt{M})$ this implies $\frac{2t+2}{2(t-\mu)}=O(1)$. And $t$ being unbounded in $M$ together with  (\ref{ineqt})  also implies $t\leq \sqrt{2\ln(1/a)}$, which combined with the assumptions of the proposition, $\mu\sqrt{\ln(1/a)}=O(1)$ respectively $\mu\sqrt{\ln(1/a)}=o(1)$, implies $\mu t=O(1)$ respectively $\mu t=o(1)$.
Hence,
$$
\frac{1-G_\mu(a)}{a}=\frac{\Phi(\Phi^{-1}(a)+\mu)}{a}\leq \frac{2t+2}{2(t-\mu)} e^{\mu t} =\left\{
\begin{array}{ll}
 O(1),    & \mbox{ if } \mu\sqrt{\ln(1/a)}=O(1), \\
1+ o(1)    & \mbox{ if } \mu\sqrt{\ln(1/a)}=o(1).
\end{array} \right.
$$
This proves the first property.

For the second property we observe
$$ a^*=G_\mu(a^*) = \Phi(\Phi^{-1}(1-a^*)-\mu).$$
This is equivalent to
$$ \Phi^{-1}(a^*)=\Phi^{-1}(1-a^*)-\mu = -\Phi^{-1}(a^*)-\mu,
$$
hence,
$$ a^* = \Phi(-\mu/2).$$
Together with $\mu=O(1/\sqrt{M})$, the Taylor series expansion of the cumulative normal distribution around $0$, 
gives the following approximation:
$$
a^*=\Phi(-\mu/2)= \frac{1}{2} - \frac{\mu}{\sqrt{8\pi}} +O(M^{-3/2}).
$$
This concludes the proof of the proposition. \hfill $\Box$

\begin{proposition} \label{prop:lowerf} 
We denote the range of Proposition~\ref{prop:FPinv}(b) by $[\tau_M,1-\tau_M]$.
There exist
$0\leq \gamma'_M=O(\sqrt{\ln M}/M)$ and $\tau_M\leq \delta'_M=o(1/M)$ (where $\gamma'_M$ and $\delta'_M$ depend on $M$ alone) such that for every $a\in[\tau_M,1-\delta'_M]$,
$$
f(a) \geq G_{\mu_M+\gamma'_M}(a+\delta'_M) -\delta'_M \ \ \mbox{ with } \ \ \mu_M=\sqrt{\frac{e^{1/\sigma^2}-1}{M-1}}.
$$
\end{proposition}

\noindent
{\em Proof:}
By Proposition \ref{prop:asf}(i), for every $a\in[\tau_M,1-\tau_M]$,
$$
f(a) \geq \Phi(y-\nu) -\frac{e^{2/\sigma^2}}{2(M-1)}\cdot y\mathrm{n}(y) -\bar\delta_M,
\qquad
y=\Phi^{-1}(1-a-\epsilon_M(\alpha^{-1}(a))),
\qquad
\nu=\mu_M+\bar\gamma_M.
$$

We first analyse the case $y>0$.  The Taylor expansion of Equation~(\ref{eqTaylorPhi}) evaluated at $y-\nu$ with shift $x_M$ gives
$$
\Phi(y-(\nu+x_M)) \leq \Phi(y-\nu)-x_M\mathrm{n}(y-\nu) +\frac{x_M^2}{\sqrt{8e\pi}}.
$$
We choose $x_M$ such that
\begin{equation}
x_M\mathrm{n}(y-\nu) \geq \frac{e^{2/\sigma^2}}{2(M-1)}\cdot |y|\mathrm{n}(y),
\qquad\mbox{that is}\qquad
x_M \geq \frac{e^{2/\sigma^2}}{2(M-1)}\cdot |y|e^{-\nu y+\nu^2/2},
\label{fix4b}
\end{equation}
where $\mathrm{n}(y)/\mathrm{n}(y-\nu)=e^{-\nu y+\nu^2/2}$. On $[\tau_M,1-\tau_M]$ inequality (\ref{PhitauM}) in the proof of Proposition~\ref{prop:FPinv}(b) gives $|y|\leq p_M(\kappa'_M)=\kappa'_M-\kappa''_M=O(\sqrt{\ln M})$, with $\kappa'_M$ as in (\ref{kmpr}). Since $\nu\geq 0$ and $y>0$, we have $e^{-\nu y}\leq 1$. Hence the choice
$$
x_M=\frac{e^{2/\sigma^2}\cdot p_M(\kappa'_M)\cdot e^{\nu^2/2}}{2(M-1)}=O(\frac{\sqrt{\ln M}}{M})
$$
satisfies~(\ref{fix4b}) uniformly wherever $y>0$.

For $y>0$, the Taylor bound and~(\ref{fix4b}) imply
$$
\Phi(y-\nu) -\frac{e^{2/\sigma^2}}{2(M-1)}\cdot y\mathrm{n}(y) \geq \Phi(y-(\nu+x_M)) -\frac{x_M^2}{\sqrt{8e\pi}}.
$$
For $y\leq 0$ the same inequality follows from
$$
\Phi(y-\nu) -\frac{e^{2/\sigma^2}}{2(M-1)}\cdot y\mathrm{n}(y) \geq \Phi(y-\nu) \geq \Phi(y-(\nu+x_M)).
$$

We define
$$ \delta'_M=\bar\delta_M+\frac{x_M^2}{\sqrt{8e\pi}} \ \ \mbox{ and } \ \ \gamma'_M=\bar\gamma_M+x_M,$$
so that $\tau_M\leq \delta'_M=o(1/M)$ and $0\leq \gamma'_M=O(\sqrt{\ln M}/M)$, since $x_M^2=O(\ln M/M^2)=o(1/M)$. Notice again that sequences $\delta'_M$ and $\gamma'_M$ do not depend on $a$. 

Since $\delta'_M\geq \bar\delta_M\geq \tau_M$ we have $[\tau_M,1-\delta'_M]\subseteq [\tau_M,1-\tau_M]$, and $\tau_M\geq |\epsilon_M|$ gives $0\leq a+\epsilon_M(\alpha^{-1}(a))\leq 1$ and $a+\epsilon_M(\alpha^{-1}(a))\leq a+\delta'_M$ throughout this range. Combining the preceding inequalities,
\begin{eqnarray*}
f(a) &\geq& \Phi(y-\nu) -\frac{e^{2/\sigma^2}}{2(M-1)}\cdot y\mathrm{n}(y) -\bar\delta_M \\
&\geq& \Phi(y-(\nu+x_M))-\delta'_M \\
&=& G_{\mu_M+\gamma'_M}(a+\epsilon_M(\alpha^{-1}(a)))-\delta'_M \\
&\geq& G_{\mu_M+\gamma'_M}(a+\delta'_M)-\delta'_M,
\end{eqnarray*}
where the last inequality uses that $G_\mu(a)$ is decreasing in $a$. Therefore,
\begin{equation}
f(a) \geq G_{\mu_M+\gamma'_M}(a+\delta'_M) -\delta'_M \ \ \mbox{ for } \ \ a\in[\tau_M,1-\delta'_M],
\label{fix2}
\end{equation}
which proves the proposition.
\hfill $\Box$

We are now ready to prove the main asymptotical result of Theorem \ref{theo:asym}.

\noindent
{\em Proof of Theorem \ref{theo:asym}:}  
We write $f_M=T(P^{(M)}_0,P^{(M)}_1)$ for the trade-off function that Propositions~\ref{prop:asf} and~\ref{prop:lowerf} denote by $f$. The function $f$ of Theorem~\ref{theo:asym} is the largest trade-off function of the mechanism.
As a consequence of Proposition \ref{prop:lowerf} we have that there exists a sequence $\delta'_M$ and a sequence $\gamma'_M$ with $\tau_M\leq \delta'_M=o(1/M)$ and $0\leq \gamma'_M=O(\sqrt{\ln M}/M)$ such that
trade-off function $f_M$ satisfies
$$
f_M(a) \geq G_{\mu_M+\gamma'_M}(a+\delta'_M) -\delta'_M
$$
for $a\in[\tau_M,1-\delta'_M]$, where
$$\mu_M=\sqrt{\frac{e^{1/\sigma^2}-1}{M-1}}.$$

Let $a^*_M$ denote the fixed point of $G_{\mu_M+\gamma'_M}$.

Define
$$
\delta_M^{\dagger}=\sqrt{\delta'_M/M}=o(1/M).
$$
Put $\lambda_M=\mu_M+\gamma'_M$ and define the sequence in Theorem~\ref{theo:asym} by
$$
\hat{\delta}_M=\max\{\delta_M^{\dagger}/a_M^*,\,1-G_{\lambda_M}(\delta_M^{\dagger})+\delta_M^{\dagger}\}.
$$
Since $\delta'_M\geq\tau_M\geq 1/(M\sqrt{4\pi\ln M})$,
$$
\delta_M^{\dagger}\geq\frac{1}{M(4\pi\ln M)^{1/4}},
\qquad
\lambda_M\sqrt{\ln(1/\delta_M^{\dagger})}
=O\!\left(\sqrt{\frac{\ln M}{M}}\right)=o(1).
$$
Thus Proposition~\ref{prop:Gaus}(i)--(ii) applies with $\mu=\lambda_M$ and $a=\delta_M^{\dagger}$.
We have
$$
\frac{1-G_{\lambda_M}(\delta_M^{\dagger})}{\delta_M^{\dagger}}
=1+o(1) \mbox{ and } a_M^*=\Phi(-\lambda_M/2)=1/2
+o(1)
$$
This yields
$$
\hat{\delta}_M=(2+o(1))\delta_M^{\dagger}=o(1/M).
$$

Define
$$
h(a)=\max\{G_{\lambda_M}(\min\{a+\delta_M^{\dagger},1\})-\delta_M^{\dagger},0\},
$$
and let $k$ be the function of Lemma~\ref{lem:boundF} with $g=G_{\lambda_M}$, $\delta=\delta_M^{\dagger}$ and $\hat{\delta}=\hat{\delta}_M$. For all sufficiently large $M$, $\delta_M^{\dagger}<a_M^*$. Application of the lemma  gives $h\geq k=G_{\lambda_M}\otimes f_{0,\hat{\delta}_M}$.
By Proposition~\ref{prop:lowerf} on the range $[\tau_M,1-\delta'_M]$, we have $f_M\geq h\geq k$ on $[\tau_M,1-\delta_M^{\dagger}]$ for all sufficiently large $M$. Since $\hat{\delta}_M\geq\delta_M^{\dagger}/a_M^*\geq\delta_M^{\dagger}$ and $k=0$ on $[1-\hat{\delta}_M,1]$, this also holds for the range $[1-\delta_M^{\dagger},1]$.
For the range $[0,\tau_M]$ we apply Proposition~\ref{prop:asf}(ii) and $\bar\delta_M\leq\delta'_M$ from the proof of Proposition~\ref{prop:lowerf},
$$
f_M(\tau_M)\;\geq\;1-\delta'_M =
1- M(\delta_M^{\dagger})^2
=
1- M\hat\delta_M^2/(4+o(1)),
$$
using $\hat\delta_M=(2+o(1))\delta^{\dagger}_M$. For $M$ large enough, the right hand side is at least $1-\hat{\delta}=k(0)$ (since $\hat{\delta}_M=o(1/M)$). Since both $f_M$ and $k$ are decreasing, we have  
for every $a\in[0,\tau_M]$, $f_M(a)\geq f_M(\tau_M)\geq k(0)\geq k(a)$. Combining the three ranges proves
$$
f_M(a) \geq (G_{\mu_M+\gamma'_M} \otimes f_{0,\hat{\delta}_M})(a) \ \  \mbox{ for } \ \ a\in [0,1].
$$

The lower bound $G_{\lambda_M}\otimes f_{0,\hat{\delta}_M}$ is symmetric, since $G_{\lambda_M}$ and $f_{0,\hat{\delta}_M}$ are symmetric and a tensor product of symmetric trade-off functions is symmetric~\citep[Section~3.1]{DBLP:journals/corr/abs-1905-02383}. By the argument following the proof of Lemma~\ref{lem:bound} leading to its corollary, it therefore also bounds $f_M^{-1}=T(P^{(M)}_1,P^{(M)}_0)$ from below, so by Theorem~\ref{thm:fs-dom}(b) it bounds the trade-off function of the mechanism in both orientations of neighboring datasets. Hence $f\geq G_{\lambda_M}\otimes f_{0,\hat{\delta}_M}$, since $f$ is the largest trade-off function for which the mechanism is $f$-DP, which is the first claim of the theorem.

We have, see ~\citep{DBLP:journals/corr/abs-1905-02383},
$$
f^{\otimes E}\geq G_{\mu_M\cdot \sqrt{E}+\gamma'_M\cdot\sqrt{E}}\otimes f_{0,1-(1-\hat{\delta}_M)^E}\geq
 G_{\mu_M\cdot \sqrt{E}+\gamma'_M\cdot\sqrt{E}}\otimes f_{0,\hat{\delta}_M \cdot E}.
$$
Now we can take the limit with $E=r_M^2\cdot M$  and $M\rightarrow \infty$ where $r_M\rightarrow c$ for some constant $c\geq 0$. Since the functions $G_.$ and $f_{.,.}$ are continuous in their (hyper) parameters and since $\hat{\delta}_M E
=o(r_M^2)\rightarrow 0$ and
$\gamma'_M \cdot \sqrt{E}\rightarrow 0$ as well as $\mu_M\cdot  \sqrt{E}=r_M \sqrt{e^{1/\sigma^2}-1}\sqrt{M/(M-1)}\rightarrow c \sqrt{e^{1/\sigma^2}-1}$,
we have for $M\rightarrow \infty$,
$$ f^{\otimes E}(a)\geq G_{c \sqrt{e^{1/\sigma^2}-1}}(a)-o(1)$$
uniformly in $a\in[0,1]$. If $c=0$, the upper bound $f^{\otimes E}(a)\leq1-a$ also gives $f^{\otimes E}(a)\to1-a$ uniformly in $a\in[0,1]$.
\hfill $\Box$

\begin{lemma}\label{lem:sepsub}
For any trade-off functions $f$ and $h$ (not necessarily symmetric),
$$\begin{aligned}\mbox{sep}(f\otimes h)\;&\leq\;
(1-(1-\sqrt{2}\cdot \mbox{sep}(f))(1-\sqrt{2}\cdot \mbox{sep}(h)))/\sqrt{2}
\\ &=\;
\mbox{sep}(f)+\mbox{sep}(h)-\sqrt{2}\cdot \mbox{sep}(f)\mbox{sep}(h)
\\ &\leq\;
\mbox{sep}(f)+\mbox{sep}(h),\end{aligned}$$
with equality in the first inequality if $h=f_{0,\delta}$ for some $\delta\in[0,1]$.
\end{lemma}
{\em Proof:} Since $0\leq f(a)\leq 1-a$, the point of $\{(\gamma,1-\gamma)\}$ closest to $(a,f(a))$ has $\gamma=(1+a-f(a))/2\in[0,1]$ and lies at distance $(1-a-f(a))/\sqrt2$. Hence $\mbox{sep}(f)=\delta_f/\sqrt2$ with $\delta_f=\max_a\{1-a-f(a)\}$, and $f\geq f_{0,\delta_f}$, where $f_{0,\delta}(a)=\max\{1-\delta-a,0\}$; likewise $h\geq f_{0,\delta_h}$. By Lemma~C.2 of \citet{DBLP:journals/corr/abs-1905-02383}, applied in each factor, $f\otimes h\geq f_{0,\delta_f}\otimes f_{0,\delta_h}
=f_{0,1-(1-\delta_f)(1-\delta_h)}$. Hence $\max_a\{1-a-(f\otimes h)(a)\}\leq 1-(1-\delta_f)(1-\delta_h)=\delta_f+\delta_h-\delta_f\delta_h$.
For equality, write $f=T(P,Q)$, so $\delta_f=\mathrm{TV}(P,Q)$; with $h=f_{0,\delta}$ the product pair equals $(P,Q)$ on a part of mass $1-\delta$ and is mutually singular on the rest, so its total variation is $(1-\delta)\mathrm{TV}(P,Q)+\delta=1-(1-\delta_f)(1-\delta)$. \hfill $\Box$

\noindent
{\em Proof of Corollary \ref{cor:ths}:}
For constant $\sigma$ and $E=O(M)$, Theorem \ref{theo:asym} together with Proposition \ref{prop:Gaus}.(ii) shows that (the second inequality is justified by Lemma~\ref{lem:sepsub})
\begin{eqnarray*}
\mbox{sep}(f^{\otimes E})&\leq & \mbox{sep}(G_{(\mu+\gamma'_M)\sqrt{E}}\otimes f_{0,1-(1-\hat{\delta}_M)^E}) \\
&\leq &
\frac{1}{\sqrt{2}}(1-2\Phi(-(\mu+\gamma'_M)\sqrt{E}/2))+\frac{1}{\sqrt{2}}(1-(1-\hat{\delta}_M)^E) \\
&\leq &
\frac{1}{\sqrt{2}}(\frac{1}{\sqrt{2\pi}}
  (\mu+\gamma'_M)\sqrt{E} +  O(M^{-3/2}E^{3/2}) +\hat{\delta}_M E)
  = O(\mu\sqrt{E}).
\end{eqnarray*}
For $E\geq M$, we have $\mu\sqrt{E}\geq\sqrt{e^{1/\sigma^2}-1}$ and $\mbox{sep}(f^{\otimes E})\leq 1/\sqrt{2}$, so $\mbox{sep}(f^{\otimes E})=O(\mu\sqrt{E})$, since $\sigma$ is fixed. Theorem \ref{th:concrete} shows that for constant $\sigma$,
$$ \mbox{sep}(f^{\otimes E})\leq \frac{1}{\sqrt{2}}(1-(1-\delta)^E)\leq \frac{1}{\sqrt{2}}\delta E = O( \mu  E). $$
\hfill $\Box$

\section{The Exact Poisson Asymptotic, and the Looseness of the $C_p$ Route} \label{app:cproute} 

For $E=r_M^2\cdot M$ with $r_M\rightarrow c$, the \emph{lower bound} of Corollary~5.4 in \citet{DBLP:journals/corr/abs-1905-02383} on the trade-off function over $E$ epochs with each epoch having $M$ rounds converges to $$G_{c\cdot \sqrt{2(e^{1/\sigma^2}\Phi(\frac{3}{2\sigma})+3\Phi(-\frac{1}{2\sigma})-2)}}(a).$$ 

Corollary~5.4 of~\citet{DBLP:journals/corr/abs-1905-02383} is the limit of the $C_p(f)$ subsampling bound, which symmetrizes and convexifies each step before composing, so it is a lower bound on the per-step trade-off function. Its constant is the \emph{one-sided} $\chi^2$. Writing $L=e^{X/\sigma-1/(2\sigma^2)}$ with $X\sim N(0,1)$, one checks $\mathbb{E}[L^2\mathbf 1_{L>1}]=e^{1/\sigma^2}\Phi(\tfrac{3}{2\sigma})$, $\mathbb{E}[L\mathbf 1_{L>1}]=\Phi(\tfrac{1}{2\sigma})$ and $\Pr[L>1]=\Phi(-\tfrac{1}{2\sigma})$, whence
\[
  2\,\mathbb{E}\bigl[(L-1)_+^2\bigr]\;=\;2\bigl(e^{1/\sigma^2}\Phi(\tfrac{3}{2\sigma})+3\Phi(-\tfrac{1}{2\sigma})-2\bigr),
\]
whereas our coefficient for random allocation is the \emph{two-sided} $\chi^2$, $\mathbb{E}[(L-1)^2]=e^{1/\sigma^2}-1$. 
Figure~\ref{fig:asymptotic_mu} plots the two Gaussian DP coefficients
$\sqrt{e^{1/\sigma^2}-1}$ (random allocation, Theorem~\ref{theo:asym}) and
$\sqrt{2\bigl(e^{1/\sigma^2}\Phi(\tfrac{3}{2\sigma})
+3\Phi(-\tfrac{1}{2\sigma})-2\bigr)}$ (subsampled Gaussian, Cor.~5.4
of~\citet{DBLP:journals/corr/abs-1905-02383}), together with their ratio.
The ratio is monotone in $\sigma$ on $(0,\infty)$, with limits
$\sqrt{2}$ as $\sigma\!\to\!0^+$ and $1$ as $\sigma\!\to\!\infty$.
The plotted ratio is therefore exactly $\sqrt{2\mathbb{E}[(L-1)_+^2]/\mathbb{E}[(L-1)^2]}$, which runs from $\sqrt2$ to $1$ for the elementary reason that a one-sided second moment is half a two-sided one in the symmetric limit and equal to it in the heavy-tailed limit. Since~\citet{DBLP:journals/corr/abs-1905-02383} note that the same $C_p(f)$ bound also holds for independent-inclusion sampling, this ratio measures the looseness of the $C_p$ route on a fixed sampling scheme. The \emph{exact} $q\to0$ asymptotic for Poisson subsampling at rate $q=1/M$ over $T=EM$ steps is Theorem~6 of Appendix~B.2 together with Lemma~B.1 of~\citet{Bu2019}, giving $G_{c\sqrt{e^{1/\sigma^2}-1}}$, the \emph{same} coefficient as Theorem~\ref{theo:asym}. \citet[Remark~4.4]{FeldmanShenfeld2025} also observe that allocation and Poisson agree asymptotically, at the level of the privacy profile and without a closed form. 

\begin{figure}[t]
    \centering
    \includegraphics[width=\columnwidth]{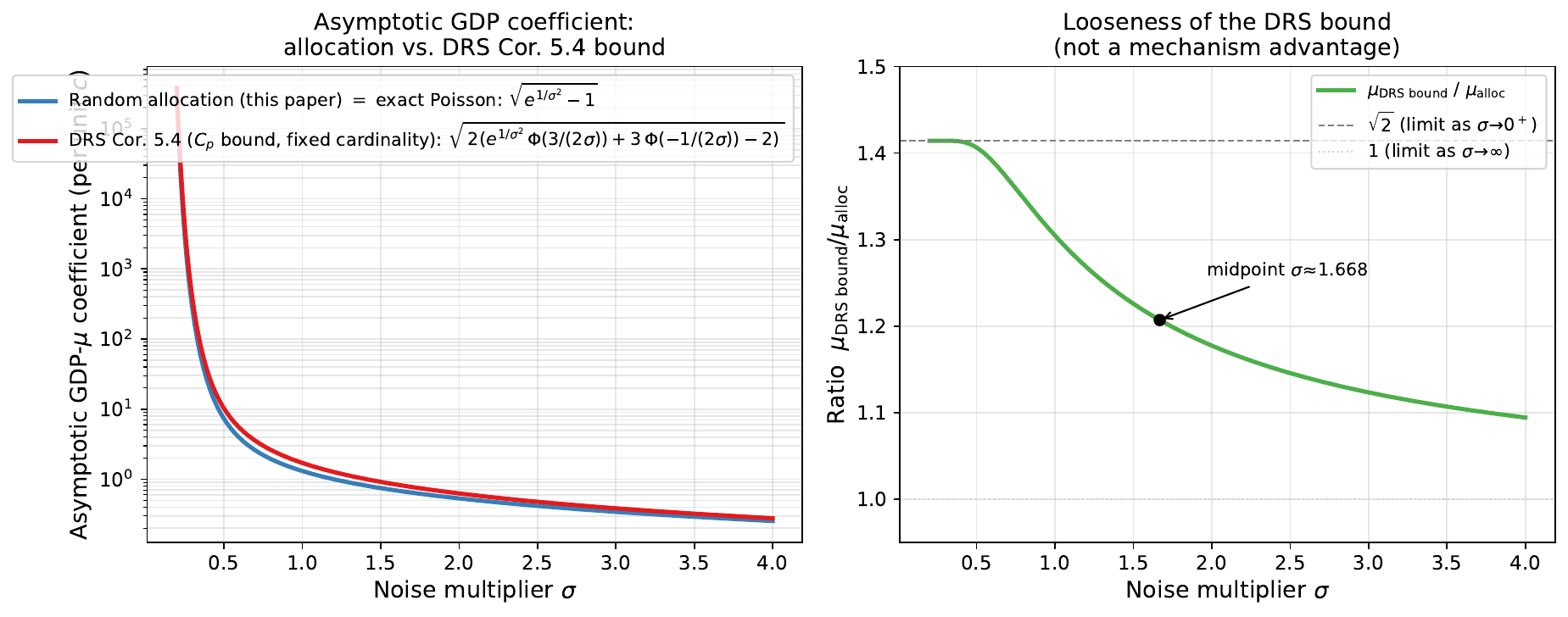}
    \caption{Asymptotic GDP-$\mu$ coefficient (per unit $c=\lim_{M\to\infty}r_M$, $E=r_M^2 M$) for random allocation (Theorem~\ref{theo:asym}, the same coefficient as the exact Poisson limit) vs.\ the $C_p(f)$ fixed-cardinality-subsampling bound (Cor.~5.4 of~\citet{DBLP:journals/corr/abs-1905-02383}). \textbf{Left:} both coefficients diverge as $\sigma\!\to\!0^+$ and decay like $1/\sigma$ as $\sigma\!\to\!\infty$ (log scale). \textbf{Right:} the ratio $\mu_{\mathrm{DRS\ bound}}/\mu_{\mathrm{alloc}}$ decreases monotonically from $\sqrt{2}$ (as $\sigma\!\to\!0^+$) to $1$ (as $\sigma\!\to\!\infty$); marker: midpoint $\sigma\!\approx\!1.668$ where the ratio equals $(1+\sqrt{2})/2\!\approx\!1.207$. This ratio equals $\sqrt{2\mathbb{E}[(L-1)_+^2]/\mathbb{E}[(L-1)^2]}$, i.e.\ a one-sided versus two-sided second moment; it measures the looseness of the Cor.~5.4 \emph{bound} of~\citet{DBLP:journals/corr/abs-1905-02383}; the exact Poisson asymptotic at matched rate equals our coefficient~\citep[Thm.~6 and Lem.~B.1]{Bu2019}.}
    \label{fig:asymptotic_mu}
\end{figure}

\section{Why Fixed-Size Shuffling Is Excluded}
\label{app:fixedsize}

The reduction of Appendix~\ref{app:adversarial} does not extend to fixed-size random shuffling, even under zero-out adjacency. Fix a direction and in each round let every other sample have a clipped gradient in the opposite direction with norm equal to $C$. The clipped gradient of the differing sample is set at zero for all rounds $j$ except for the round $J$ to which the differing sample is allocated, where the clipped gradient is also in the opposite direction of the fixed direction with norm equal to $C$.
The partial sum $X_j$ corresponding to the other samples in the $j$-th minibatch, written as in the channel argument of Appendix~\ref{app:adversarial} with $\Delta_j$ indicating whether the differing sample is allocated to the $j$-th minibatch, is equal to
\[
  X_j=-C(m-\Delta_j).
\]
So, for $J$ equal to the round index where the differing sample is used, $X_J=-C(m-1)$ while $X_j=-Cm$ for every $j\neq J$, and $X$ identifies $J$ exactly. Revealing $X$ and then analyzing a hidden-location residual mixture given by~\eqref{pdfd}--\eqref{pdfdprime} is therefore not a valid reduction here. After subtraction the augmented experiment $(X,G)$ is a Gaussian shift in a known coordinate and has trade-off function $G_{1/\sigma}$. That statement concerns the augmented experiment; it does not by itself identify the output-only trade-off function of $G$.

The output-only mechanism fails as well. Translating each batch sum by $mC$ and normalizing by $C\sigma$, write
\[
  P^{(M)}_c:=\frac1M\sum_{j=1}^M N\!\left(\frac{c}{\sigma}\,\mathbf e_j,\;I_M\right),
\]
so that $P^{(M)}_0=N(0,1)^{\otimes M}$. Now take the differing sample's contribution to be $0$ under $d$ and $+C$ under $d'$, with every common contribution equal to $-C$; all values lie within the clipping norm, so this is a legal neighboring pair. The exact output distributions are now $P^{(M)}_1$ and $P^{(M)}_2$, whereas the scalar reduction analyzes $P^{(M)}_0$ and $P^{(M)}_1$. The proposed pair does not dominate the actual one. \citet{Chua2024} reach the same two probability distributions from the same antipodal choice of the common gradients and record them as their Equations~(1) and~(2), in the ordered form $(P^{(M)}_2,P^{(M)}_1)$, the reverse of our $(H_0,H_1)$; they conjecture that pair to be tightly dominating for fixed-size shuffling (their Conjecture~3.2, which they leave open), and observe that under shuffling the non-differing samples leak the location of the differing one.

At $M=2$ and $\sigma=1$ the failure of domination can be verified exactly. Let $Z_1,Z_2$ be i.i.d.\ $N(0,1)$, put $u_i=e^{Z_i-1/2}$ and $v_i=e^{2Z_i-2}$, and let $L_1=(u_1+u_2)/2$ and $L_2=(v_1+v_2)/2$ be the likelihood ratios of $P^{(2)}_1$ and $P^{(2)}_2$ with respect to $P^{(2)}_0$. All expectations of the form $\mathbf{E}[.]$ without subscript in the remainder of the argument are taken under $P^{(2)}_0$. Then $\chi^2(P^{(2)}_1\|P^{(2)}_0)=\mathbf E[L_1^2]-1=(e-1)/2$.
Since $dP^{(2)}_2/dP^{(2)}_1=L_2/L_1$, we have $\chi^2(P^{(2)}_2\|P^{(2)}_1)=\mathbf E_{P^{(2)}_1}[(L_2/L_1)^2]-1=\mathbf E[L_1(L_2/L_1)^2]-1$. 
Substitute $S=(Z_1+Z_2)/2$ and $W=(Z_1-Z_2)/2$, which are independent $N(0,1/2)$. Then $L_1=e^{S-1/2}\cosh(W)$ and $L_2=e^{2S-2}\cosh(2W)$, and the identity $\cosh^2(2w)/\cosh(w)=\cosh(3w)-\cosh(w)+\operatorname{sech}(w)$ gives
\[
\frac{L_2^2}{L_1}=e^{3S-7/2}\big(\cosh(3W)-\cosh(W)+\operatorname{sech}(W)\big).
\]
Since $\mathbf E[e^{3S}]=\mathbf E[\cosh(3W)]=e^{9/4}$ and $\mathbf E[\cosh(W)]=e^{1/4}$, this yields $\mathbf E[L_2^2/L_1]=e-e^{-1}+e^{-5/4}\,\mathbf E[\operatorname{sech}(W)]$, and hence
\begin{eqnarray*}
\chi^2(P^{(2)}_2\|P^{(2)}_1) &=& \mathbf E\big[L_2^2/L_1\big]-1 \\
&=& e-1-e^{-1}+e^{-5/4}\,\mathbf E[\operatorname{sech}(W)] \\
&>& e-1-e^{-1}\;>\;\frac{e-1}{2}\;=\;\chi^2(P^{(2)}_1\|P^{(2)}_0).
\end{eqnarray*}
Numerically, $\chi^2(P^{(2)}_2\|P^{(2)}_1)=1.5895$, using $\mathbf E[\operatorname{sech}(W)]=0.8345$, whereas $\chi^2(P^{(2)}_1\|P^{(2)}_0)=(e-1)/2=0.8591$. 
Domination of trade-off functions forces the reverse inequality leading to a contradiction, since for $L=dQ/dP$ and $H_\lambda(Q\|P)=\mathbf E_P[(L-\lambda)_+]$, which is the hockey-stick divergence of Appendix~\ref{app:adversarial} at $\lambda=e^\varepsilon$, one has $\chi^2(Q\|P)+1=2\int_0^\infty H_\lambda(Q\|P)\,d\lambda$, and domination would imply $H_\lambda(P_1^{(2)}||P_0^{(2)})\geq H_\lambda(P_2^{(2)}||P_1^{(2)})$ at every $\lambda\geq 0$. The identity holds because $\int_0^\infty(\ell-\lambda)_+\,d\lambda=\ell^2/2$ for every $\ell\geq 0$, so that, by Tonelli's theorem we may swap expectation and integral, $2\int_0^\infty H_\lambda(Q\|P)\,d\lambda=\mathbf E_P[L^2]=\chi^2(Q\|P)+1$.

The failure persists at large $M$. With $u=e^{Z/\sigma-1/(2\sigma^2)}$ and $v=e^{2Z/\sigma-2/\sigma^2}$ for $Z\sim N(0,1)$,
\[
  M\,\chi^2\big(P^{(M)}_1\,\big\|\,P^{(M)}_0\big)=e^{1/\sigma^2}-1
\]
exactly. 
Let $Z_1,\ldots,Z_M$ be the coordinates under $P^{(M)}_0$, i.e.\ i.i.d.\ $N(0,1)$, form $u_i,v_i$ from $Z_i$ as $u,v$ are formed from $Z$, and put $A_M=M^{-1/2}\sum_{i=1}^M(v_i-u_i)$ and $B_M=M^{-1}\sum_{i=1}^M u_i$. The likelihood ratios of $P^{(M)}_1$ and $P^{(M)}_2$ with respect to $P^{(M)}_0$ are $L_1=B_M$ and $L_2=B_M+M^{-1/2}A_M$, both with mean $1$, so $\chi^2(P^{(M)}_2\|P^{(M)}_1)=\mathbf E[(L_2-L_1)^2/L_1]$ and
\[
  M\,\chi^2\big(P^{(M)}_2\,\big\|\,P^{(M)}_1\big)
  =\mathbf E[A_M^2/B_M],
  \qquad
  \mathbf E[A_M^2]=\mathbf E[(v-u)^2],
\]
all expectations being taken under $P^{(M)}_0$. The fourth-moment formula for sums of independent centered variables gives $\mathbf E[A_M^4]=O(1)$ and $\mathbf E[(B_M-1)^4]=O(M^{-2})$, and convexity of $x\mapsto x^{-4}$ gives $\mathbf E[B_M^{-4}]\leq\mathbf E[u^{-4}]=e^{10/\sigma^2}$. Since $A_M^2/B_M-A_M^2=A_M^2(1-B_M)/B_M$, the Cauchy--Schwarz inequality gives $|\mathbf E[A_M^2/B_M]-\mathbf E[A_M^2]|\leq(\mathbf E[A_M^4])^{1/2}(\mathbf E[(B_M-1)^2B_M^{-2}])^{1/2}$, and applying it once more to the last expectation yields
\[
 \begin{aligned}
  \left|\mathbf E[A_M^2/B_M]-\mathbf E[A_M^2]\right|
  &\leq
  \big(\mathbf E[A_M^4]\big)^{1/2}
  \big(\mathbf E[(B_M-1)^4]\big)^{1/4}
  \big(\mathbf E[B_M^{-4}]\big)^{1/4}\\
  &=O(M^{-1/2}).
 \end{aligned}
\]
Consequently, for each fixed $\sigma>0$,
\[
  M\,\chi^2\big(P^{(M)}_2\,\big\|\,P^{(M)}_1\big)\;\longrightarrow\;\mathbf E\big[(v-u)^2\big]=e^{4/\sigma^2}+e^{1/\sigma^2}-2e^{2/\sigma^2}.
\]
The two limiting coefficients differ by $\big(e^{2/\sigma^2}-1\big)^{2}>0$, so the scalar pair fails to dominate at every fixed $\sigma>0$ once $M$ is large enough. 
The $M=2$ computation is an exact counterexample, but it lies outside the regime the theorems cover, since $\sqrt{3/\ln 2}=2.08>1$.
The large-$M$ failure lies inside that regime, so the side condition $\sigma\geq\sqrt{3/\ln M}$ does not rescue the fixed-size scheme. Random allocation is therefore the scope of the results in this paper.

\newpage

\end{document}